\documentclass{article}

\usepackage[final]{nips_2016}

\usepackage[utf8]{inputenc} 
\usepackage[T1]{fontenc}    
\usepackage{hyperref}       
\usepackage{url}            
\usepackage{booktabs}       
\usepackage{amsfonts}       
\usepackage{nicefrac}       
\usepackage{microtype}      
\usepackage{wrapfig}
\usepackage{enumitem}
\usepackage{mathtools}
\usepackage{tikz}
\usepackage{amsmath,amssymb,amsthm}
\usepackage{subfigure}
\usepackage{xspace}
\usepackage[lined, boxed, ruled, commentsnumbered, noend]{algorithm2e}
\usepackage{hyperref}
\usepackage{xcolor,xpatch}

\hypersetup{
  colorlinks = true,
  citecolor = cyan,
}

\usepackage{cleveref}

\title{Near-optimal Bayesian Active Learning \\ with Correlated and Noisy Tests}
\author{
  Yuxin Chen \\ ETH Zurich \\ \texttt{yuxin.chen@inf.ethz.ch}
  \And
  S. Hamed Hassani \\ ETH Zurich \\ \texttt{hamed@inf.ethz.ch}
  \And
  Andreas Krause \\ ETH Zurich \\ \texttt{krausea@ethz.ch}
}

\usepackage{dsfont}
\usepackage{cancel}
\DeclarePairedDelimiter\abs{\lvert}{\rvert}%
\DeclarePairedDelimiter\norm{\lVert}{\rVert}%

\makeatletter
\let\oldabs\abs
\def\abs{\@ifstar{\oldabs}{\oldabs*}}
\let\oldnorm\norm
\def\norm{\@ifstar{\oldnorm}{\oldnorm*}}
\makeatother

\newcommand{\circled}[1]{\tikz[inner sep=.25ex,baseline=-.75ex] \node[circle,draw,color=white,fill=black!60] {#1};}

\newcommand{\pto}{\mathrel{\ooalign{\hfil$\mapstochar$\hfil\cr$\to$\cr}}}

\newcommand{\at}[2][]{#1|_{#2}}

\newcommand\setiplus{\Hiddenvar_i^+}
\newcommand\setiminus{\Hiddenvar_i^-}
\newcommand{\hplus}{h_+}
\newcommand{\hminus}{h_-}

\newcommand{\eqrefcus}[1]{(\ref{#1})}

\newcommand{\argmax}{\operatornamewithlimits{argmax}}
\newcommand{\argmin}{\operatornamewithlimits{argmin}}

\newcommand{\unit}[1]{\mathds{1}\left\{#1\right\}}

\newcommand{\policy}{\pi}

\newcommand{\Obs}{X}
\newcommand{\obs}{x}

\newcommand{\Targetvar}{Y}
\newcommand{\targetvar}{y}

\newcommand{\numtar}{t}
\newcommand{\numtest}{m}
\newcommand{\numrc}{n}

\newcommand{\Hiddenvar}{\Theta}
\newcommand{\hiddenvar}{\theta}
\newcommand{\HvarSupp}{\supp(\Theta)}

\newcommand{\nr}{\epsilon} 
\newcommand{\nrbar}{\bar{\epsilon}} 

\newcommand{\gsz}{k'}
\newcommand{\optsz}{k}

\newcommand{\entropy}[1]{\mathbb{H}\paren{#1}}
\newcommand{\bientropy}[1]{\mathbb{H}_2\paren{#1}}

\newcommand{\errorprob}{p_\textsc{err}}
\newcommand{\errormap}{p^{\textsc{MAP}}_\textsc{err}}

\newcommand{\errorst}{p_\textsc{e}}

\newcommand{\epnoisybot}{p_{\textsc{e}, \text{noisy}}^\bot}

\newcommand{\epnlessbot}{p_{\textsc{e}, \text{noiseless}}^\bot}

\newcommand{\epnoisybotatl}{p_{\textsc{e}, \text{noisy}, \psi_\ell}^\bot}

\newcommand{\epnlessbotatl}{p_{\textsc{e}, \text{noiseless}, \psi_\ell}^\bot}

\newcommand{\eptop}{p_{\textsc{e}}^\top }

\newcommand{\eptopatl}{p_{\textsc{e},\psi_\ell}^\top }

\newcommand{\expct}[1]{\mathbb{E}\left[#1\right]}
\newcommand{\expctover}[2]{\mathbb{E}_{#1}\!\left[#2\right]}
\newcommand{\given}{\mid}

\renewcommand{\expct}{\mathbb{E}\expectarg}
\DeclarePairedDelimiterX{\expectarg}[1]{[}{]}{%
  \ifnum\currentgrouptype=16 \else\begingroup\fi
  \activatebar#1
  \ifnum\currentgrouptype=16 \else\endgroup\fi
}

\newcommand{\innermid}{\nonscript\;\delimsize\vert\nonscript\;}
\newcommand{\activatebar}{%
  \begingroup\lccode`\~=`\|
  \lowercase{\endgroup\let~}\innermid
  \mathcode`|=\string"8000
}

\newcommand{\ecedgraph}{\raisebox{-.5ex}{\includegraphics[height=11pt]{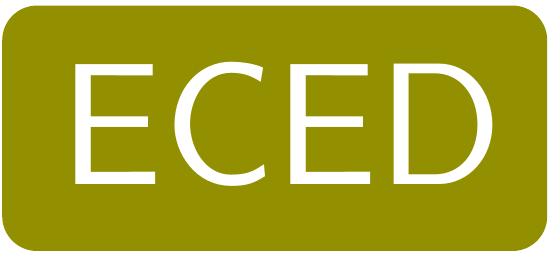}}\xspace}
\newcommand{\optgraph}{\raisebox{-.5ex}{\includegraphics[height=11pt]{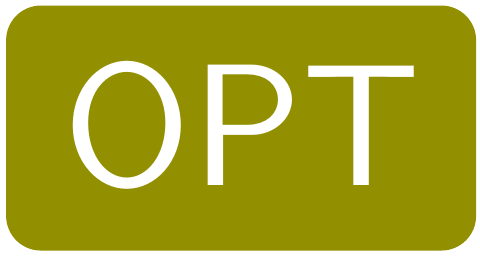}}\xspace}

\newcommand{\ecedgaingraph}{\raisebox{-.5ex}{\includegraphics[height=11pt]{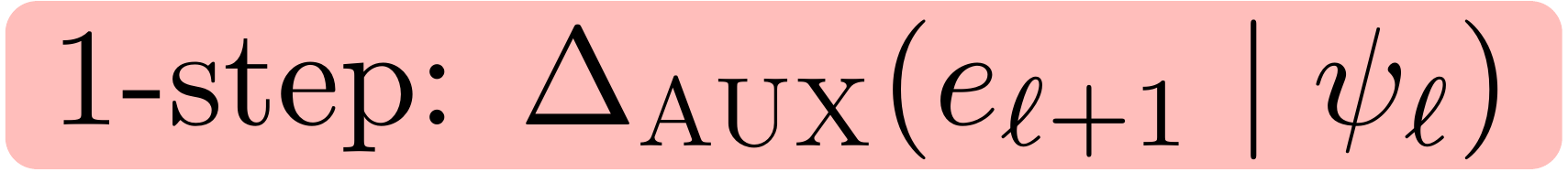}}\xspace}
\newcommand{\optkgaingraph}{\raisebox{-.5ex}{\includegraphics[height=11pt]{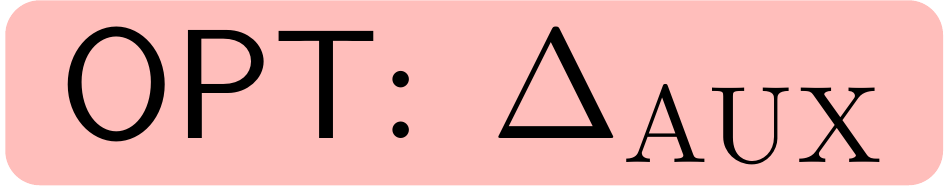}}\xspace}

\newcommand{\fvsperrubgraph}{\raisebox{-.5ex}{\includegraphics[height=12pt]{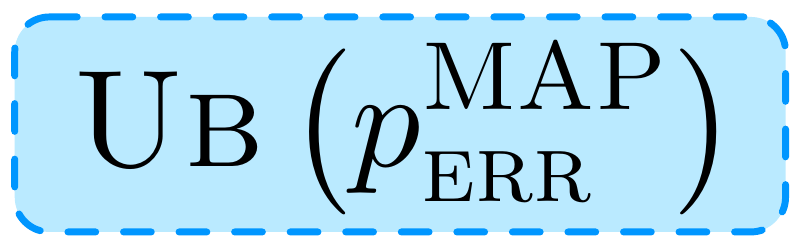}}\xspace}
\newcommand{\fvsperrlbgraph}{\raisebox{-.5ex}{\includegraphics[height=12pt]{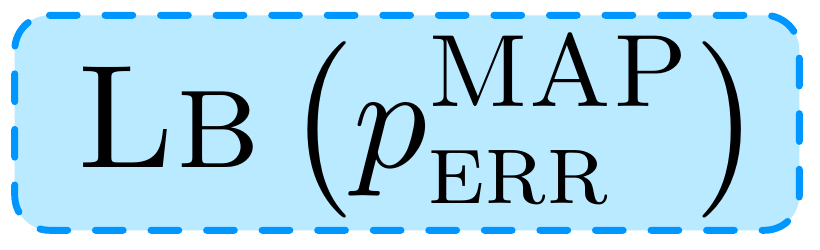}}\xspace}

\newcommand{\adasubmgraph}{\raisebox{-.5ex}{\includegraphics[height=11pt]{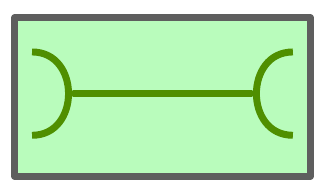}}\xspace}

\newcommand{\ectgainggraph}{\raisebox{-.7ex}{\includegraphics[height=12pt]{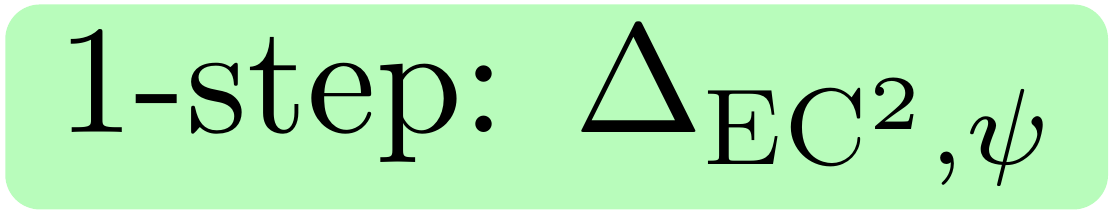}}\xspace}
\newcommand{\optgainectggraph}{\raisebox{-.7ex}{\includegraphics[height=12pt]{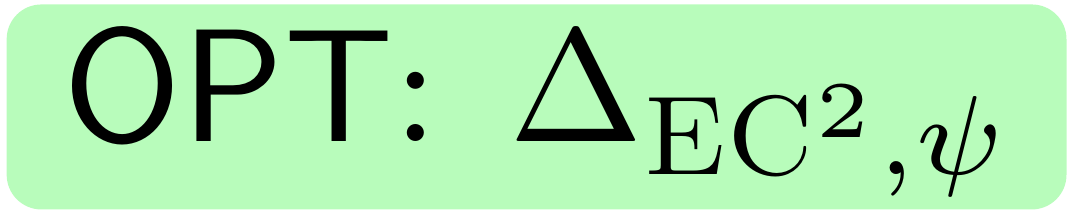}}\xspace}

\newcommand{\auxgainggraph}{\raisebox{-.6ex}{\includegraphics[height=11pt]{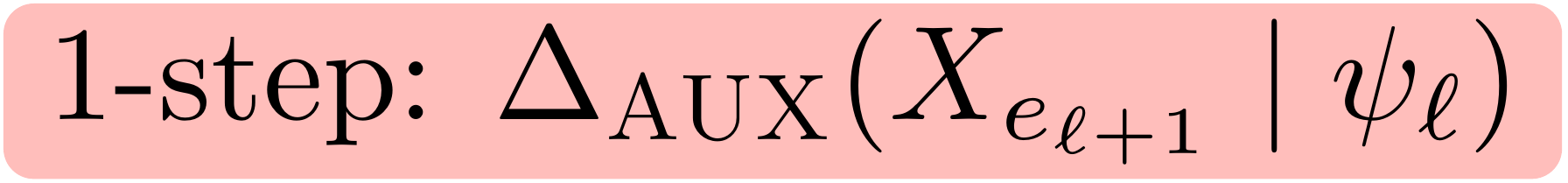}}\xspace}
\newcommand{\auxvsecarrowgraph}{\raisebox{-.3ex}{\includegraphics[height=8pt]{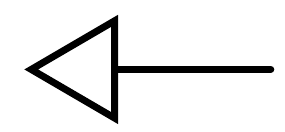}}\xspace}
\newcommand{\auxvsecforkgraph}{\raisebox{-.1ex}{\includegraphics[height=6pt]{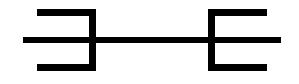}}\xspace}

\newcommand{\ecobj}{f_{\textsc{EC}^2}}
\newcommand{\ecobjatl}{f_{\textsc{EC}^2, \psi_\ell}}

\newcommand{\auxobj}{f_{\textsc{aux}}}
\newcommand{\auxgain}{\Delta_{\textsc{aux}}}

\newcommand{\ecedgainPeroutcomeoffset}{\delta_{\textsc{offset}}}

\newcommand{\ecedgainPerOutcome}{\delta_{\textsc{bs}}}

\newcommand{\eced}{{\sc ECED}}

\newcommand{\auxavg}{f^{\text{avg}}_{\textsc{aux}}}

\newcommand{\ecedgain}{\Delta_{\textsc{ECED}}}
\newcommand{\ecedgainat}[2]{\Delta_{\textsc{ECED},{#1}}\!\left(#2\right)}

\newcommand{\ectat}[1]{f_{\textsc{EC}^2,{#1}}}
\newcommand{\ectgain}{\Delta_{\textsc{EC}^2}}
\newcommand{\ectgainat}[2]{\Delta_{\textsc{EC}^2,{#1}}\!\left(#2\right)}

\newcommand{\cV}{{\mathcal{V}}}

\newcommand{\cX}{{\mathcal{X}}}
\newcommand{\cY}{{\mathcal{Y}}}

\newcommand{\cO}{{\mathcal{O}}}

\newcommand{\paren} [1] {\ensuremath{ \left( {#1} \right) }}

\newcommand{\bigO}[1]{\ensuremath{O\paren{#1}}}

\newcommand{\bigOmega}[1]{\ensuremath{\Omega\paren{#1}}}

\newcommand{\OPT}{{\sf OPT}\xspace}
\newcommand{\opt}{{\sf opt}}
\newcommand{\deltag}{\delta_{\sf g}}

\renewcommand{\Pr}[1]{\ensuremath{\mathbb{P}\left[#1\right] }}

\def \argmax {\mathop{\rm arg\,max}}
\def \argmin {\mathop{\rm arg\,min}}

\newcommand{\NonNegativeReals}{\ensuremath{\mathbb{R}_{\ge 0}}}

\newcommand{\Testset}{\cV}
\newcommand{\Test}{X}
\newcommand{\test}{x}

\newcommand{\targetVarDom}{\cY}
\newcommand{\hiddenVarDom}{\HvarSupp}
\newcommand{\obsDom}{\cO}

\DeclareMathOperator{\cost}{cost}
\DeclareMathOperator{\supp}{supp}
\newcommand{\MAP}{{\textsc{MAP}}\xspace}

\newcommand{\ECT}{{$\textsc{EC}^2$}\xspace}
\newcommand{\VOI}{{\sc VoI}\xspace}

\newcommand{\GBS}{{\sc GBS}\xspace}

\newcommand{\ECED}{{\sc ECED}\xspace}

\newtheorem{theorem}{Theorem}

\newtheorem{lemma}[theorem]{Lemma}

\theoremstyle{definition}

\theoremstyle{remark}

\numberwithin{equation}{section}

\newcommand{\tableref}[1]{Table~\ref{#1}}
\newcommand{\figref}[1]{Fig.~\ref{#1}}

\newcommand{\secref}[1]{\S\ref{#1}}
\newcommand{\thmref}[1]{Theorem~\ref{#1}}

\newcommand{\lemref}[1]{Lemma~\ref{#1}}
\newcommand{\algref}[1]{Algorithm~\ref{#1}}

\begin{document}

\maketitle

\begin{abstract}
  We consider the Bayesian active learning and experimental design problem, where the goal is to learn the value of some unknown target variable through a sequence of informative, noisy tests. In contrast to prior work, we focus on the challenging, yet practically relevant setting where test outcomes can be conditionally \emph{dependent} given the hidden target variable. Under such assumptions, common heuristics, such as greedily performing tests that maximize the reduction in uncertainty of the target, often perform poorly.

In this paper, we propose \eced, a novel, computationally efficient active learning algorithm, and prove strong theoretical guarantees that hold with correlated, noisy tests. Rather than directly optimizing the prediction error, at each step, \ECED picks the test that maximizes the gain in a surrogate objective, which takes into account the dependencies between tests. Our analysis relies on an information-theoretic auxiliary function to track the progress of \ECED, and utilizes adaptive submodularity to attain the near-optimal bound. We demonstrate strong empirical performance of \eced~on two problem instances, including a Bayesian experimental design task intended to distinguish among economic theories of how people make risky decisions, and an active preference learning task via pairwise comparisons.

\end{abstract}

\section{Introduction}

Optimal information gathering, i.e., selectively acquiring the most useful data, is one of the central challenges in machine learning. The problem of optimal information gathering has been studied in the context of active learning 
\citep{dasgupta04,settles.book12}, Bayesian experimental design \citep{chaloner1995bayesian}, policy making \citep{Runge20111214}, optimal control \citep{smallwood1973optimal}, and numerous other domains. In a typical set-up for these problems, there is some unknown \emph{target variable} $\Targetvar$ of interest, and a set of \emph{tests} which correspond to observable variables defined through a probabilistic model. The goal is to determine the value of the target variable with a sequential \emph{policy} -- which adaptively selects the next test based on previous observations -- such that the cost of performing these tests is minimized.

Deriving the optimal testing policy is NP-hard in general \citep{chakaravarthy2007decision}; however, under certain conditions, some approximation results are known. In particular, if test outcomes are deterministic functions of the target variable (i.e., in the \emph{noise-free} setting), a simple greedy algorithm, namely Generalized Binary Search (GBS), is guaranteed to provide a near-optimal approximation of the optimal policy \citep{kosaraju1999optimal}.
On the other hand, if test outcomes are noisy, but the outcomes of different tests are \emph{conditionally independent} given $\Targetvar$ (i.e., under the Na{\"i}ve Bayes assumption), then using the most informative selection policy, which greedily selects the test that maximizes the expected reduction in uncertainty of the target variable (quantified in terms of Shannon entropy), is guaranteed to perform near-optimally \citep{chen15mis}.

However, in many practical problems, due to the effect of noise or complex structural assumptions in the probabilistic model (beyond Na{\"i}ve Bayes), we only have access to tests that are \emph{indirectly informative} about the target variable $\Targetvar$ (i.e., test outcomes depend on $\Targetvar$ through another hidden random variable. See \figref{fig:model_ec2}.) -- as a consequence, the test outcomes become conditionally dependent given $\Targetvar$.
Consider a medical diagnosis example, where a doctor wants to predict the best treatment for a patient, by carrying out a series of medical tests, each of which reveals some information about the patient's physical condition. Here, outcomes of medical tests are conditionally independent given the patient's condition, but are \emph{not} independent given the treatment, which is made based on the patient's condition. It is known that in such cases, both GBS and the most informative selection policy (which myopically maximizes the information gain w.r.t. the distribution over $\Targetvar$) can perform arbitrarily poorly. \citet{golovin10near} then formalize this problem as an \emph{equivalence class determination} problem (See \secref{sec:ec2}), and show that if the tests' outcomes are noise-free, then one can obtain near-optimal expected cost, by running a greedy policy based on a surrogate objective function. Their results rely on the fact that the surrogate objective function exhibits \emph{adaptive submodularity} \citep{golovinjair2011}, a natural diminishing returns property that generalizes the classical notion of submodularity to adaptive policies. Unfortunately, in the more general setting where tests are noisy, no efficient policies are known to be provably competitive with the optimal policy.

\vspace{-2mm}
\paragraph{Our Contribution.}
In this paper, we introduce \textbf{E}quivalence \textbf{C}lass \textbf{E}dge \textbf{D}iscounting (\ECED), 
a novel algorithm for practical Bayesian active learning and experimental design problems, and prove strong theoretical guarantees with correlated, noisy tests. In particular, we focus on the setting where the tests' outcomes  indirectly depend on the target variable 
(and hence conditionally dependent given $\Targetvar$), and we assume that the outcome of each test can be corrupted by some random, \emph{persistent} noise (\secref{sec:model}). 
We prove that when the test outcomes are binary, and the noise on test outcomes are mutually independent, then \ECED is guaranteed to obtain near-optimal cost, 
compared with an optimal policy that achieves a lower prediction error (\secref{sec:eced}). We develop a theoretical framework for analyzing such sequential policies, where we leverage an information-theoretic auxiliary function to reason about the effect of noise, and combine it with 
the theory of adaptive submodularity to attain the near-optimal bound (\secref{sec:analysis}). The key insight is to show that \ECED is effectively making progress in the long run as it picks more tests, even if the myopic choices of tests do not have immediate gain in terms of reducing the uncertainty of the target variable. 
We demonstrate the compelling performance of \eced~on two real-world problem instances, 
a Bayesian experimental design task intended to distinguish among economic theories of how people make risky decisions, and an active preference learning task via pairwise comparisons (\secref{sec:exp}). 
To facilitate better understanding, we provide the detailed proofs, illustrative examples and a third application on pool-based active learning in the supplemental material.


\vspace{-2mm}
\section{Preliminaries and Problem Statement}\label{sec:model}
\vspace{-2mm}
\paragraph{The Basic Model}
Let $\Targetvar$ be the target random variable whose value we want to learn. The value of $\Targetvar$, which ranges among set $\targetVarDom = \{\targetvar_1, \dots, \targetvar_\numtar\}$, depends deterministically on another random variable $\Hiddenvar \in \HvarSupp = \{\hiddenvar_1, \dots, \hiddenvar_n\}$ with some known distribution $\Pr{\Hiddenvar}$. Concretely, 
there is a deterministic mapping $r: \hiddenVarDom \rightarrow \targetVarDom$ that gives $\Targetvar = r(\Hiddenvar)$.
Let $\cX = \{X_1, \dots, X_m\}$ be a collection of discrete observable variables that are statistically dependent on $\Hiddenvar$ (see \figref{fig:model_ec2}). 
\begin{wrapfigure}{r}{0.3\textwidth}
  \vspace{-.3cm}
  \begin{center}
    \includegraphics[width=.25\textwidth]{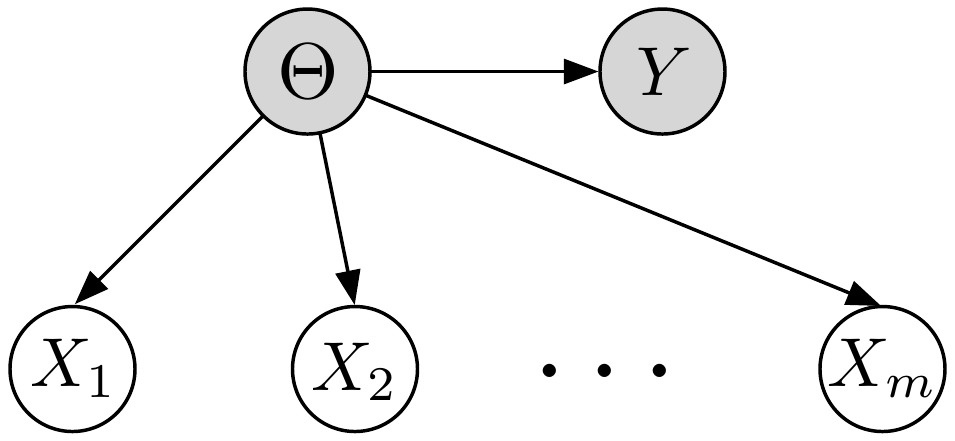}
  \end{center}
  \vspace{-.4cm}
  \caption{The basic model} \label{fig:model_ec2}
  \vspace{-.4cm}
\end{wrapfigure}
We use $e \in \Testset \triangleq \{1,\dots,m\}$ as the indexing variable of a test. Performing each test $\Obs_e$ produces an outcome $x_e \in \obsDom$ (here, $\obsDom$ encodes the set of possible outcomes of a test), and incurs a \emph{unit cost}.  
We can think of $\Hiddenvar$ as representing the underlying ``root-cause'' among a set of $n$ possible root-causes of the joint event $\{X_1, \dots, X_m\}$, and $\Targetvar$ as representing the optimal ``target action'' to be taken for root-cause $\Hiddenvar$. Also, each of the $\Obs_e$'s is a ``test'' that we can perform, whose observation reveals some information about $\Hiddenvar$. In our medical diagnosis example (see \figref{fig:medicaldiagnosis}),  $\Test_e$'s encode tests' outcomes, $\Targetvar$ encodes the treatment, and $\Hiddenvar$ encodes the patient's physical condition. 

Crucially, we assume that $\Obs_e$'s are \emph{conditionally independent} given $\Hiddenvar$, i.e., $\Pr{\Hiddenvar,\Obs_1,\dots,\Obs_m} = \Pr{\Hiddenvar} \prod_{i=1}^m \Pr{X_i \given \Hiddenvar}$ with known parameters.
Note that noise is implicitly encoded in our model, as we can equivalently assume that $\Test_e$'s are first generated from a deterministic mapping of $\Hiddenvar$, and then perturbed by some random noise. As an example, if test outcomes are binary, then we can think of $\Test_e$ as resulting from flipping the deterministic outcome of test $e$ given $\Hiddenvar$ with some probability, 
and the flipping events of the tests are mutually independent.
\paragraph{Problem Statement}\label{sec:problemdef}
We consider sequential, adaptive policies for picking the tests. Denote a policy by $\pi$. In words, a policy specifies which test to pick next, as well as when to stop picking tests, based on the tests picked so far and their corresponding outcomes. After each pick, our observations so far can be represented as a partial realization 
$\Psi\in 2^{\Testset\times \obsDom}$ (e.g., $\Psi$ encodes what tests have been performed and what their outcomes are). Formally, a policy $\policy:2^{\Testset \times \obsDom} \pto \Testset$ is defined to be a partial mapping from partial realizations $\Psi$ to tests.
Suppose that running $\pi$ till termination returns a sequence of test-observation pairs of length $k$, denoted by $\psi_\pi$, i.e., $\psi_\pi \triangleq \{(e_{\pi, 1}, \obs_{e_{\pi,1}}), (e_{\pi, 2}, \obs_{e_{\pi, 2}}), \cdots, (e_{\pi, \optsz}, \obs_{e_{\pi, \optsz}}) \}$. This can be interpreted as a random path\footnote{What $\pi$ returns in the end is random, dependent on the outcomes of selected tests. 
} taken by policy $\pi$. Once $\psi_\pi$ is observed, we obtain a new posterior on $\Hiddenvar$ (and consequently on $\Targetvar$). After observing $\psi_\pi$,
the \MAP estimator of $\Targetvar$ has error probability
$\errorprob^{\MAP}({\psi_\pi}) \triangleq 1 - \max_{\targetvar \in \targetVarDom} p(\targetvar \given \psi_\pi)$.
The expected error probability after running policy $\pi$ is then defined as $ \errorprob({\pi}) \triangleq \expctover{\psi_\pi}{\errorprob^\MAP({\psi_\pi})}$. In words, $\errorprob({\pi})$ is the expected error probability w.r.t. the posterior, given the final outcome of $\pi$.
Let the (worst-case) cost of a policy $\pi$ be 
$\cost(\pi) \triangleq \max_{\psi_{\pi}} {\left|\psi_{\pi}\right|}$, i.e., the maximum number of tests performed by $\pi$ over all possible paths it takes.
Given some small tolerance $\delta \in [0,1]$, we seek a policy with the minimal cost, such that upon termination, it will achieve expected error probability less than $\delta$. Denote such policy by $\OPT(\delta)$. Formally, we seek
\begin{align}
  \OPT(\delta) \in \argmin_{\policy}{\cost(\policy)},\quad
  \text{s.t.} ~\errorprob({\pi}) < \delta.
  \label{eq:drd}
\end{align}

\subsection{Special Case: The Equivalence Class Determination Problem}\label{sec:ec2}
\vspace{-1mm}
\begin{figure*}[t]
  \centering
  \subfigure[Medical diagnosis example]{%
    \includegraphics[width=.28\textwidth]{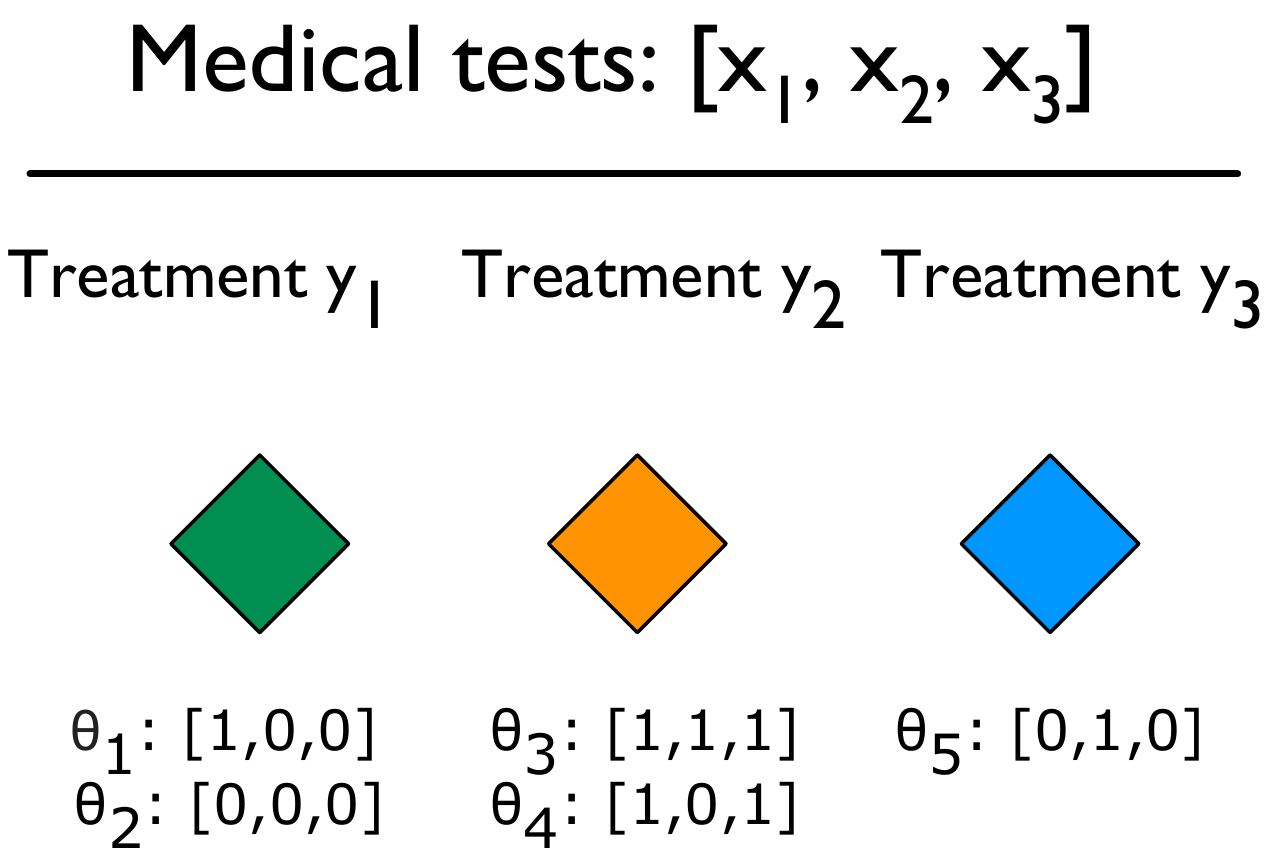}
    \label{fig:medicaldiagnosis}
  }\qquad
  \subfigure[Initialization]{%
    \includegraphics[width=.18\textwidth]{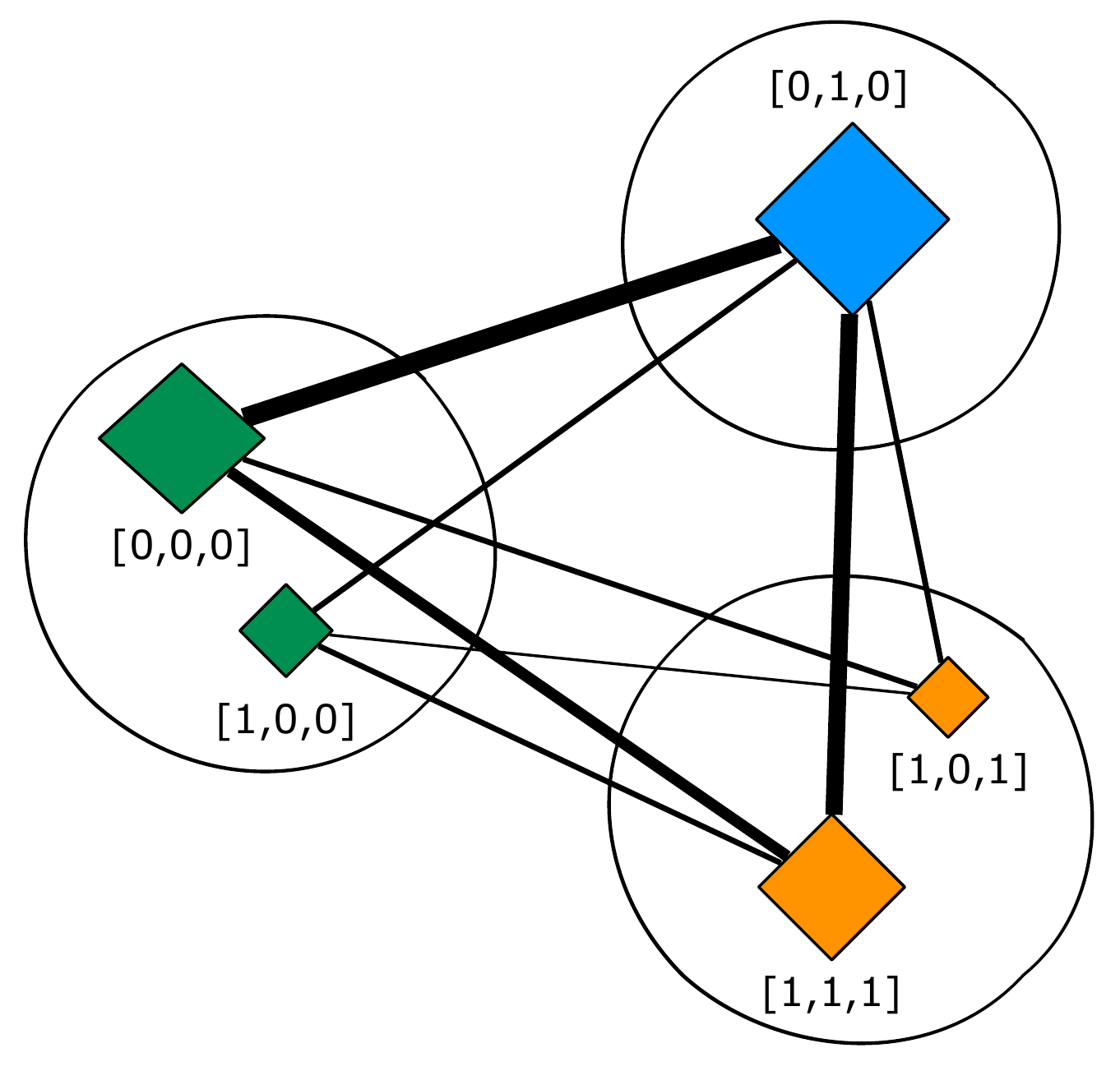}
    \label{fig:ecd}
  }\quad
  \subfigure[\ECT]{%
    \includegraphics[width=.18\textwidth]{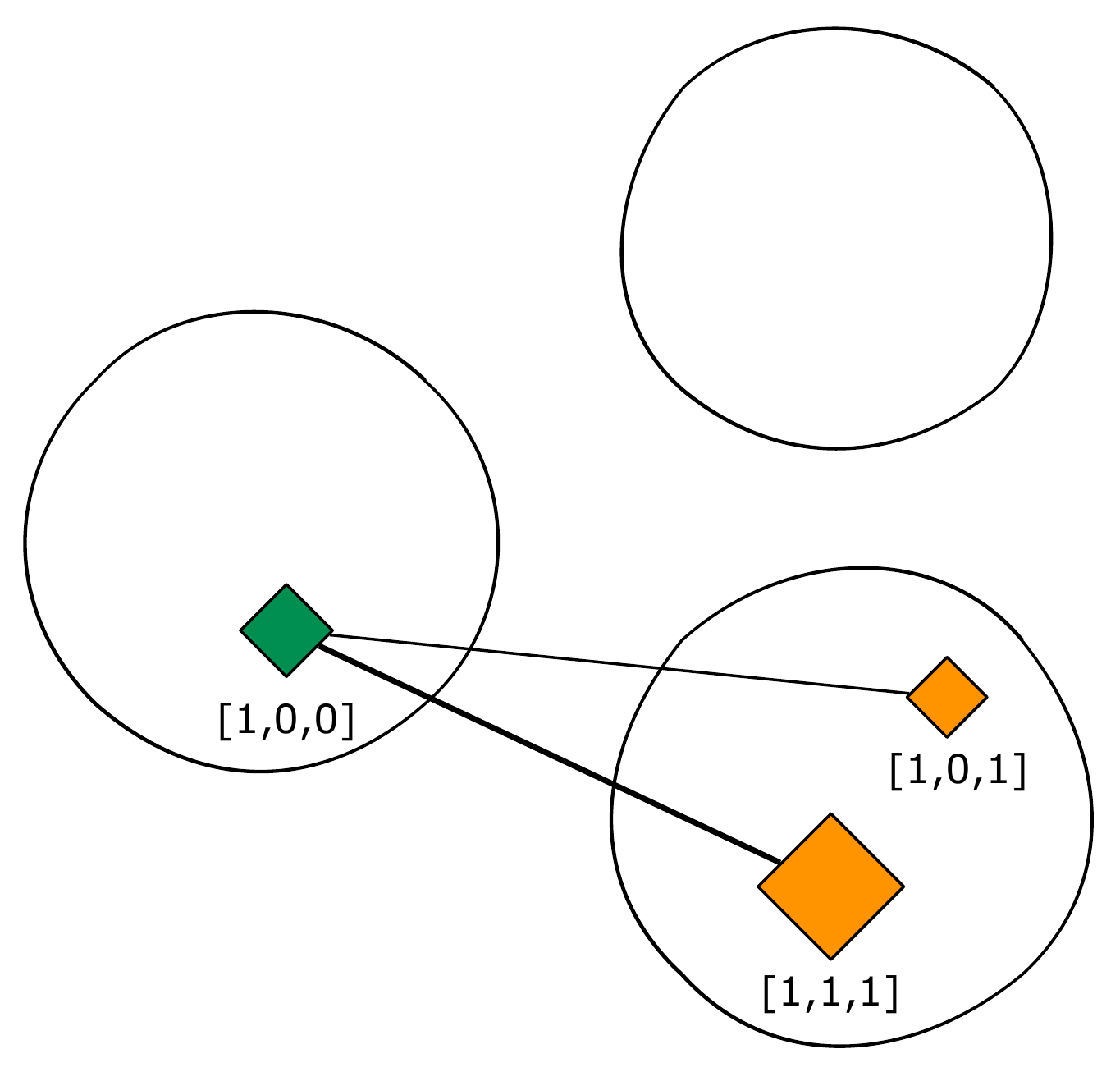}
    \label{fig:ecd_cut}
  }\quad
  \subfigure[\ECED]{%
    \includegraphics[width=.18\textwidth]{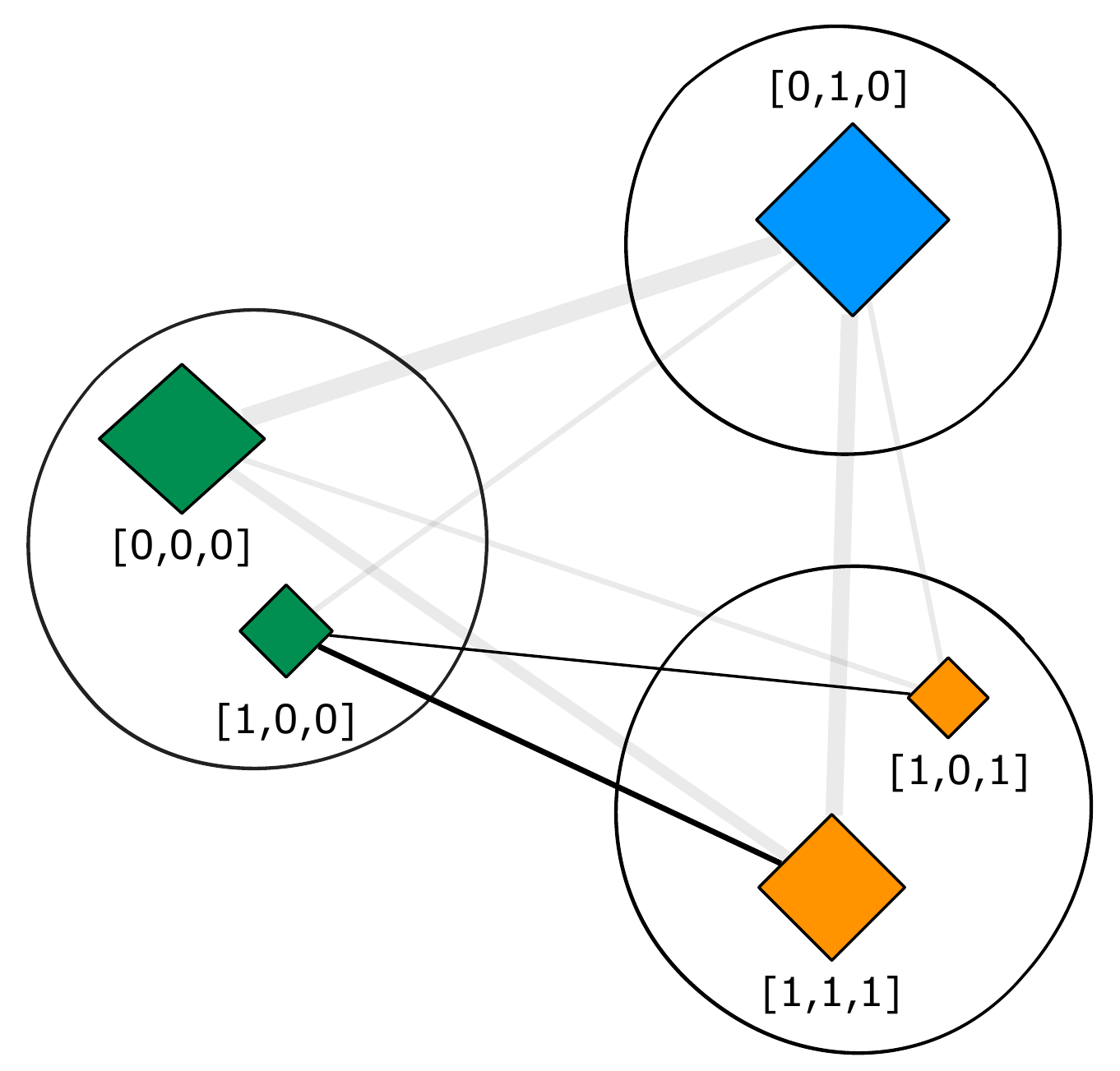}
    \label{fig:eced}
  }
  \caption{
    (a) shows an illustrative example of the medical diagnosis problem. In (b), we initialize \ECT, by drawing edges between all pairs of root-causes (diamonds) that are mapped into different treatments (circles). In (c), we run \ECT and remove all the edges incident to root-causes $\hiddenvar_2 [0,0,0]$ and $\hiddenvar_5 [0,1,0]$ if we observe $X_1=1$. (d) \ECED, instead, discounts the edge weights accordingly.} \label{fig:example_eced}
\vspace{-3mm}
\end{figure*}

Note that computing the optimal policy for Problem~(\ref{eq:drd}) is intractable in general. 
When $\delta = 0$, this problem reduces to the equivalence class determination problem \citep{golovin10near, bellala2010extensions}. Here, the target variables are referred to as \emph{equivalence classes}, 
since each $\targetvar \in \targetVarDom$ corresponds to a subset of root-causes in $\hiddenVarDom$ that (equivalently) share the same ``action''.

If tests are noise-free, i.e., $\forall e,~\Pr{\Test_e \given \Hiddenvar} \in \{0,1\}$, this problem can be solved near-optimally by the \emph{equivalence class edge cutting} (\ECT) algorithm \citep{golovin10near}. As is illustrated in \figref{fig:example_eced}, \ECT employs an edge-cutting strategy based on a weighted graph $G = (\hiddenVarDom, E)$, where vertices represent root-causes, and edges link root-causes that we want to distinguish between. Formally, $E \triangleq \{\{ \hiddenvar, \hiddenvar' \}: r(\hiddenvar)\neq r(\hiddenvar')\}$ consists of all (unordered) pairs of root-causes corresponding to different target values (see \figref{fig:ecd}). We define a weight function $w: E \rightarrow \NonNegativeReals$ by $w(\{\hiddenvar, \hiddenvar'\}) \triangleq \Pr{\hiddenvar} \cdot \Pr{\hiddenvar'}$, i.e., as the product of the probabilities of its incident root-causes.
We extend the weight function on sets of edges $E' \subseteq E$, as the sum of weight of all edges $\{\hiddenvar,\hiddenvar'\}\in E'$, i.e., $w(E') \triangleq \sum_{\{\hiddenvar, \hiddenvar'\}\in E'}w(\{\hiddenvar, \hiddenvar'\})$. 

Performing test $e \in \Testset$ with outcome $\test_e$ is said to ``\emph{cut}'' an edge, if at least one of its incident root-causes is inconsistent with $x_e$ (See \figref{fig:ecd_cut}). Denote $E(\test_e) \triangleq \{\{ \hiddenvar, \hiddenvar' \}\in E: \Pr{\test_e\given \hiddenvar} = 0~\vee~\Pr{\test_e\given \hiddenvar'} = 0\}$ as the set of edges cut by observing $\test_e$. The \ECT objective (which is greedily maximized per iteration of \ECT), is then defined as the total weight of edges cut by the current partial observation $\psi_\pi$: $ \ecobj (\psi_\pi) \triangleq w\Big( \bigcup_{(e,x_e)\in \psi_\pi} E(x_e) \Big).$


The \ECT objective function is \emph{adaptive submodular}, and \emph{strongly adaptive monotone} \citep{golovin10near}.
Formally, let $\psi_1, \psi_2 \in 2^{\Testset \times \obsDom}$ be two partial realizations of tests' outcomes. We call $\psi_1$ a {\em subrealization} of $\psi_2$, denoted as  $\psi_1\preceq\psi_2$, if every test seen by $\psi_1$ is also seen by $\psi_2$, and $\Pr{\psi_2 \given \psi_1} > 0$. A function $f:2^{\Testset \times \obsDom} \rightarrow \mathbb{R}$ is called {\em adaptive submodular} w.r.t.~a distribution $\mathbb{P}$, if for any $\psi_1\preceq\psi_2$ and any $\Obs_e$ it holds that $\Delta(\Obs_e \given \psi_1) \geq \Delta(\Obs_e \given \psi_2)$, where $\Delta(\Obs_e\given \psi) := \expctover{\obs_e}{f(\psi \cup \{(e, x_e)\})- f(\psi) \given \psi}$ (i.e., ``adding information earlier helps more''). Further, function $f$ is called {\em strongly adaptively monotone} w.r.t.~$\mathbb{P}$, if for all $\psi$, test $e$ not seen by $\psi$, and $\obs_e \in \obsDom$, it holds that $ f(\psi) \leq f(\psi \cup \{(e, x_e)\})$ (i.e., ``adding new information never hurts''). For sequential decision problems satisfying adaptive submodularity and strongly adaptive monotonicity, the policy that greedily, upon having observed $\psi$, 
selects the test $e^*\in\argmax_e \Delta(X_e\given\psi)$, 
is guaranteed to attain near-minimal cost \citep{golovinjair2011}. 


In the noisy setting, however, we can no longer attain 0 error probability (or equivalently, cut all the edges constructed for \ECT), even if we exhaust all tests. A natural approach to solving Problem~\eqrefcus{eq:drd} for $\delta>0$ would be to pick tests greedily maximizing the expected reduction in the error probability $\errorprob$. However, this objective is not adaptive submodular; in fact, as we show in the supplemental material, such policy can perform arbitrarily badly if there are complementaries among tests, i.e., the gain of a set of tests can be far better than sum of the individual gains of the tests in the set. 
Therefore, motivated by the \ECT objective in the noise-free setting, we would like to optimize a surrogate objective function which captures the effect of noise, while being amenable to greedy optimization. 


\section{The \eced~Algorithm}\label{sec:eced}
We now introduce \eced~for Bayesian active learning under correlated noisy tests, which strictly generalizes \ECT to the noisy setting, while preserving the near-optimal guarantee.

\vspace{-1mm}
\paragraph{\ECT with Bayesian Updates on Edge Weights}
In the noisy setting, the test outcomes are not necessarily deterministic given a root-cause, i.e., $\forall \hiddenvar,~ \Pr{\Test_e \given \hiddenvar} \in [0,1]$. Therefore, one can no longer ``cut away'' a root-cause $\hiddenvar$ by observing $x_e$,  
as long as $\Pr{\Test_e = \test_e \given \hiddenvar} > 0$. In such cases, a natural extension of the edge-cutting strategy will be -- instead of cutting off edges -- to \emph{discount} the edge weights through Bayesian updates: 
After observing $x_e$, we can discount the weight of an edge $\{\hiddenvar, \hiddenvar'\}$, by multiplying the probabilities of its incident root-causes with the likelihoods of the observation\footnote{Here we choose \emph{not} to normalize the probabilities of $\hiddenvar, \hiddenvar'$ to their posterior probabilities. Otherwise, we can end up having 0 gain in terms of edge weight reduction, even if we perform a very informative test.}: $w(\{\hiddenvar, \hiddenvar'\} \given \test_e) := \Pr{\hiddenvar} \Pr{\hiddenvar'} \cdot \Pr{\test_e \given \hiddenvar} \Pr{\test_e \given \hiddenvar' } = \Pr{\hiddenvar, \test_e} \cdot \Pr{\hiddenvar', \test_e}$. This gives us a greedy policy that, at every iteration, picks the test that has the maximal expected reduction in total edge weight. 
We call such policy \ECT-Bayes. 
Unfortunately, as we demonstrate later in \secref{sec:exp}, this seemingly promising update scheme is not ideal for solving our problem: it tends to pick tests that are very noisy, which do not help facilitate differentiation among different target values. 
Consider a simple example with three root-causes distributed as $\Pr{\hiddenvar_1} = 0.2, \Pr{\hiddenvar_2} = \Pr{\hiddenvar_3} = 0.4$, and two target values $r(\hiddenvar_1) = r(\hiddenvar_2) = \targetvar_1, r(\hiddenvar_3) = \targetvar_2$. We want to evaluate two tests: (1) a purely noisy test $\Test_1$, i.e., $\forall \hiddenvar,~\Pr{\Test_1 = 1 \given \hiddenvar} = 0.5$, and (2) a noiseless test $\Test_2$ with $\Pr{\Test_2 = 1 \given \hiddenvar_1} = 1$ and $\Pr{\Test_2 = 1 \given \hiddenvar_2} = \Pr{\Test_2 = 1 \given \hiddenvar_3} = 0$. One can easily verify that by running \ECT-Bayes, one actually prefers $\Test_1$ (with expected reduction in edge weight 0.18, as opposed to 0.112 for $\Test_2$).

\vspace{-1mm}
\paragraph{The \eced~Algorithm}
The example above hints us on an important principle of designing proper objective functions for this task: as the noise rate increases, one must take reasonable precautions when evaluating the informativeness of a test, such that the undesired contribution by noise is accounted for. Suppose we have performed test $e$ and observed $x_e$. We call a root-cause $\hiddenvar$ to be ``consistent'' with observation $x_e$, if $x_e$ is the most likely outcome of $X_e$ given $\hiddenvar$ (i.e., $x_e \in \argmax_x\Pr{X_e = x \given \hiddenvar}$). Otherwise, we say $\hiddenvar$ is inconsistent. 
Now, instead of discounting the weight of all root-causes by the likelihoods $\Pr{X_e = x_e \given \hiddenvar}$ (as \ECT-Bayes does), we choose to discount the root-causes by the \emph{likelihood ratio}: $\lambda_{\hiddenvar,x_e} \triangleq \frac{\Pr{X_e = x_e \given \hiddenvar}}{\max_{x_e'} \Pr{X_e = x_e' \given \hiddenvar}}$. Intuitively, this is because we want to ``penalize'' a root-cause (and hence the weight of its incident edges), only if it is \emph{inconsistent} with the observation (See \figref{fig:eced}). When $x_e$ is consistent with root-cause $\hiddenvar$, then $\lambda_{\hiddenvar, x_e} = 1$ and we do not discount $\hiddenvar$; otherwise, if $x_e$ is inconsistent with $\hiddenvar$, we have $\lambda_{\hiddenvar, x_e} < 1$. When a test is not informative for root-cause $\hiddenvar$, i.e. $\Pr{X_e \given \hiddenvar}$ is uniform, then $\lambda_{\hiddenvar,e} = 1$, so that it neutralizes the effect of such test in terms of edge weight reduction.
Formally, given observations $\psi_\pi$, we define the value of observing $\test_e$ as the total amount of edge weight discounted: $\ecedgainPerOutcome(\test_e \given \psi_\pi) \triangleq \sum_{\{\hiddenvar, \hiddenvar'\}\in E} \Pr{\hiddenvar, \psi_\pi} \Pr{\hiddenvar', \psi_\pi} \cdot  (1-\lambda_{\hiddenvar,x_e}\lambda_{\hiddenvar',x_e})$. 

Further, we call test $e$ to be \emph{non-informative}, if its outcome does not affect the distribution of $\Hiddenvar$, i.e., $\forall~\hiddenvar,\hiddenvar'\in \hiddenVarDom$ and $\test_e \in \obsDom$, $\Pr{\Test_e = \test_e \given \hiddenvar} = \Pr{\Test_e = \test_e \given \hiddenvar'}$. Obviously, performing a non-informative test does not reveal any useful information of $\Hiddenvar$ (and hence $\Targetvar$). Therefore, we should augment our basic value function $\ecedgainPerOutcome$, such that the value of a non-informative test is 0. 
Following this principle, we define $\ecedgainPeroutcomeoffset(\test_e \given \psi_\pi) \triangleq \sum_{\{\hiddenvar, \hiddenvar'\}\in E} \Pr{\hiddenvar, \psi_\pi} \Pr{\hiddenvar', \psi_\pi} \cdot (1- \max_\hiddenvar\lambda_{\hiddenvar,x_e}^2)$,
as the \emph{offset} value for observing outcome $x_e$. It is easy to check that if test $e$ is non-informative, then it holds that $\ecedgainPerOutcome(\test_e \given \psi_\pi) - \ecedgainPeroutcomeoffset(\test_e \given \psi_\pi) = 0$ for all $x_e \in \obsDom$; otherwise $\ecedgainPerOutcome(\test_e \given \psi_\pi) - \ecedgainPeroutcomeoffset(\test_e \given \psi_\pi) \geq 0$.  This motivates us to use the following objective function:
\begin{align}
  \ecedgain(\Test_e \given \psi_\pi) \triangleq \expctover{\test_e}{\ecedgainPerOutcome(\test_e \given \psi_\pi) - \ecedgainPeroutcomeoffset(\test_e \given \psi_\pi)},\label{eq:ecedgain}
\end{align}
as the expected amount of edge weight that is effectively reduced by performing test $e$.
We call the algorithm that greedily maximizes $\ecedgain$ the \emph{Equivalence Class Edge Discounting} (\eced) algorithm, and present the pseudocode in \algref{alg:main}. 

\setlength{\textfloatsep}{14pt}
\begin{algorithm}[t!]
  \nl {\bf Input}: {$[\lambda_{\hiddenvar,\test}]_{\numrc\times \numtest}$ (or Conditional Probabilities $\Pr{\Test \given \Hiddenvar}$), Prior $\Pr{\Hiddenvar}$, Mapping $r: \hiddenVarDom \rightarrow \targetVarDom$}; \\
  \Begin{
    \nl  $\psi_\pi \leftarrow \emptyset$; \\
    \ForEach{$(\hiddenvar,\hiddenvar') \in E$}{
      \nl $w_{\hiddenvar,\hiddenvar'} \leftarrow \Pr{\hiddenvar}\Pr{\hiddenvar'}$; \\
    }
    \While{$\errorprob(\psi_\pi) > \delta$}
    {
      \vspace{-.5cm}
      \nl $e^* \leftarrow \argmax_{e} \mathbb{E}_{\test_e} \Big[\sum_{\{\hiddenvar, \hiddenvar'\}\in E} w_{\hiddenvar,\hiddenvar'}
      \big(\overbrace{1 - \lambda_{\hiddenvar,x_e}\lambda_{\hiddenvar',x_e}}^{\footnotesize \substack{\text{weight}\\\text{discounted}}} - \overbrace{(1- \textstyle {\max_{\hiddenvar} \lambda_{\hiddenvar,\test_e}^2})}^{\footnotesize \substack{\text{offset term} }} \big) \Big]$; \label{alg:ln:argmin}\\
      \nl Observe $x_{e^*}$; \quad $w_{\hiddenvar,\hiddenvar'} \leftarrow w_{\hiddenvar,\hiddenvar'} \cdot \Pr{x_{e^*} \given \hiddenvar}\Pr{x_{e^*} \given \hiddenvar'}$; \\
      \nl $\psi_\pi \leftarrow \psi_\pi \cup \{(e^*,x_{e^*})\}$; \label{algln:local_while_loop_end}\\
    }
    \nl {\bf Output}: $\targetvar^* = \argmax_\targetvar \Pr{\targetvar \given \psi_\pi}$.\\
  }
  \caption{The Equivalence Class Edge Discounting (\eced) Algorithm}\label{alg:main}
\end{algorithm}


Similar with \ECT, the efficiency (in terms of computation complexity as well as the query complexity) of \ECED depends on the number of root-causes. Let $\epsilon_{\theta,e} \triangleq 1-\max_x \Pr{X_e = x \given \hiddenvar}$ be the noise rate for test $e$. 
As our main theoretical result, we show that under the basic setting where test outcomes are \emph{binary}, and the test noise is \emph{independent} of the underlying root-causes (i.e., $\forall \hiddenvar \in \hiddenVarDom,~\epsilon_{\hiddenvar,e} = \epsilon_e$), 
\ECED is competitive with the optimal policy that achieves a lower error probability for Problem~\eqrefcus{eq:drd}:
\begin{theorem}\label{thm:mainresults}
  Fix $\delta \in (0, 1)$. To achieve expected error probability less than $\delta$, 
  it suffices to run \ECED for 
  $\bigO{\frac{k}{c_\varepsilon} \left(\log \frac{kn}{\delta}\log \frac{n}{\delta}\right)^2 }$ steps
  where $n \triangleq |\hiddenVarDom|$ denotes the number of root-causes, 
  $c_\varepsilon \triangleq \min_{e\in \Testset}(1 - 2\epsilon_e)^2$ characterizes the severity of noise, and
  $k\triangleq \cost \left( \OPT(\delta_\opt)\right)$ is the worst-case cost of the
  optimal policy that achieves expected error probability $\delta_\opt \triangleq \bigO{\frac{\delta}{\left(\log n \cdot \log (1/\delta)\right)^2}}$.
\end{theorem}
Note that a pessimistic upper bound for $k$ is the total number of tests $m$, and hence the cost of \ECED is at most $\bigO{\left(\log (mn/\delta) \log (n/\delta)\right)^2/c_\varepsilon}$ times the worst-case cost of the optimal algorithm, which achieves a lower error probability $\bigO{\delta/ (\log n \cdot \log (1/\delta))^2}$. Further, as one can observe, the upper bound on the cost of \ECED degrades as we increase the maximal noise rate of the tests. When $c_\varepsilon = 1$, we have $\epsilon_e = 0$ for all test $e$, and \ECED reduces to the \ECT algorithm. \thmref{thm:mainresults} implies that running \ECT for $\bigO{k \left(\log \frac{kn}{\delta} \log \frac{n}{\delta}\right)^2}$ in the noise-free setting is sufficient to achieve $\errorprob \leq \delta$. Finally, notice that by construction \ECED never selects any non-informative test. Therefore, we can always remove purely noisy tests (i.e., $\{e: \forall \hiddenvar,~\Pr{X_e = 1 \given \hiddenvar} =\Pr{X_e = 0 \given \hiddenvar}= 1/2\}$), so that $c_\varepsilon > 0$, and the upper bound in \thmref{thm:mainresults} becomes non-trivial.

\section{Theoretical Analysis}\label{sec:analysis}
\vspace{-1mm}

\paragraph{Information-theoretic Auxiliary Function}
We now present the main idea behind the proof of \thmref{thm:mainresults}. In general, an effective way to relate the performance (measured in terms of the gain in the target objective function) of the greedy policy to the optimal policy is by showing that, the \emph{one-step} gain of the greedy policy always makes effective progress towards approaching the cumulative gain of \OPT~\emph{over $k$ steps}. One powerful tool facilitating this is the \emph{adaptive submodularity} theory, which imposes a lower bound on the one-step greedy gain against the optimal policy, given that the objective function in consideration exhibits a natural diminishing returns condition. Unfortunately, in our context, the target function to optimize, i.e., the expected error probability of a policy, does not satisfy adaptive submodularity. Furthermore, it is nontrivial to understand how one can directly relate the two objectives: the \ECED objective of \eqrefcus{eq:ecedgain}, which we utilize for selecting informative tests, and the gain in the reduction of error probability, which we use for evaluating a policy.

We circumvent such problems by introducing surrogate functions, 
as a proxy to connect the \ECED objective $\ecedgain$ with the expected reduction in error probability $\errorprob$.
Ideally, we aim to find some auxiliary objective $\auxobj$, such that the tests with the maximal $\ecedgain$ also have a high gain in $\auxobj$; meanwhile, $\auxobj$ should also be comparable with the error probability $\errorprob$, such that minimizing $\auxobj$ itself is sufficient for achieving low error probability. 

We consider the function $\auxobj: 2^{\Testset\times \obsDom} \rightarrow \NonNegativeReals$, defined as
\begin{align}
  \auxobj(\psi) = \sum_{(\hiddenvar,\hiddenvar')\in E} \Pr{\hiddenvar \given \psi} \Pr{\hiddenvar' \given \psi} \cdot \log{ \frac{1}{\Pr{\hiddenvar \given \psi}\Pr{\hiddenvar' \given \psi}} } + c \sum_{\targetvar\in\targetVarDom} \bientropy{\Pr{\targetvar \given \psi}}. \label{eq:auxfunction}
\end{align}
Here $\bientropy{x} := - x \log x - (1-x) \log(1-x)$, and $c$ is a constant that will be made concrete shortly (in \lemref{lm:hi_eced_hi_faux}). 
Interestingly, we show that function $\auxobj$ is intrinsically linked to the error probability:
\begin{lemma}\label{lm:aux_vs_perr}
  We consider the auxiliary function defined in Equation~\eqrefcus{eq:auxfunction}. 
  Let $n \triangleq |\hiddenVarDom|$ be the number of root-causes, and $\errorprob^{\text{MAP}}(\psi)$ be the error probability given partial realization $\psi$. Then
  \begin{align*}
    2c \cdot \errorprob^{\text{MAP}}(\psi) \leq \auxobj(\psi) \leq  (3c + 4) \cdot \left(\bientropy{\errorprob^{\text{MAP}}(\psi)} + \errorprob^{\text{MAP}}(\psi) \log n\right).
  \end{align*}
\end{lemma}

Therefore, if we can show that by running \ECED, we can effectively reduce $\auxobj$, 
then by 
\lemref{lm:aux_vs_perr}, we can conclude that \ECED also makes significant progress in reducing the error probability $\errorprob^{\text{MAP}}$. 

\paragraph{Bounding the Gain w.r.t. the Auxiliary Function}
It remains to understand how \ECED~interacts with $\auxobj$. For any test $e$, we define $\auxgain(\Test_e \given \psi ) \triangleq \expctover{x_e}{\auxobj(\psi \cup \{e,x_e\}) - \auxobj(\psi) \given \psi}$ to be the expected gain of test $e$ in $\auxobj$. Let $\ectgainat{\psi}{\Test_e}$ denote the gain of test $e$ in the \ECT objective, assuming that the edge weights are configured according to the \emph{posterior distribution} $\Pr{\Hiddenvar \given \psi}$. Similarly, let $\ecedgainat{\psi}{\Test_e}$ denote the \ECED gain, if the edge weights are configured according to $\Pr{\Hiddenvar \given \psi}$.
We prove the following result:
\begin{lemma}\label{lm:hi_eced_hi_faux}
  Let $n = |\hiddenVarDom|$, $\numtar = |\targetVarDom|$, and $\nr$ be the noise rate associated with test $e\in\Testset$. Fix $\eta \in (0,1)$. We consider $\auxobj$ as defined in Equation~\eqrefcus{eq:auxfunction}, with 
  $c = 8 \left(\log (2n^2/\eta)\right)^2$. It holds that
  \begin{align*}
    \auxgain(\Test_e \given \psi) + c_{\eta,\nr} \geq \ecedgainat{\psi}{\Test_e}\cdot (1-\nr)^2/16 = c_\nr \ectgainat{\psi}{\Test_e},
  \end{align*}
  where $c_{\eta,\nr} = 2\numtar(1-2\nr)^2 \eta$, and $c_\nr \triangleq (1-2\nr)^2/16$.
\end{lemma}

\lemref{lm:hi_eced_hi_faux} indicates that the test being selected by \ECED can effectively reduce $\auxobj$.

\paragraph{Lifting the Adaptive Submodularity Framework}
Recall that our general strategy is to bound the one step gain in $\auxobj$ against the gain of an optimal policy. In order to do so, we need to show that our surrogate 
exhibits, to some extent, the diminishing returns property. By \lemref{lm:hi_eced_hi_faux} 
we can relate $\auxgain (\Test_{e}\given \psi_\pi)$, i.e., the gain in $\auxobj$ under the \emph{noisy} setting, to $\ectgainat{\psi}{\Test_{e}}$, 
i.e., the expected weight of edges \emph{cut} by the \ECT algorithm. 
Since $\ecobj$ is adaptive submodular, this allows us to lift the adaptive submodularity framework into the analysis. As a result, we can now relate the 1-step gain w.r.t. $\auxobj$ of a test selected by \ECED, to the cumulative gain w.r.t. $\ecobj$ of an optimal policy in the noise-free setting. Further, observe that the \ECT objective at $\psi$ satisfies: 
\begin{equation}
  \ectat{\psi}{} := \sum_{\targetvar} \Pr{\targetvar \given \psi} \left(1 - \Pr{\targetvar \given \psi}\right) \stackrel{(a)}{\geq} 1 - \max_{\targetvar} \Pr{\targetvar \given \psi} = \errorprob^{\text{MAP}}(\psi).
  \label{eq:ec2_vs_errorprob}
\end{equation}
Hereby, step (a) is due to the fact that the error probability of a MAP estimator always lower bounds that of a stochastic estimator (which is drawn randomly according to the posterior distribution of $Y$).
Suppose we want to compare \ECED against an optimal policy $\OPT$. 
By adaptive submodularity, we can relate the 1-step gain of \ECED in $\ectat{\psi}{}$ to the cummulative gain of $\OPT$. Combining Equation~\eqrefcus{eq:ec2_vs_errorprob} with \lemref{lm:aux_vs_perr} and \lemref{lm:hi_eced_hi_faux}, we can bound the 1-step gain in $\auxobj$ of \ECED against the $k$-step gain of $\OPT$, and consequently bound the cost of \ECED against $\OPT$ for Problem~\ref{eq:drd}. We defer a more detailed proof outline and the full proof to the supplemental material.

\vspace{-1mm}
\section{Experimental Results}\label{sec:exp}
\vspace{-2mm}
We now demonstrate the performance of \ECED on two real-world problem instances: a Bayesian experimental design task intended to distinguish among economic theories of how people make risky decisions, and an active preference learning task via pairwise comparisons. Due to space limitations, we defer a third case study on pool-based active learning to the supplemental material.
\vspace{-2mm}
\paragraph{Baselines.} The first baseline we consider is \ECT-Bayes, which uses the Bayes' rule to update the edge weights when computing the gain of a test (as described in \secref{sec:eced}). Note that after observing the outcome of a test, both \ECED and \ECT-Bayes update the posteriors on $\Hiddenvar$ and $\Targetvar$ according to the Bayes' rule; the only difference is that they use different strategies when \emph{selecting} a test. We also compare with two commonly used sequential information gathering policies: Information Gain ({\sc IG}), and Uncertainty Sampling ({\sc US}), which consider picking tests that greedily maximizing the reduction of entropy over the target variable $\Targetvar$, and root-causes $\Hiddenvar$ respectively. Last, we consider myopic optimization of the decision-theoretic value of information (\VOI) \citep{Howard66}. In our problems, the \VOI policy greedily picks the test maximizing the expected reduction in prediction error in $\Targetvar$.

\begin{figure*}[t]
  \centering
  \subfigure[\textsl{Risk Choice Theory}]{%
    \includegraphics[width=.323\textwidth]{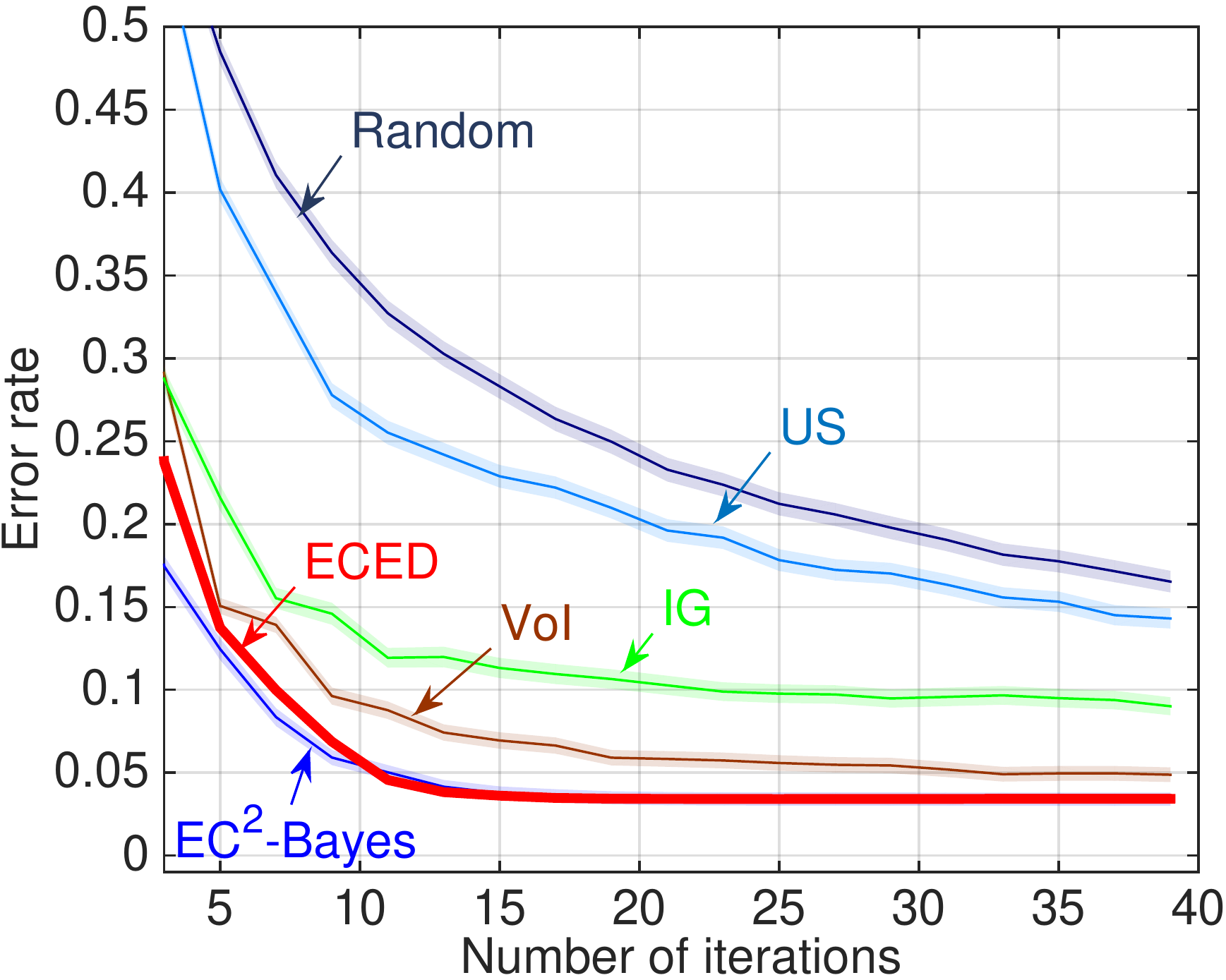}
    \label{fig:theroy_error}}
  \subfigure[\textsl{MovieLens}]{%
    \includegraphics[width=.323\textwidth]{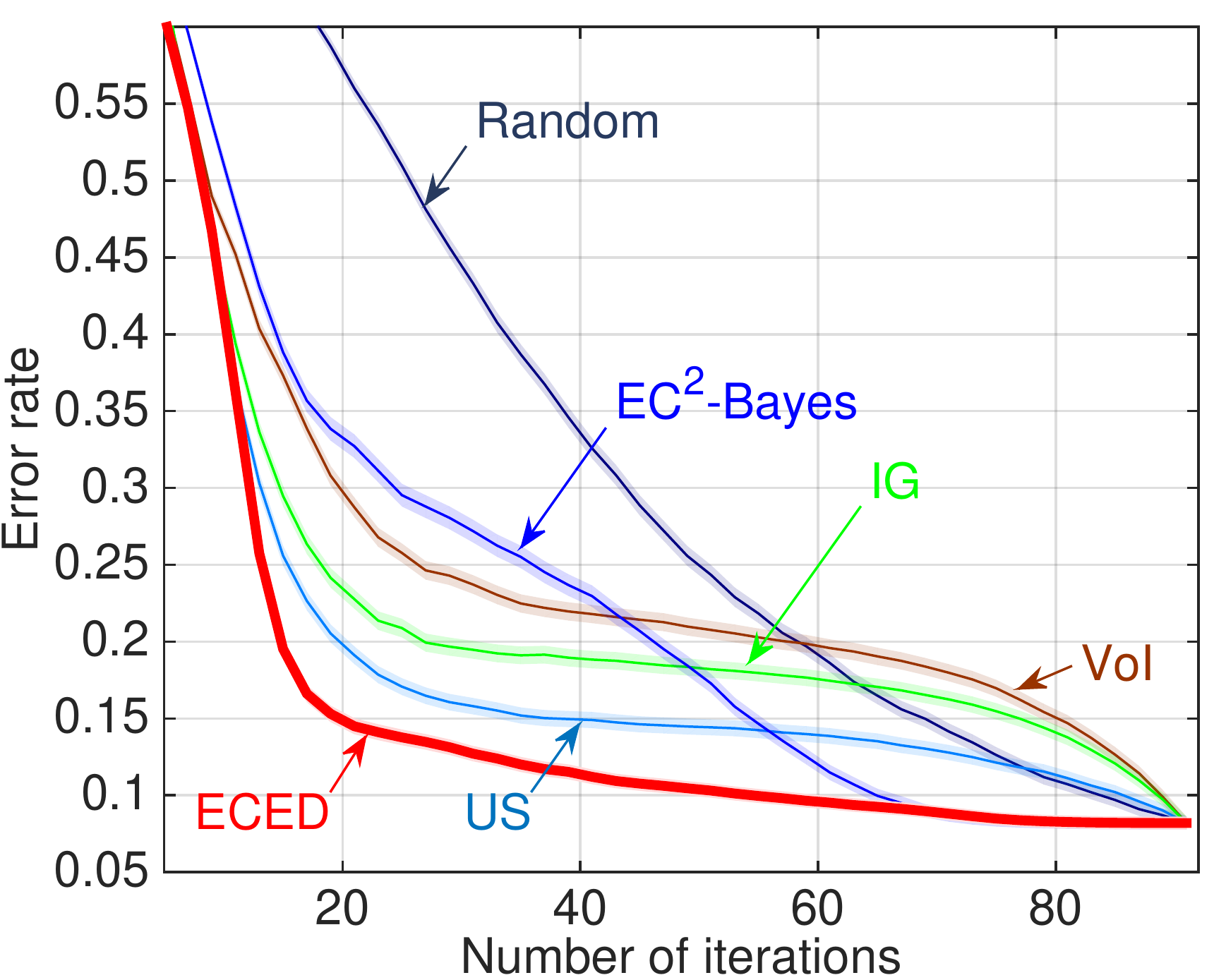}
    \label{fig:movie_error_20regs}}
  \subfigure[\textsl{MovieLens} - Varying noise]{%
    \includegraphics[width=.323\textwidth]{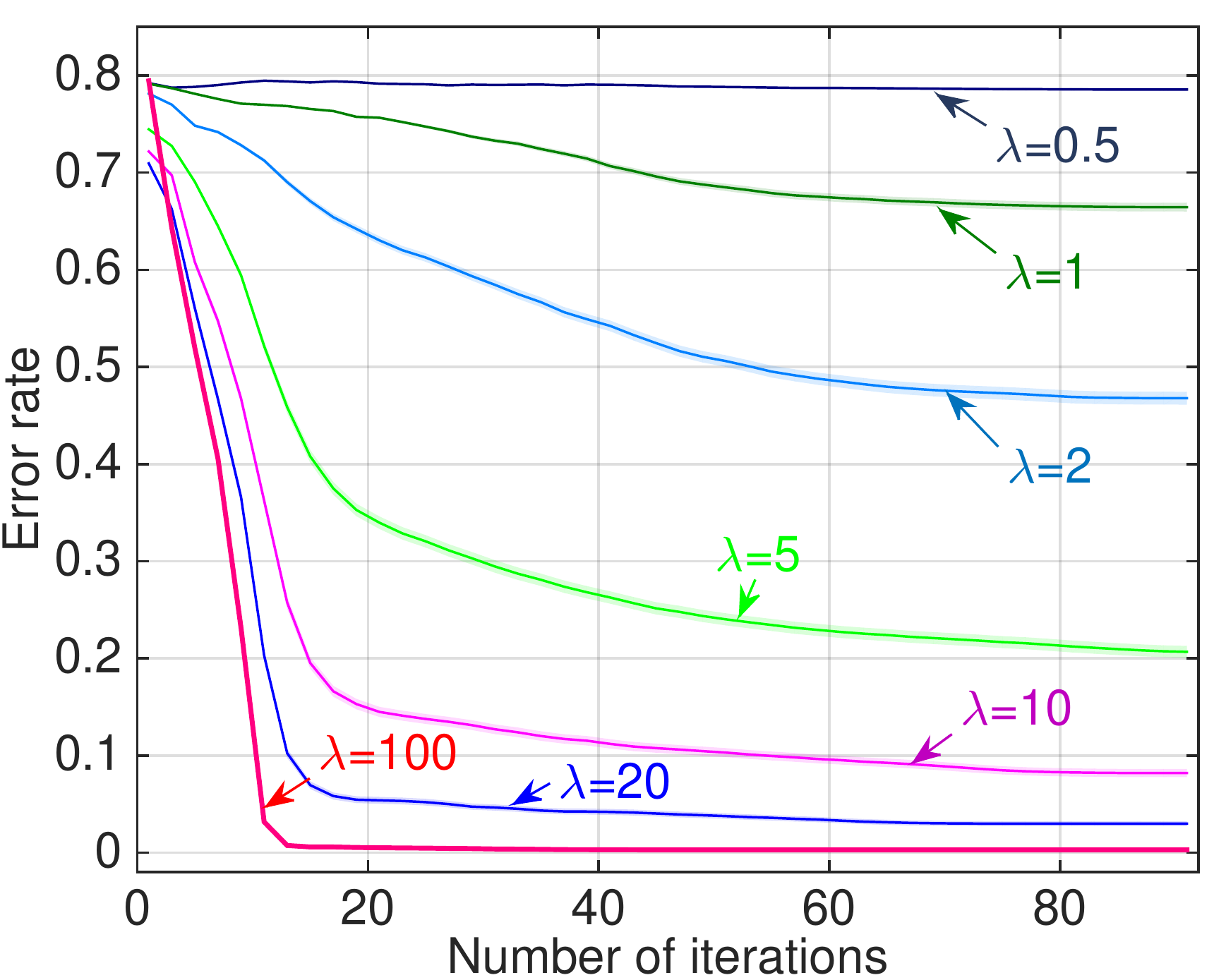}
    \label{fig:movie_error_vs_lambda}}
\vspace{-3mm}
  \caption{Experimental results: \ECED outperforms most baselines on both data sets.}
  \label{fig:movie_result}
\end{figure*}
\vspace{-1mm}
\subsection{Preference Elicitation in Behavioral Economics}
\vspace{-2mm}
We first conduct experiments on a Bayesian experimental design task, which intends to distinguish among economic theories of how people make risky decisions. Several theories have been proposed in behavioral economics to explain how people make decisions under risk and uncertainty. We test \ECED on six theories of subjective valuation of risky choices \citep{wakker2010prospect, tversky1992prospect, sharpe1964}, namely
(1) {\it expected utility with constant relative risk aversion}, 
(2) {\it expected value}, 
(3) {\it prospect theory}, 
(4) {\it cumulative prospect theory}, 
(5) {\it weighted moments}, 
and (6) {\it weighted standardized moments}. 
Choices are between risky lotteries, 
i.e., known distribution over payoffs (e.g., the monetary value gained or lost).
A test $e \triangleq (L_1, L_2)$ is a pair of lotteries, and root-causes $\Hiddenvar$ correspond to parametrized theories that predict, for a given test, which lottery is preferable. The goal, is to adaptively select a sequence of tests to present to a human subject in order to distinguish which of the six theories best explains the subject's responses. We employ the same set of parameters used in \cite{ray2012bayesian} to generate tests and root-causes. In particular, we have generated $\sim$16K tests. 
Given root-cause $\hiddenvar$ and test $e = (L_1, L_2)$, one can compute the values of $L_1$ and $L_2$, denoted by $v_1$ and $v_2$. Then, the probability that root-cause $\hiddenvar$ favors $L_1$ is modeled as $\Pr{X_e = 1 \given \hiddenvar} = \frac{1}{1+\exp(-\lambda\cdot (v_1-v_2))}$.

\vspace{-2mm}
\paragraph{Results}
\figref{fig:theroy_error} demonstrates the performance of \ECED on this data set. 
The average error probability has been computed across 1000 random trials for all methods. We observe that \ECED and \ECT-Bayes have similar behavior on this data set; however, the performance of the US algorithm is much worse. 
This can be explained by the nature of the data set: it has more concentrated distribution over $\Hiddenvar$, but not $\Targetvar$. Therefore, since tests only provide indirect information about $\Targetvar$ through $\Hiddenvar$, what the uncertainty sampling scheme tries to optimize is actually $\Hiddenvar$, hence it performs quite poorly.

\vspace{-1mm}
\subsection{Preference Learning via Pairwise Comparisons}
\vspace{-1mm}
The second application considers a comparison-based movie recommendation system, 
which learns a user's movie preference (e.g., the favorable genre) by sequentially showing her pairs of candidate movies, and letting her choose which one she prefers. We use the {\sl MovieLens 100k} dataset \citep{Herlocker:1999:AFP:312624.312682}, which consists of a matrix of 1 to 5 ratings of 1682 movies from 943 users, and adopt the experimental setup proposed in \cite{chen15submodular}. In particular, we extract movie features by computing a low-rank approximation of the user/rating matrix of the {\sl MovieLens 100k} dataset through singular value decomposition (SVD). We then simulate the target ``categories'' $\Targetvar$ that a user may be interested by partitioning the set of movies into $\numtar$ (non-overlapping) clusters in the Euclidean space. A root-cause $\Theta$ corresponds to user's favorite movie, and tests $e$'s are given in the form of movie pairs, i.e., $e\triangleq (m_a,m_b)$, where $a$ and $b$ are embeddings of movie $m_a$ and $m_b$ in Euclidean space. Suppose user's movie is represented by $\hiddenvar$, then test $e$ is realized as $1$ if $a$ is closer to $y$ than $b$, and $0$ otherwise. We simulate the effect of noise by $\Pr{X_e = 1 \given \hiddenvar} = \frac{1}{1+\exp(-\lambda\cdot (d(m_a, \hiddenvar)-d(m_b, \hiddenvar)))}$. where $d(\cdot, \cdot)$ is the distance function, and $\lambda$ control the level of noise in the system. 

\vspace{-1mm}
\paragraph{Results}
\figref{fig:movie_error_20regs} shows the performance of \ECED compared other baseline methods, when we fix the size of $\targetVarDom$ to be 20 and $\lambda$ to be 10. We compute the average error probability across 1000 random trials for all methods. We can see that \ECED consistently outperforms all other baselines. Interestingly, \ECT-Bayes performs poorly on this data set. This may be due to the fact that the noise level is still high, misguiding the two heuristics to select noisy, uninformative tests. \figref{fig:movie_error_vs_lambda} shows the performance of \ECED as we vary $\lambda$. When $\lambda = 100$, the tests become close to deterministic given a root-cause, and \ECED is able to achieve $0$ error with $\sim 12$ tests. As we increase the noise rate (i.e., decrease $\lambda$), it takes \ECED many more queries for the prediction error to converge. This is because with high noise rate, \ECED discounts the root-causes more uniformly, hence they are hardly informative in $Y$. This comes at the cost of performing more tests, and hence low convergence rate.



\vspace{-1mm}
\section{Related Work}
\vspace{-2mm}
\paragraph{Active learning in statistical learning theory.}
In most of the theoretical active learning literature (e.g., \citet{dasgupta2005analysis, hanneke2007bound, hanneke2014theory, balcan2015active}), sample complexity bounds have been characterized in terms of the structure of the hypothesis class, as well as additional distribution-dependent complexity measures (e.g., splitting index \citep{dasgupta2005analysis}, disagreement coefficient \citep{hanneke2007bound}, etc); In comparison, in this paper we seek {\em computationally-efficient} approaches that are {\em provably competitive} with the optimal policy. Therefore, we do not seek to bound how the optimal policy behaves, and hence we make no assumptions on the hypothesis class.
\vspace{-1mm}
\paragraph{Persistent noise vs non-persistent noise.} 
\looseness -1 If tests can be repeated with i.i.d. outcomes, the noisy problem can then be effectively reduced to the noise-free setting \citep{kaariainen2006active, karp2007noisy, nowak2009noisy}. While the modeling of non-persistent noise may be appropriate in some settings (e.g., if the noise is due to measurement error), it is often important to consider the setting of \emph{persistent noise} in many other applications. In many applications, repeating tests are impossible, or repeating a test produces identical outcomes. For example, it could be unrealistic to replicate a medical test for practical clinical treatment. Despite of some recent development in dealing with persistent noise in simple graphical models \citep{chen15mis} and strict noise assumptions \citep{golovin10near}, more general settings, which we focus on in this paper, are much less understood.

\vspace{-1mm}
\section{Conclusion}
\vspace{-2mm}

We have introduced \eced, which strictly generalizes the \ECT algorithm, for solving practical Bayesian active learning and experimental design problems with correlated and noisy tests. We have proved that \ECED enjoys strong theoretical guarantees, 
by introducing an analysis framework that draws upon adaptive submodularity and information theory. We have demonstrated the compelling performance of \ECED on two (noisy) problem instances, including an active preference learning task via pairwise comparisons, and a Bayesian experimental design task for preference elicitation in behavioral economics. We believe that our work makes an important step towards understanding the theoretical aspects of complex, sequential information gathering problems, and provides useful insight on how to develop practical algorithms to address noise.


\subsubsection*{Acknowledgments}

This work was supported in part by ERC StG 307036, a Microsoft Research Faculty Fellowship, and a Google European Doctoral Fellowship.

\bibliographystyle{icml2016}
{\small
  \bibliography{noisyec2_main}
}

\appendix

\section{Table of Notations Defined in the Main Paper}
We summarize the notations used in the main paper in Table~\ref{tb:main-notations}.
\begin{table}[h]
  \centering
  \caption{
    A reference table of notations used in the main paper
  }\label{tb:main-notations}
  \begin{tabular}[t]{|l|p{12cm}|} \hline
    $\Targetvar$ & random variable encoding the value of the target variable \\ \hline
    $\targetVarDom$ & domain of the target variable \\ \hline
    $\targetvar$ & value of $\Targetvar$ \\ \hline
    $\Hiddenvar$ & random variable encoding the root-cause \\ \hline
    $\hiddenVarDom$ & the ground set / domain of root-causes\\ \hline
    $\hiddenvar$ & root-cause\\ \hline
    $r$ & $\Hiddenvar \rightarrow \Targetvar$, a function that maps a root-cause to a target value \\ \hline
    $\Testset$ & the ground set of tests \\ \hline
    $\numtest$ & $|\Testset|$, number of tests \\ \hline
    $e$ & test\\ \hline
    $\Test_e$ & random variable encoding the test outcome\\ \hline
    $\test_e$ & observed test outcome\\ \hline
    $\numtar$ & $|\targetVarDom|$, number of possible target values \\ \hline
    $\numrc$ & $|\hiddenVarDom|$, number of root-causes \\ \hline
    $\policy$ & policy, i.e., a (partial) mapping from observation vectors to tests \\ \hline
    $\Psi$ & random variable encoding a partial realization, i.e., set of test-observation pairs \\ \hline
    $\psi_\pi$ & the partial realization, i.e., set of test-observation pairs observed by running policy $\policy$ \\ \hline
    $\delta$ & tolerance of prediction error \\ \hline
    $\errorprob^\MAP({\psi})$ & error probability (of a MAP decoder), having observed partial realization $\psi$\\ \hline
    $\errorprob(\policy)$ & $\expctover{\psi_\pi}{\errorprob^\MAP({\psi_\pi})}$, expected error probability by running policy $\policy$\\ \hline
    $\OPT$ & optimal policy for Problem~\eqrefcus{eq:drd} \\ \hline
    $G$ & $G=(\hiddenVarDom, E)$, the (weighted) graph constructed for the \ECT algorithm \\ \hline
    $w(\{\hiddenvar, \hiddenvar'\})$ & weight of edge $\{\hiddenvar,\hiddenvar'\}\in E$ in the \ECT graph $G$  \\ \hline
    $\ecobj$ & the \ECT objective function, with $\ecobj(\emptyset) := \sum_{\hiddenvar,\hiddenvar'\in E}\Pr{\hiddenvar} \Pr{\hiddenvar'}$.\\ \hline
    $\ectat{\psi}$ & the \ECT objective function, with $\ectat{\psi}(\emptyset) := \sum_{\hiddenvar,\hiddenvar'\in E}\Pr{\hiddenvar\given \psi}\Pr{\hiddenvar' \given \psi}$. 
    \\ \hline
    $\lambda_{\hiddenvar,e}$ & discount coefficient of root-cause $\theta$, used by \ECED when computing $\ecedgain$. \\ \hline
    $\epsilon_{\hiddenvar,e}$ & $1-\argmax_e \Pr{X_e = x_e}$, the noise rate for a test $e$\\ \hline
    $\ecedgainPerOutcome(\test_e \given \psi)$ & the ``basic'' component in the \ECED gain by observing $x_e$, having observed $\psi$ \\ \hline
    $\ecedgainPeroutcomeoffset(\test_e \given \psi)$ & the ``offset'' component in the \ECED gain by observing $x_e$, having observed $\psi$ \\ \hline
    $\ecedgain(\Test_e \given \psi)$ & 
                                       the \ECED gain which is myopically optimized at each iteration of the \ECED algorithm \\ \hline
    $\ecedgainat{\psi}{\Test_e}$ & suppose we have observed $\psi$, and re-initialize the \ECT graph so that the total edge weight is $\ectat{\psi}(\emptyset)$. Then, $\ectgainat{\psi}{\Test_e}$ is the expected reduction in edge weight, by performing test $e$ and \emph{discounting} edges' weight according to \ECED. It is the re-normalized version of $\ecedgain(\test_e \given \psi)$, i.e., $\ecedgainat{\psi}{\Test_e} = \ecedgain(\test_e \given \psi) / \Pr{\psi}^2$. \\ \hline
    $\ectgainat{\psi}{\Test_e}$ & the expected gain in $\ectat{\psi}$ by performing test $e$, and \emph{cutting} edges weight according to \ECT. It can be interpreted as $\ecedgainat{\psi}{\Test_e}$, as if the test's outcome is noise-free, i.e., $\forall \hiddenvar,~\epsilon_{\hiddenvar, e} = 0$. \\ \hline
    $\auxobj$ & the auxiliary function defined in Equation~\eqrefcus{eq:auxfunction}\\ \hline
    $\eta$ & parameter of $\auxobj$ (see Equation~\eqrefcus{eq:auxfunction}, \lemref{lm:hi_eced_hi_faux}). It is only used for analysis.\\ \hline
    $c$ & $8 \left(\log (n^2/\eta)\right)^2$, parameter of $\auxobj$. It is only used for the analysis of \ECED. \\ \hline
    $\auxgain(\Test_e \given \psi)$ & the expected gain in $\auxobj$ by performing test $e$, conditioning on partial realization $\psi$ \\ \hline
    $c_{\eta,\epsilon},c_{\epsilon}$ & constants required by \lemref{lm:hi_eced_hi_faux} \\ \hline
    $\lambda$ & parameter controlling the error rate of tests (see \secref{sec:exp}) \\ \hline

  \end{tabular}
\end{table}

\clearpage


\section{The Analysis Framework}\label{sec:supp_proofs}

In this section, we provide the proofs of our theoretical results in full detail. Recall that for the theoretical analysis, we study the basic setting where test outcomes are \emph{binary}, and the test noise is \emph{independent} of the underlying root-causes (i.e., given a test $e$, the noise rate on the outcome of test $e$ is only a function of $e$, but not a function of $\hiddenvar$). 

\subsection{The Auxiliary Function and the Proof Outline}
The general idea behind our analysis, is to show that by running \ECED, the one-step gain in learning the value of the target variable is significant, compared with the cumulative gain of an optimal policy over $k$ steps (see \figref{fig:gain_1vsk}).

\begin{figure*}[h]
  \centering
  \includegraphics[width=.95\textwidth]{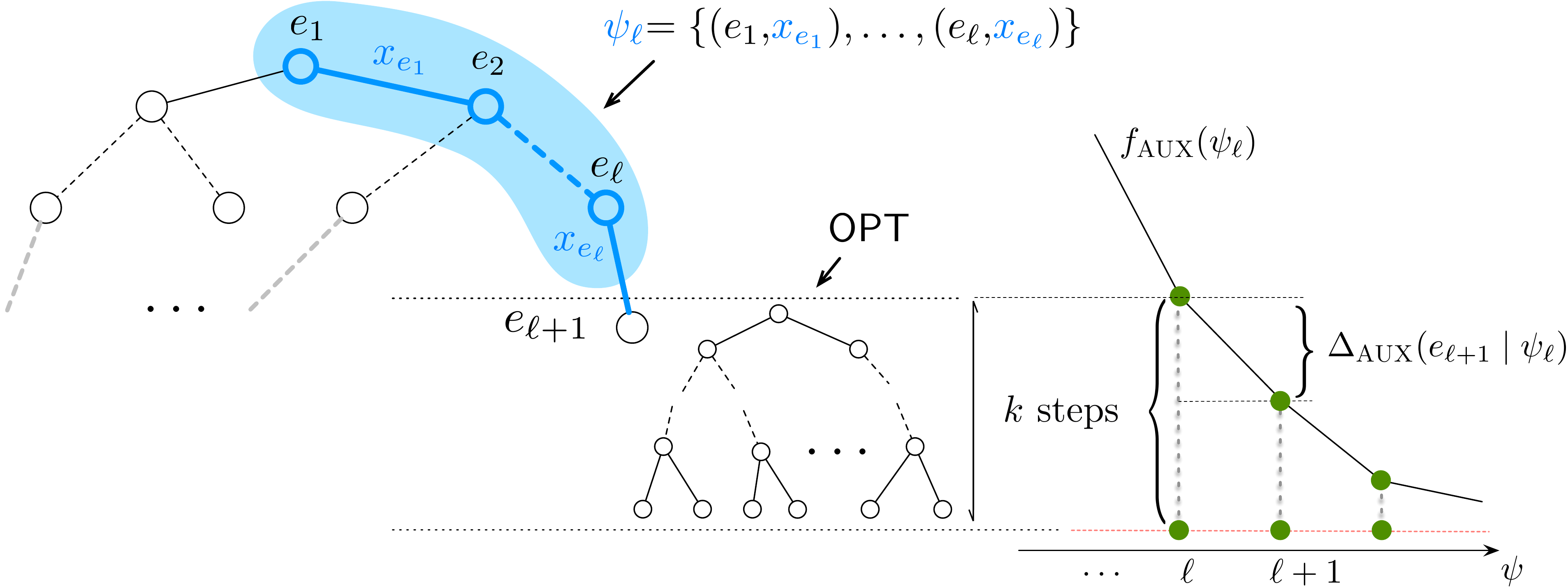}
  \caption{ On the left, we demonstrate a sequential policy in the form of its decision tree representation. Nodes represent tests selected by the policy, and edges represent outcomes of tests. At step $\ell$, a policy maps partial realization $\psi_\ell = \{(e_1, x_{e_1}), \dots, (e_\ell, x_{e_\ell}) \}$ to the next test $e_{\ell+1}$ to be performed. In the middle, we demonstrate the tests selected by an optimal policy \OPT of length $k$. On the right, we illustrate the change in the auxiliary function as \ECED selects more tests. Running \OPT at any step of execution of \ECED will make $\auxobj$ below some threshold (represented by the red dotted line). The key idea behind our proof, is to show that the greedy policy \ECED, at each step, is making effective progress in reducing the expected prediction error (in the long run), compared with \OPT.}\label{fig:gain_1vsk}
\end{figure*}

In Appendix \secref{sec:countereg}, we show that if tests are greedily selected to optimize the (reduction in) expected prediction error, we may end up failing to pick some tests, which have negligible immediate gain in terms of error reduction, but are very informative in the long run. \ECED bypasses such an issue by selecting tests that maximally distinguish root-causes with different target values. In order to analyze \ECED, we need to find an auxiliary function that properly tracks the ``progress'' of the \ECED algorithm; meanwhile, this auxiliary function should allow us to connect the heuristic by which we select tests (i.e., $\ecedgain$), with the target objective of interest (i.e., the expected prediction error $\errorprob$).

We consider the auxiliary function defined in Equation \eqrefcus{eq:auxfunction}. For brevity, we suppress the dependence of $\psi$ where it is unambiguous. Further, we use $p_{\hiddenvar}$, $p_{\hiddenvar'}$, and $p_{\targetvar}$ as shorthand notations for $\Pr{\hiddenvar \given \psi}$, $\Pr{\hiddenvar' \given \psi}$ and $\Pr{\targetvar \given \psi}$. Equation \eqrefcus{eq:auxfunction} can be simplified as
\begin{align}
  \auxobj = \sum_{(\hiddenvar,\hiddenvar')\in E} p_{\hiddenvar}p_{\hiddenvar'} \log{ \frac{1}{p_{\hiddenvar}p_{\hiddenvar'}} } + c \sum_{\targetvar\in \targetVarDom} \bientropy{p_{\targetvar}} \label{eq:def_auxobj_additive}
\end{align}

\begin{figure*}[h]
  \centering
  \includegraphics[width=.9\textwidth]{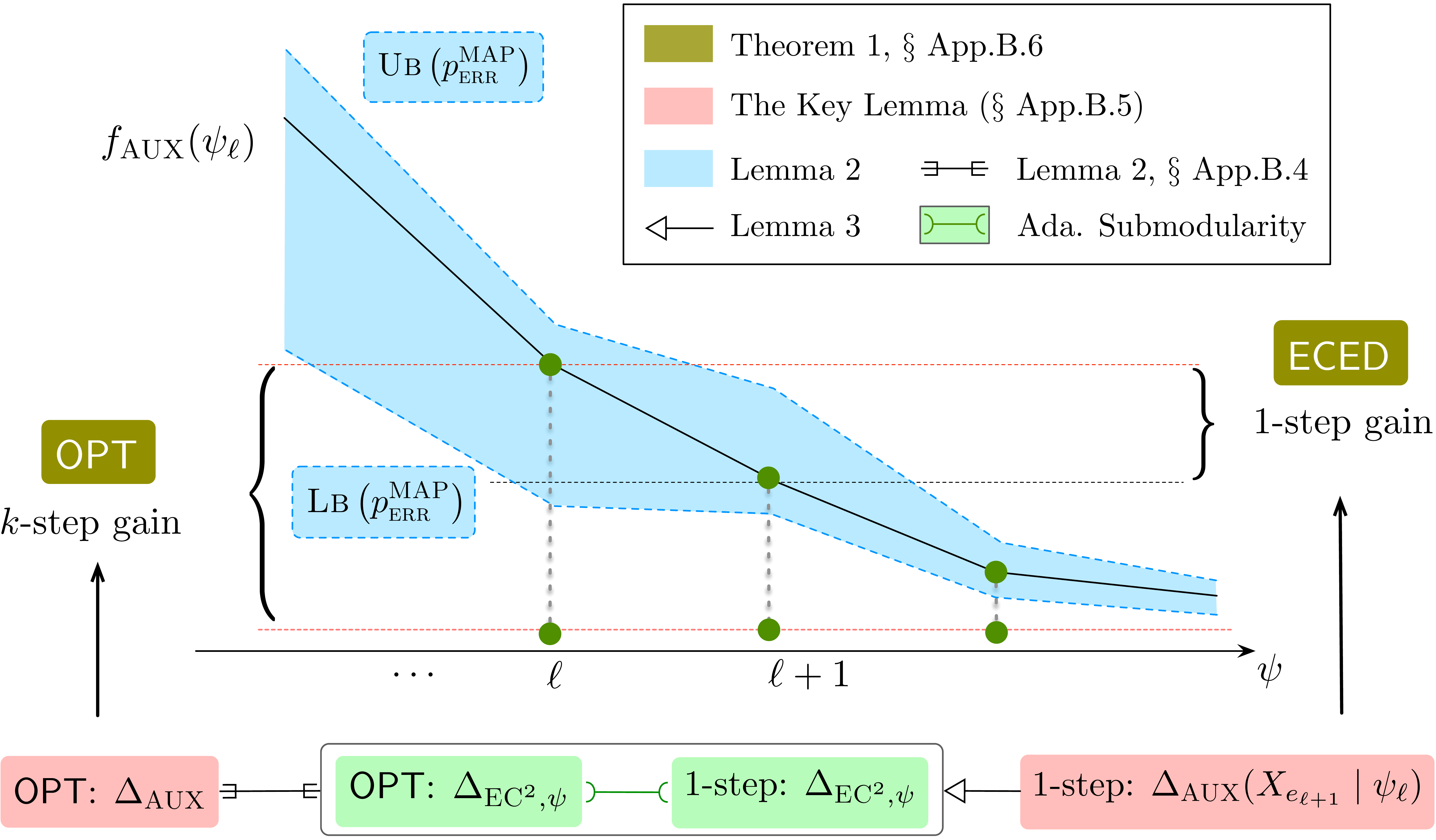}
  \caption{The proof outline.
  }\label{fig:proof_outline}
\end{figure*}

We illustrate the outline of our proofs in \figref{fig:proof_outline}. Our goal is to bound the cost of \ecedgraph against the cost of \optgraph (\thmref{thm:mainresults}; proof provided in Appendix \secref{sec:proofmainresults}). As we have explained earlier, our strategy is to relate the one-step gain of \ECED \ecedgaingraph with the gain of \OPT in $k$-steps \optkgaingraph (Appendix \secref{sec:proofkeylemma}, \lemref{lm:keylemma_onestepgain_supp}). To achieve that, we divide our proof into three parts:
\begin{enumerate}
\item We show that the auxiliary function $\auxobj$ is closely related with the target objective function $\errorprob$. More specifically, we provide both an upper bound \fvsperrubgraph and a lower bound \fvsperrlbgraph of $\auxobj$ in \lemref{lm:aux_vs_perr}, and give the detailed proofs in Appendix \secref{sec:proof_f_vs_perr}.
\item To analyze the one-step gain of \ECED, we introduce another intermediate auxiliary function: For a test $e_{\ell+1}$ chosen by \ECED, we relate its one-step gain in the auxiliary function \auxgainggraph, to its one-step gain in the \ECT objective \ectgainggraph (\auxvsecarrowgraph \lemref{lm:hi_eced_hi_faux}, detailed proof provided in Appendix \secref{sec:proof_auxvsec2}). The reason why we introduce this step is that the \ECT objective is \emph{adaptive submodular} \adasubmgraph, by which we can relate the 1-step gain of a greedy policy \ectgainggraph to an optimal policy \optgainectggraph.
\item To close the loop, it remains to connect the gain of an optimal policy \OPT in the \ECT objective function \optgainectggraph, with the gain of \OPT in the auxiliary function \optkgaingraph. We show how to achieve this connection (\auxvsecforkgraph) in Appendix \secref{sec:proof_nyvsnl}, by relating \optgainectggraph to the expected reduction in prediction error, and further in \secref{sec:proofkeylemma}, by applying the upper bound \fvsperrubgraph provided in \secref{sec:proof_f_vs_perr}.
\end{enumerate}

To make the proof more accessible, we insert the annotated color blocks from \figref{fig:proof_outline} (i.e., \fvsperrubgraph, \fvsperrlbgraph, \auxgainggraph, \ectgainggraph, \optgainectggraph, \optkgaingraph, etc), into the subsequent subsections in Appendix \secref{sec:supp_proofs}, so that readers can easily relate different parts of this section to the proof outline. Note that we only use these annotated color blocks for positioning the proofs, and hence readers can ignore the notations, as it may slightly differ from the ones used in the proof.





\subsection{ Proof of \lemref{lm:aux_vs_perr}: Relating $\auxobj$ to $\errorprob$}\label{sec:proof_f_vs_perr}

Define $\errorst(\psi) \triangleq \sum_{\targetvar \in \targetVarDom} \Pr{\targetvar \given \psi} \left(1 - \Pr{\targetvar \given \psi}\right)$ as the prediction error of a \emph{stochastic estimator} upon observing $\psi$, i.e., the probability of mispredicting $\targetvar$ if we make a random draw from $\Pr{\Targetvar \given \psi}$. We show in \lemref{lm:stoc_vs_map_supp} that $\errormap(\psi)$ is within a constant factor of $\errorst(\psi)$:

\begin{lemma}\label{lm:stoc_vs_map_supp}
  Fix $\psi$, it holds that $\errormap(\psi) \leq \errorst(\psi) \leq 2\errormap(\psi).$
\end{lemma}
\begin{proof}[Proof of Lemma~\ref{lm:stoc_vs_map_supp}]
  We can always lower bound $\errorst$ by $\errormap$, since by definition, $\errormap(\psi) =  1 - \max_{\targetvar} \Pr{\targetvar \given \psi} = \sum_{y\in \targetVarDom} \Pr{\targetvar \given \psi} \cdot  \left(1 - \max_{\targetvar} \Pr{\targetvar \given \psi}\right) \leq \sum_{\targetvar \in \targetVarDom} \Pr{\targetvar \given \psi} \left(1 - \Pr{\targetvar \given \psi}\right) = \errorst(\psi)$.

  To prove the second part, we write $p_{\targetvar_i} = \Pr{\Targetvar = \targetvar_i \given \psi}$ for all $y_i \in \targetVarDom$. 
  W.l.o.g., we assume $p_{\targetvar_1} \geq p_{\targetvar_2} \geq \dots \geq p_{\targetvar_\numtar}$. Then $\errormap = 1 - p_{\targetvar_1}$. We further have
  \begin{align*}
    2\errorprob^\text{MAP} &= 2(1 - p_{\targetvar_1}) = 2(\sum_{i=2}^\numtar p_{\targetvar_i}) = 2(\sum_{i=1}^\numtar p_{\targetvar_i}) (\sum_{i=2}^\numtar p_{\targetvar_i}) =  2(p_{\targetvar_1} + \sum_{i=2}^\numtar p_{\targetvar_i}) (\sum_{i=2}^\numtar p_{\targetvar_i}) \\
                           &\geq 2p_{\targetvar_1} (\sum_{i=2}^\numtar p_{\targetvar_i}) + (\sum_{i=2}^\numtar p_{\targetvar_i})^2 \\
                           &\geq \sum_{i\neq j}^\numtar p_{\targetvar_i} p_{\targetvar_j} = \sum_i p_{\targetvar_i} (1-p_{\targetvar_i}) = \errorst
  \end{align*}
\end{proof}

Now, we provide lower and upper bounds of the second term in the RHS of Equation \eqrefcus{eq:def_auxobj_additive}:
\begin{lemma}\label{lm:faux_additiveterm_bound_supp}
  $2 \errormap \leq \sum_{\targetvar \in \targetVarDom} \bientropy{p_\targetvar} \leq 3 (\bientropy{\errormap} + \errormap \log n)$.
\end{lemma}
\begin{proof}[Proof of Lemma~\ref{lm:faux_additiveterm_bound_supp}]

  We first prove the inequality on the left. Expanding the middle term involving the binary entropy of $p_\targetvar$, we get
  \begin{align*}
    \sum_{\targetvar\in \targetVarDom} \bientropy{p_\targetvar} &= \sum_{\targetvar \in \targetVarDom} \left(p_\targetvar \log \frac{1}{p_\targetvar} + (1-p_\targetvar) \log \frac{1}{1- p_\targetvar} \right) \\
                                                                &\stackrel{\text{(a)}}{\geq} \frac{2}{\ln 2} \sum_{\targetvar \in \targetVarDom} p_\targetvar (1-p_\targetvar) \\
                                                                &\geq 2 \errorst \stackrel{\textrm{\lemref{lm:stoc_vs_map_supp}}}{\geq} 2 \errormap
  \end{align*}
  Here, step (a) is by inequality $\ln x \geq 1 - 1/x$ for $x\geq 0$.

  To prove the second part, we first show in the following that $\sum_\targetvar (1-p_\targetvar) \log \frac{1}{1- p_\targetvar} \leq 2 \sum_\targetvar p_\targetvar \log \frac{1}{p_\targetvar} $.

  W.l.o.g., we assume that the probabilities $p_\targetvar$'s are in decreasing order, i.e., $p_{\targetvar_1} \geq p_{\targetvar_2} \geq \dots \geq p_{\targetvar_\numtar}$. Observe that if $p_\targetvar \in [0,1/2]$, then $(1-p_\targetvar) \log \frac{1}{1- p_\targetvar} \leq p_\targetvar \log \frac{1}{p_\targetvar}$. Consider the following two cases:
  \begin{enumerate}
  \item $p_{\targetvar_1} \leq 1/2$. In this case, we have $\sum_\targetvar (1-p_\targetvar) \log \frac{1}{1- p_\targetvar} \leq \sum_\targetvar p_\targetvar \log \frac{1}{p_\targetvar} $.
  \item $p_{\targetvar_1} > 1/2$. Since $\sum_{i>1} p_{\targetvar_i} = 1 - p_{\targetvar_1}$, we have
    \begin{align*}
      \sum_i (1-p_{\targetvar_i}) \log \frac{1}{1- p_{\targetvar_i}} &= (1-p_{\targetvar_1}) \log \frac{1}{1- p_{\targetvar_1}} + \sum_{i>1} (1-p_{\targetvar_i}) \log \frac{1}{1- p_{\targetvar_i}}\\
                                                                     &=\sum_{i>1} p_{\targetvar_i} \log \frac{1}{\sum_{i>1} p_{\targetvar_i}} + \sum_{i>1} (1-p_{\targetvar_i}) \log \frac{1}{1- p_{\targetvar_i}}\\
                                                                     &\leq \sum_{i>1} p_{\targetvar_i} \log \frac{1}{p_{\targetvar_i}} + \sum_{i>1} (1-p_{\targetvar_i}) \log \frac{1}{1- p_{\targetvar_i}} \\
                                                                     &\leq \sum_{i>1} p_{\targetvar_i} \log \frac{1}{p_{\targetvar_i}} + \sum_{i>1} p_{\targetvar_i} \log \frac{1}{p_{\targetvar_i}} \\
                                                                     &\leq 2 \sum_{i>0} p_{\targetvar_i} \log \frac{1}{p_{\targetvar_i}}
    \end{align*}
  \end{enumerate}
  Therefore,
  \begin{align}
    \label{eq:aux_addidive_term_bient}
    \sum_{\targetvar\in \targetVarDom} \bientropy{p_\targetvar} \leq 3 \sum_{i>0} p_{\targetvar_i} \log \frac{1}{p_{\targetvar_i}} = 3 \entropy{\Targetvar}.
  \end{align}
  Furthermore, by Fano's inequality (in the absence of conditioning), we know that $\entropy{\Targetvar} \leq \bientropy{\errormap}  + \errormap \log(|\targetVarDom|-1)$. Combining with Equation~\eqrefcus{eq:aux_addidive_term_bient} we get
  \begin{align*}
    \sum_\targetvar \bientropy{p_\targetvar} \leq 3 \entropy{\Targetvar} \leq 3 \left(\bientropy{\errormap} + \log(|\targetVarDom|-1)\right) \stackrel{(b)}{\leq} 3  \left(\bientropy{\errormap} + \log(n)\right)
  \end{align*}
  where in (b) we use the fact that $t = |\targetVarDom| \leq |\hiddenVarDom| = \numrc$, since $\Targetvar = r(\Hiddenvar)$ is a function of $\Hiddenvar$. Hence it completes the proof.
\end{proof}

Next, we bound the first term on the RHS of Equation \eqrefcus{eq:def_auxobj_additive}, i.e., $\sum_{\{\hiddenvar,\hiddenvar'\}\in E} p_{\hiddenvar}p_{\hiddenvar'} \log{ \frac{1}{p_{\hiddenvar}p_{\hiddenvar'}} }$, against $\errormap$:
\begin{lemma}\label{lm:aux_vs_perr_part1_supp}
  $\sum_{\{\hiddenvar,\hiddenvar'\}\in E} p_{\hiddenvar}p_{\hiddenvar'} \log{ \frac{1}{p_{\hiddenvar}p_{\hiddenvar'}} } \leq 2(\bientropy{\errorst} + \errorst \log n)$.
\end{lemma}
\begin{proof}[Proof of Lemma~\ref{lm:aux_vs_perr_part1_supp}]
  We can expand the LHS as
  \begin{align}
    \text{LHS} &= - \sum_{\hiddenvar'} p_{\hiddenvar'} \sum_{\hiddenvar: r(\hiddenvar) \neq r(\hiddenvar')} p_{\hiddenvar} (\log{ p_{\hiddenvar} } + \log{p_{\hiddenvar'} }) \nonumber \\
               &= - 2 \sum_{\hiddenvar'} p_{\hiddenvar'} \sum_{\hiddenvar: r(\hiddenvar) \neq r(\hiddenvar')} p_{\hiddenvar} \log{ p_{\hiddenvar} }  \nonumber \\
               &= - 2 \sum_{\targetvar \in \targetVarDom} \sum_{\hiddenvar': r(\hiddenvar')= y} p_{\hiddenvar'} \sum_{\hiddenvar: r(\hiddenvar) \neq y} p_{\hiddenvar}  \log{ p_{\hiddenvar} }   \nonumber \\
               &= 2 \sum_{\targetvar \in \targetVarDom} p_{\targetvar} (1-p_{\targetvar}) \sum_{\hiddenvar: r(\hiddenvar) \neq \targetvar} \frac{p_\hiddenvar}{1-p_{\targetvar}} \left(\log \frac{p_\hiddenvar}{1-p_{\targetvar}} + \log {(1-p_{\targetvar})} \right) \nonumber \\
               &= - 2\sum_{\targetvar \in \targetVarDom} p_{\targetvar} (1-p_{\targetvar}) \log (1 - p_{\targetvar}) + 2 \sum_{\targetvar \in \targetVarDom} p_{\targetvar} (1-p_{\targetvar}) \entropy{\left\{ \frac{p_\hiddenvar}{(1-p_{\targetvar})} \right\}_{\hiddenvar: r(\hiddenvar) \neq {\targetvar}}} \label{eq:huexpand}\\
               &\leq 2 \sum_{\targetvar \in \targetVarDom} p_{\targetvar} \bientropy{1 - p_{\targetvar}} + 2 \sum_{\targetvar \in \targetVarDom} p_{\targetvar} (1-p_{\targetvar}) \entropy{\left\{ \frac{p_\hiddenvar}{(1-p_{\targetvar})} \right\}_{\hiddenvar: r(\hiddenvar) \neq {\targetvar}}} \nonumber
  \end{align}
  Since $\entropy{\left\{ \frac{p_\hiddenvar}{(1-p_{\targetvar})} \right\}_{\hiddenvar: r(\hiddenvar) \neq {\targetvar}}} \leq \log \numtar \leq \log \numrc$, we have
  \begin{align*}
    \text{LHS}
    &\leq 2 \sum_{\targetvar \in \targetVarDom} p_{\targetvar} \bientropy{1 - p_{\targetvar}} + 2 \sum_{\targetvar}  \underbrace{ p_{\targetvar} (1-p_{\targetvar})\log n }_{\errorst \log n} \\
    &\stackrel{\text{Jensen}}{\leq} 2 \bientropy{\sum_{\targetvar \in \targetVarDom} p_{\targetvar} (1 - p_{\targetvar})} + 2 \errorst \log n \\
    &= 2 \left( \bientropy{\errorst} + \errorst \log n \right).
  \end{align*}
  which completes the proof.
\end{proof}

Now, we are ready to state the upper bound \fvsperrubgraph and lower bound \fvsperrlbgraph of $\auxobj$.
\begin{proof}[Proof of \lemref{lm:aux_vs_perr}]
  Clearly, $\sum_{\{\hiddenvar,\hiddenvar'\}\in E} p_{\hiddenvar}p_{\hiddenvar'} \log{ \frac{1}{p_{\hiddenvar}p_{\hiddenvar'}} } \geq 0$. By \lemref{lm:faux_additiveterm_bound_supp} we get the lower bound:
  \begin{align*}
    \auxobj(\psi) \geq 2c \cdot \errorprob^{\text{MAP}}(\psi).
  \end{align*}
  Now assume $\errormap \leq 1/4$. By \lemref{lm:stoc_vs_map_supp} we know $\errorst \leq 2 \errormap$, and $\bientropy{\errorst} \leq \bientropy{2 \errormap} \leq 2\bientropy{\errormap} $. Combining with \lemref{lm:faux_additiveterm_bound_supp} and \lemref{lm:aux_vs_perr_part1_supp}, we get
  \begin{align*}
    \auxobj(\psi)
    &\leq  3c \cdot \left(\bientropy{\errormap} + \errormap \log n\right) + 4 \left(\bientropy{\errorst} + \errorst \log n\right) \\
    &\leq (3c+4) \cdot \left(\bientropy{\errormap} + \errormap \log n\right),
  \end{align*}
  which completes the proof.
\end{proof}

\subsection{Proof of \lemref{lm:hi_eced_hi_faux}: Bounding $\auxgain$ against $\ectgain$,  $\ecedgain$}\label{sec:proof_auxvsec2}
In this section, we analyze the 1-step gain in the auxiliary function \auxgainggraph, of any test $e\in \Testset$. By the end of this section, we will show that it is lowered bounded by the one-step gain in the \ECT objective \ectgainggraph.

Recall that we assume test outcomes are binary for our analysis, and in the following of this section, we assume the outcome $x_e$ of test $e$ is in $\{+,-\}$ instead of $\{0,1\}$, for clarity purposes.

\subsubsection{Notations and the Intermediate Goal}
\begin{figure*}[h]
  \centering
  \includegraphics[width=.8\textwidth]{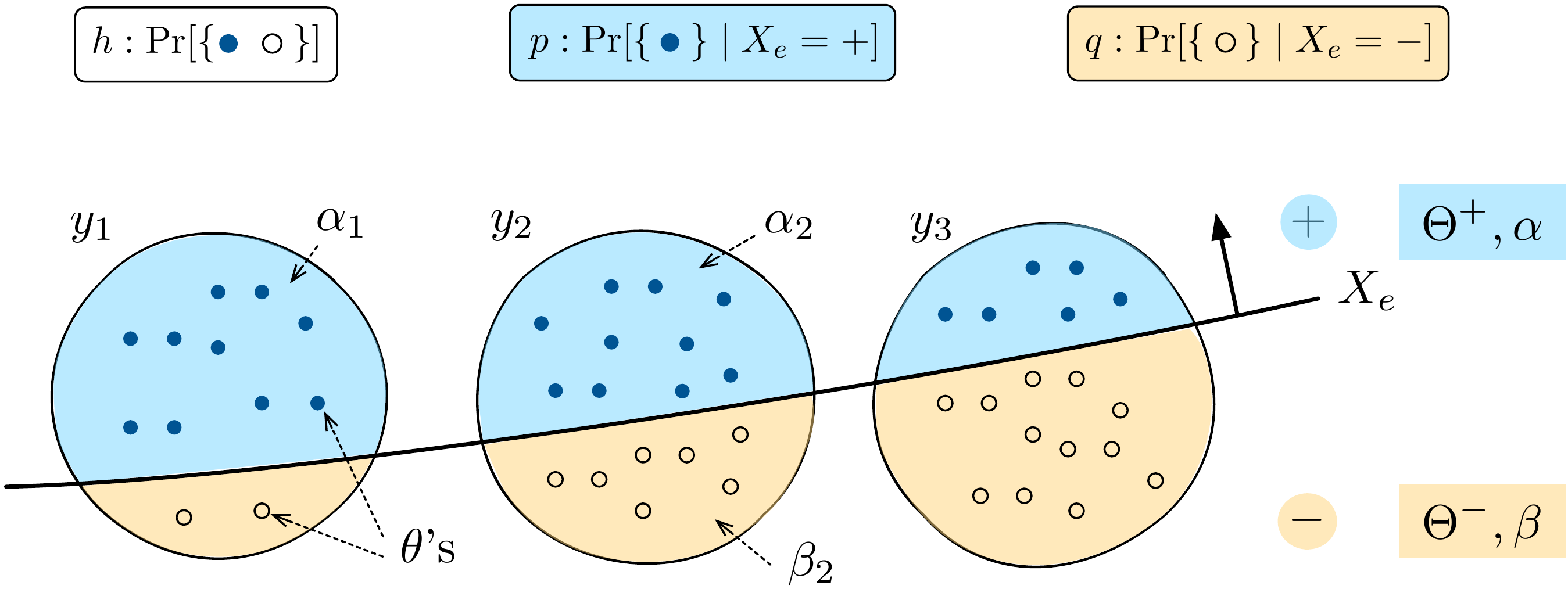}
  \caption{Performing binary test $e$ on $\Hiddenvar$ and $\Targetvar$. Dots represent root-causes $\hiddenvar \in \hiddenVarDom$, and circles represent values of the target variable $\targetvar \in \targetVarDom$. The favorable outcome of $\Test_e$ for the root-causes in solid dots are $+$; the favorable outcome for root-causes in hollow dots are $-$. We also illustrate the short-hand notations used in \secref{sec:proof_auxvsec2}. They are: $p,q$ (i.e., the posterior probability distribution over $\Targetvar$ and $\Hiddenvar$), $h$ (i.e., the prior distribution over $\Targetvar$ and $\Hiddenvar$) and $\alpha, \beta$ (i.e., the probability mass of solid and hollow dots, respectively, before performing test $e$).}\label{fig:illu_aux_vs_ec2_example}
\end{figure*}

\begin{table}[h]
  \centering
  \caption{Summary of notations introduced for the proof of \lemref{lm:hi_eced_hi_faux}}\label{tb:notations-auxgain-vs-ec2}
  \begin{tabular}[t]{|l|p{12cm}|} \hline
    $h$ & $\Pr{\cdot \given \psi}$, i.e., probability distribution on $\Hiddenvar$ and $\Targetvar$, before performing test $e$ \\ \hline
    $h_+$, $h_-$ & $\Pr{X_e = + \given \psi}$,$\Pr{X_e = - \given \psi}$  \\ \hline
    $p_\hiddenvar$, $p_\targetvar$ & $\Pr{\cdot \given \psi, X_e = +}$, i.e., probability distribution on $\Hiddenvar$ and $\Targetvar$ having observed $X_e = +$ \\ \hline
    $q_\hiddenvar$, $q_\targetvar$ & $\Pr{\cdot \given \psi, X_e = -}$, i.e., probability distribution on $\Hiddenvar$ and $\Targetvar$ having observed $X_e = -$ \\ \hline
    $\Hiddenvar^+$, $\Hiddenvar^-$ & set of positive / negative root-causes \\ \hline
    $\setiplus$, $\setiminus$ & set of positive / negative root-causes associated with target $y_i$\\ \hline
    $\alpha$, $\beta$ & total probability mass of positive / negative root-causes \\ \hline
    $\alpha_i$, $\beta_i$ & probability mass of positive / negative root-causes associated with target $y_i$ \\ \hline
    $\mu_i$, $\nu_i$ & $\alpha_i / \alpha$, $\beta_i / \beta$ (defined in \secref{sec:proof_auxgain_vs_ec2gain})\\ \hline
    $\hiddenvar \nsim \hiddenvar'$ & $r(\hiddenvar) \neq r(\hiddenvar')$, i.e., root-causes $\hiddenvar$ and $\hiddenvar'$ do not share the same target value \\ \hline
  \end{tabular}
\end{table}



For brevity, we first define a few short-hand notations to simplify our derivation. Let $p, q$ be two distributions on $\Hiddenvar$, and $h = \hplus p + \hminus q$ be the convex combination of the two, where $\hplus, \hminus \geq 0$ and $\hplus + \hminus = 1$.

In fact, we are using $p$ and $q$ to refer to the posterior distribution over $\Hiddenvar$ \emph{after} we observe the (noisy) outcome of some binary test $e$, and use $h$ to refer to the distribution over $\Hiddenvar$ \emph{before} we perform the test, i.e., $p_\hiddenvar \triangleq \Pr{\hiddenvar\given \Test_e=+}$, $q_\hiddenvar \triangleq \Pr{\hiddenvar\given \Test_e=-}$, and $h_\hiddenvar \triangleq \Pr{\hiddenvar} = \hplus p_\hiddenvar + \hminus q_\hiddenvar$, where $\hplus = \Pr{\Test_e = +}$ and $\hminus = \Pr{\Test_e = -}$. For $y_i \in \targetVarDom$, we use $p_{i} \triangleq \sum_{\hiddenvar: r(\hiddenvar)= \targetvar_i} p_\hiddenvar$ to denote the probability of $y_i$ under distribution $p$, and use $q_{i} \triangleq \sum_{\hiddenvar: r(\hiddenvar)= \targetvar_i} q_\hiddenvar$ to denote the probability of $y_i$ under distribution $q$.

Further, given a test $e$, we define $\setiplus$, $\setiminus$ to be the set of root-causes associated with target $y_i$, whose favorable outcome of test $e$ is $+$ (for $\setiplus$) and $-$ (for $\setiminus$). Formally,
\begin{align*}
  \setiplus &\triangleq \{\hiddenvar: r(\hiddenvar) = y_i \wedge \Pr{X_e=+ \given \hiddenvar} \geq 1/2\}\\
  \setiminus &\triangleq \{\hiddenvar: r(\hiddenvar) = y_i \wedge \Pr{X_e=+ \given \hiddenvar} < 1/2\}
\end{align*}
We then define $\Hiddenvar^+ \triangleq \bigcup_{i\in \{1,\dots,\numtar\}} \setiplus$, and $\Hiddenvar^- \triangleq \bigcup_{i\in \{1,\dots,\numtar\}} \setiminus$, to be the set of ``positive'' and ``negative'' root-causes for test $e$, respectively.

Let $\alpha_i, \beta_i$ be the probability mass of the root-causes in $\setiplus$ and $\setiminus$, i.e., $\alpha_i \triangleq \sum_{\substack{\targetvar \in \setiplus}} \Pr{\hiddenvar}$, and $  \beta_i \triangleq \sum_{\substack{\targetvar \in \setiminus}} \Pr{\hiddenvar}.$
We further define $\alpha \triangleq \sum_{\targetvar_i\in \targetVarDom} \alpha_i = \sum_{\hiddenvar \in \Hiddenvar^+} \Pr{\hiddenvar}$, and $\beta \triangleq \sum_{\targetvar_i \in \targetVarDom} \beta_\targetvar = \sum_{\hiddenvar \in \Hiddenvar^-} \Pr{\hiddenvar}$, then clearly we have $\alpha + \beta = 1$. See \figref{fig:illu_aux_vs_ec2_example} for illustration.

Now, we assume that test $e$ has error rate $\epsilon$. That is, $\forall \hiddenvar, ~\min \{\Pr{\Test_e = + \given \hiddenvar}, \Pr{\Test_e = - \given \hiddenvar}\} = \epsilon$. Then, by definition of $h_+$, $h_-$, $p_i$, $q_i$, $p_\hiddenvar$, $q_\hiddenvar$, it is easy to verify that
\begin{align}
  &\hplus = \alpha \nrbar + \beta \nr,
    \quad \hminus = \alpha \nr + \beta \nrbar \nonumber \\
  &p_i = \frac{\alpha_i \nrbar + \beta_i \nr}{\hplus},
    \quad q_i = \frac{\alpha_i \nr + \beta_i \nrbar}{\hminus} \nonumber \\
  &p_\hiddenvar = \frac{h_\hiddenvar \nrbar}{\hplus},
    \quad q_\hiddenvar = \frac{h_\hiddenvar \nr}{\hminus}, \qquad \text{if $\hiddenvar \in \setiplus $} \nonumber \\
  &p_\hiddenvar = \frac{h_\hiddenvar \nr}{\hplus},
    \quad q_\hiddenvar = \frac{h_\hiddenvar \nrbar}{\hminus}, \qquad \text{if $\hiddenvar \in \setiminus $} \label{eq:def_hpq}
\end{align}

For the convenience of readers, we summarize the notations provided above in \tableref{tb:notations-auxgain-vs-ec2}.

Given root-causes $\hiddenvar$ and $\hiddenvar'$, we use $\hiddenvar \nsim \hiddenvar'$ to denote that the values of the target variable $\Targetvar$ associated with root-causes $\hiddenvar$ and $\hiddenvar'$ are different, i.e., $r(\hiddenvar) \neq r(\hiddenvar')$.


We can rewrite the auxiliary function (as defined in Equation \eqrefcus{eq:auxfunction}) as follows:
\begin{align}
  \auxobj &= \sum_{\hiddenvar \nsim \hiddenvar'} h_\hiddenvar h_{\hiddenvar'} \log \frac{1}{h_\hiddenvar h_{\hiddenvar'}} + c \sum_{\targetvar_i \in \targetVarDom} \bientropy{h_i}. \nonumber
\end{align}

If by performing test $e$ we observe $\Test_e = +$, we have
\begin{align}
  \auxobj((e, +)) &= \sum_{\hiddenvar \nsim \hiddenvar'} p_\hiddenvar p_{\hiddenvar'} \log \frac{1}{p_\hiddenvar p_{\hiddenvar'}} + c \sum_{\targetvar_i \in \targetVarDom} \bientropy{p_i} \nonumber
\end{align}
otherwise, if we observe $\Test_e = -$,
\begin{align}
  \auxobj((e, -)) &= \sum_{\hiddenvar \nsim \hiddenvar'} q_\hiddenvar q_{\hiddenvar'} \log \frac{1}{q_\hiddenvar q_{\hiddenvar'}} + c \sum_{\targetvar_i \in \targetVarDom} \bientropy{q_i} \nonumber
\end{align}

Therefore, the expected gain (i.e., \auxgainggraph) of performing test $e$ is,
\begin{align}
  \auxgain &= \overbrace{\sum_{\hiddenvar \nsim \hiddenvar'} h_\hiddenvar h_{\hiddenvar'} \log \frac{1}{h_\hiddenvar h_{\hiddenvar'}} - \left( \hplus \sum_{\hiddenvar \nsim \hiddenvar'} p_\hiddenvar p_{\hiddenvar'} \log \frac{1}{p_\hiddenvar p_{\hiddenvar'}} + \hminus \sum_{\hiddenvar \nsim \hiddenvar'} q_\hiddenvar q_{\hiddenvar'} \log \frac{1}{q_\hiddenvar q_{\hiddenvar'}}\right)}^{\circled{1}} \nonumber \\
           & \qquad + c\underbrace{ \left( \sum_{\targetvar_i \in \targetVarDom} \bientropy{h_i} -  \left( \hplus \sum_{\targetvar_i \in \targetVarDom} \bientropy{p_i} + \hminus \sum_{\targetvar_i \in \targetVarDom} \bientropy{q_i}\right) \right)}_{\circled{2}} \label{eq:def_auxobj_additive_shorhand}
\end{align}
In the following, we derive lower bounds for the above two terms respectively.

\subsubsection{A Lower Bound on Term 1}
Let $g_{\hiddenvar,\hiddenvar'} \triangleq \hplus p_\hiddenvar p_{\hiddenvar'} + \hminus q_\hiddenvar q_{\hiddenvar'}$. Then, we can rewrite Term~\circled{1} as,
\begin{align}
  \text{Term~}\circled{1} &= \underbrace{\sum_{\hiddenvar \nsim \hiddenvar'} h_\hiddenvar h_{\hiddenvar'} \log \frac{1}{h_\hiddenvar h_{\hiddenvar'}} - \sum_{\hiddenvar\nsim \hiddenvar'} g_{\hiddenvar,\hiddenvar'} \log \frac{1}{g_{\hiddenvar,\hiddenvar'}}}_{\text{Part 1}} \nonumber \\
                          &\quad + \underbrace{\sum_{\hiddenvar\nsim \hiddenvar'} g_{\hiddenvar,\hiddenvar'} \log \frac{1}{g_{\hiddenvar,\hiddenvar'}} - \left( \hplus \sum_{\hiddenvar \nsim \hiddenvar'} p_\hiddenvar p_{\hiddenvar'} \log \frac{1}{p_\hiddenvar p_{\hiddenvar'}} + \hminus \sum_{\hiddenvar \nsim \hiddenvar'} q_\hiddenvar q_{\hiddenvar'} \log \frac{1}{q_\hiddenvar q_{\hiddenvar'}}\right)}_{\text{Part 2}} \label{eq:term1_expanded}
\end{align}

\paragraph{Part 1.}
We first provide a lower bound for part 1 of Equation~\eqrefcus{eq:term1_expanded}.

Notice that for concave function $f(x) = x\log\frac{1}{x}$ and $\delta < x$, it holds that $f(x) - f(x-\delta) \geq \delta \frac{\partial f}{\partial x}\at[\big]{x} = \delta (\log \frac{1}{x} - 1)$, then we get
\begin{align*}
  \sum_{\hiddenvar \nsim \hiddenvar'} h_\hiddenvar h_{\hiddenvar'} \log \frac{1}{h_\hiddenvar h_{\hiddenvar'}} - \sum_{\hiddenvar\nsim \hiddenvar'} g_{\hiddenvar,\hiddenvar'} \log \frac{1}{g_{\hiddenvar,\hiddenvar'}}
  \geq \sum_{\hiddenvar \nsim \hiddenvar'} \left(h_\hiddenvar h_{\hiddenvar'} - g_{\hiddenvar,\hiddenvar'}\right) \left(\log \frac{1}{h_\hiddenvar h_{\hiddenvar'}} - 1\right)
\end{align*}

Further, observe
\begin{align}
  h_\hiddenvar h_{\hiddenvar'} - g_{\hiddenvar,\hiddenvar'} &=  (\hplus p_\hiddenvar + \hminus q_{\hiddenvar}) (\hplus p_{\hiddenvar'} + \hminus q_{\hiddenvar'}) - (\hplus p_\hiddenvar p_{\hiddenvar'} + \hminus q_\hiddenvar q_{\hiddenvar'})\nonumber \\
                                                            &= (\hplus p_\hiddenvar + \hminus q_{\hiddenvar}) (p_{\hiddenvar'} + q_{\hiddenvar'} - \hminus p_{\hiddenvar'} - \hplus q_{\hiddenvar'}) - (\hplus p_\hiddenvar p_{\hiddenvar'} + \hminus q_\hiddenvar q_{\hiddenvar'}) \nonumber \\
                                                            &= \hplus \hminus p_{\hiddenvar'} q_\hiddenvar - \hplus \hminus p_{\hiddenvar'} p_\hiddenvar + \hplus \hminus p_\hiddenvar q_{\hiddenvar'} - \hminus \hplus q_{\hiddenvar'} q_\hiddenvar \nonumber \\
                                                            &= - \hplus \hminus (p_\hiddenvar - q_{\hiddenvar})(p_{\hiddenvar'} - q_{\hiddenvar'}) \nonumber 
\end{align}
Combining the above two equations gives us
\begin{align*}
  \text{Part 1} & \geq \sum_{\hiddenvar \nsim \hiddenvar'} - \hplus \hminus (p_\hiddenvar - q_{\hiddenvar})(p_{\hiddenvar'} - q_{\hiddenvar'})  \left(\log \frac{1}{h_\hiddenvar h_{\hiddenvar'}} - 1\right)
\end{align*}
For any root-cause pair $\{\hiddenvar,\hiddenvar'\}$ with $\hiddenvar\nsim\hiddenvar'$, and binary test $e$, there are only 4 possible combinations in terms of the root-causes' favorable outcomes. Namely,
\begin{enumerate}[itemsep=0mm,leftmargin=5mm]
\item 
  Both $\hiddenvar$ and $\hiddenvar'$ maps $x$ to $+$, i.e., $\hiddenvar\in \Hiddenvar^+ \wedge \hiddenvar' \in \Hiddenvar^+$. 

  We define such set of root-cause pairs with positive favorable outcomes as  $U_{(+,+)} \triangleq \{\{\hiddenvar, \hiddenvar'\}: \hiddenvar\in \Hiddenvar^+ \wedge \hiddenvar' \in \Hiddenvar^+\}$ (For other cases, we define $U_{(-,-)}$, $U_{(+,-)}$, $U_{(-,+)}$ in a similar way).

  In this case, we have
  \begin{align*}
    \sum_{\{\hiddenvar,\hiddenvar'\}\in U_{(+,+)}} - &\hplus \hminus (p_\hiddenvar - q_{\hiddenvar})(p_{\hiddenvar'} - q_{\hiddenvar'})  \left(\log \frac{1}{h_\hiddenvar h_{\hiddenvar'}} - 1\right) \\
    \stackrel{\text{Eq}~\eqrefcus{eq:def_hpq}}{=}& \sum_{\{\hiddenvar,\hiddenvar'\}\in U_{(+,+)}} - \hplus \hminus \left(\frac{h_\hiddenvar \nrbar}{\hplus} - \frac{h_\hiddenvar \nr}{\hminus} \right)\left(\frac{h_{\hiddenvar'} \nrbar}{\hplus} - \frac{h_{\hiddenvar'} \nr}{\hminus} \right)  \left(\log \frac{1}{h_\hiddenvar h_{\hiddenvar'}} - 1\right)\\
    =& \hplus \hminus \left(\frac{\hminus \nrbar - \hplus \nrbar}{\hplus \hminus} \right)^2 \sum_{\{\hiddenvar,\hiddenvar'\}\in U_{(+,+)}} -  h_\hiddenvar h_{\hiddenvar'} \left(\log \frac{1}{h_\hiddenvar h_{\hiddenvar'}} - 1\right)\\
    =& \frac{\beta^2 \left(1 -2 \nr\right)^2}{\hplus \hminus} \sum_{\{\hiddenvar,\hiddenvar'\}\in U_{(+,+)}} -  h_\hiddenvar h_{\hiddenvar'} \left(\log \frac{1}{h_\hiddenvar h_{\hiddenvar'}} - 1\right) \\
    =& \frac{\beta^2 \left(1 -2 \nr\right)^2}{\hplus \hminus} \sum_{\{\hiddenvar,\hiddenvar'\}\in U_{(+,+)}} \left( -  2 h_\hiddenvar h_{\hiddenvar'} \log \frac{1}{h_\hiddenvar} + h_\hiddenvar h_{\hiddenvar'} \right) \\
    =& \frac{\beta^2 \left(1 -2 \nr\right)^2}{\hplus \hminus} \left( \sum_{\targetvar_i \in \targetVarDom} (\alpha - \alpha_i) \sum_{\hiddenvar\in \setiplus} -2 h_\hiddenvar \log \frac{1}{h_\hiddenvar} + \sum_{\targetvar_i \in \targetVarDom} \alpha_i (\alpha - \alpha_i) \right) \\
    =& \frac{\left(1 -2 \nr\right)^2}{\hplus \hminus} \left(-2 \beta^2 \sum_{\targetvar_i \in \targetVarDom} (\alpha - \alpha_i) \sum_{\hiddenvar\in \setiplus} h_\hiddenvar \log \frac{1}{h_\hiddenvar} + \beta^2 \sum_{\targetvar_i \in \targetVarDom} \alpha_i (\alpha - \alpha_i) \right)
  \end{align*}

\item 
  Both $\hiddenvar$ and $\hiddenvar'$ maps $x$ to $-$. Similarly, we get
  \begin{align*}
    \sum_{\{\hiddenvar,\hiddenvar'\}\in U_{(-,-)}} - &\hplus \hminus (p_\hiddenvar - q_{\hiddenvar})(p_{\hiddenvar'} - q_{\hiddenvar'})  \left(\log \frac{1}{h_\hiddenvar h_{\hiddenvar'}} - 1\right) \\
    =& \frac{\left(1 -2 \nr\right)^2}{\hplus \hminus} \left(-2 \alpha^2 \sum_{\targetvar_i \in \targetVarDom} (\beta - \beta_i) \sum_{\hiddenvar\in \setiminus} h_\hiddenvar \log \frac{1}{h_\hiddenvar} + \alpha^2 \sum_{\targetvar_i \in \targetVarDom} \beta_i (\beta - \beta_i) \right)
  \end{align*}

\item 
  $\hiddenvar$ maps $x$ to $+$, $\hiddenvar'$ maps $x$ to $-$. We have
  \begin{align*}
    &\sum_{(\hiddenvar,\hiddenvar')\in U_{(+,-)}} - \hplus \hminus (p_\hiddenvar - q_{\hiddenvar})(p_{\hiddenvar'} - q_{\hiddenvar'})  \left(\log \frac{1}{h_\hiddenvar h_{\hiddenvar'}} - 1\right) \\
    =& \frac{\left(1 -2 \nr\right)^2}{\hplus \hminus} \left( \alpha \beta \sum_{\targetvar_i \in \targetVarDom} (\beta - \beta_i) \sum_{\hiddenvar\in \setiplus} h_\hiddenvar \log \frac{1}{h_\hiddenvar} + \alpha \beta \sum_{\targetvar_i \in \targetVarDom} (\alpha - \alpha_i) \sum_{\hiddenvar\in \setiminus} h_\hiddenvar \log \frac{1}{h_\hiddenvar} - \alpha \beta \sum_{\targetvar_i \in \targetVarDom} \alpha_i (\beta - \beta_i) \right)
  \end{align*}

\item 
  $\hiddenvar$ maps $x$ to $-$, $\hiddenvar'$ maps $x$ to $+$. By symmetry we have
  \begin{align*}
    &\sum_{(\hiddenvar,\hiddenvar')\in U_{(-,+)}} - \hplus \hminus (p_\hiddenvar - q_{\hiddenvar})(p_{\hiddenvar'} - q_{\hiddenvar'})  \left(\log \frac{1}{h_\hiddenvar h_{\hiddenvar'}} - 1\right) \\
    =&\sum_{(\hiddenvar,\hiddenvar')\in U_{(+,-)}} - \hplus \hminus (p_\hiddenvar - q_{\hiddenvar})(p_{\hiddenvar'} - q_{\hiddenvar'})  \left(\log \frac{1}{h_\hiddenvar h_{\hiddenvar'}} - 1\right)
  \end{align*}

\end{enumerate}
Combining the above four equations, we obtain a lower bound on Part 1:
\begin{align}
  \text{Part 1} &\geq \frac{\left(1 -2 \nr\right)^2}{\hplus \hminus}
                  \left(-2 \beta^2 \sum_{\targetvar_i \in \targetVarDom} (\alpha - \alpha_i) \sum_{\hiddenvar\in \setiplus} h_\hiddenvar \log \frac{1}{h_\hiddenvar} + \beta^2 \sum_{\targetvar_i \in \targetVarDom} \alpha_i (\alpha - \alpha_i) \right. \nonumber \\
                &\quad \left. -2 \alpha^2 \sum_{\targetvar_i \in \targetVarDom} (\beta - \beta_i) \sum_{\hiddenvar\in \setiminus} h_\hiddenvar \log \frac{1}{h_\hiddenvar} + \alpha^2 \sum_{\targetvar_i \in \targetVarDom} \beta_i (\beta - \beta_i) \right. \nonumber \\
                & \quad \left. + 2 \alpha \beta \sum_{\targetvar_i \in \targetVarDom} (\beta - \beta_i) \sum_{\hiddenvar\in \setiplus} h_\hiddenvar \log \frac{1}{h_\hiddenvar} + 2 \alpha \beta \sum_{\targetvar_i \in \targetVarDom} (\alpha - \alpha_i) \sum_{\hiddenvar\in \setiminus} h_\hiddenvar \log \frac{1}{h_\hiddenvar} - 2 \alpha \beta \sum_{\targetvar_i \in \targetVarDom} \alpha_i (\beta - \beta_i) \right) \nonumber \\
                &= \frac{\left(1 -2 \nr\right)^2}{\hplus \hminus}
                  \left( \left( 2\alpha \beta \sum_{\targetvar_i \in \targetVarDom} (\beta - \beta_i) - 2 \beta^2 \sum_{\targetvar_i \in \targetVarDom} (\alpha - \alpha_i)  \right) \sum_{\hiddenvar\in \setiplus} h_\hiddenvar \log \frac{1}{h_\hiddenvar} \right. \nonumber \\
                &\qquad \qquad \qquad \left. + \left( 2 \alpha \beta \sum_{\targetvar_i \in \targetVarDom} (\alpha - \alpha_i) -2 \alpha^2 \sum_{\targetvar_i \in \targetVarDom} (\beta - \beta_i)  \right) \sum_{\hiddenvar\in \setiminus} h_\hiddenvar \log \frac{1}{h_\hiddenvar} \right. \nonumber \\
                & \qquad \qquad \qquad  \left. + \beta^2 \sum_{\targetvar_i \in \targetVarDom} \alpha_i (\alpha - \alpha_i) + \alpha^2 \sum_{\targetvar_i \in \targetVarDom} \beta_i (\beta - \beta_i) - 2 \alpha \beta \sum_{\targetvar_i \in \targetVarDom} \alpha_i (\beta - \beta_i) \right) \nonumber \\
                &= \frac{\left(1 -2 \nr\right)^2}{\hplus \hminus} \cdot \nonumber \\
                & \left( 2 \sum_{\targetvar_i \in \targetVarDom} \beta (\beta \alpha_i - \alpha \beta_i) \sum_{\hiddenvar\in \setiplus} h_\hiddenvar \log \frac{1}{h_\hiddenvar} + 2 \sum_{\targetvar_i \in \targetVarDom} \alpha (\alpha \beta_i - \beta \alpha_i) \sum_{\hiddenvar\in \setiminus} h_\hiddenvar \log \frac{1}{h_\hiddenvar} - \sum_{\targetvar_i \in \targetVarDom}\left(\beta \alpha_i - \alpha \beta_i \right)^2 \right) \nonumber \\
                &= \frac{\left(1 -2 \nr\right)^2}{\hplus \hminus} \cdot \nonumber \\
                & \left( 2  \sum_{\targetvar_i \in \targetVarDom} (\beta \alpha_i - \alpha \beta_i) \left( \beta \alpha_i \sum_{\hiddenvar\in \setiplus} \frac{h_\hiddenvar}{\alpha_i} \log \frac{1}{h_\hiddenvar} - \alpha \beta_i \sum_{\hiddenvar\in \setiminus} \frac{h_\hiddenvar}{\beta_i} \log \frac{1}{h_\hiddenvar} \right) - \sum_{\targetvar_i \in \targetVarDom}\left(\beta \alpha_i - \alpha \beta_i \right)^2 \right)
                  \label{eq:lowerbound_term1_part1_complete}
\end{align}

\paragraph{Part 2.}
Next, we will provide a lower bound on Part 2 of Equation~\eqrefcus{eq:term1_expanded}.

By definition, we have
\begin{align}
  \text{Part 2}
  &= \sum_{\hiddenvar \nsim \hiddenvar'} \left(\hplus p_\hiddenvar p_{\hiddenvar'} + \hminus q_\hiddenvar q_{\hiddenvar'} \right) \log \frac{1}{\hplus p_\hiddenvar p_{\hiddenvar'} + \hminus q_\hiddenvar q_{\hiddenvar'}} \nonumber \\
  & \qquad \qquad \qquad - \left( \hplus \sum_{\hiddenvar \nsim \hiddenvar'} p_\hiddenvar p_{\hiddenvar'} \log \frac{1}{p_\hiddenvar p_{\hiddenvar'}} + \hminus \sum_{\hiddenvar \nsim \hiddenvar'} q_\hiddenvar q_{\hiddenvar'} \log \frac{1}{q_\hiddenvar q_{\hiddenvar'}}\right) \nonumber \\
  &\stackrel{\text{(a)}}{\geq} \frac{\hplus \hminus}{2}\sum_{\hiddenvar \nsim \hiddenvar'}\frac{\left(p_\hiddenvar p_{\hiddenvar'} - q_\hiddenvar q_{\hiddenvar'}\right)^2}{p_\hiddenvar p_{\hiddenvar'} + q_\hiddenvar q_{\hiddenvar'}} \nonumber
\end{align}
Hereby, step (a) is due to the strong concavity\footnote{If $f$ is strongly concave, then for $t \in [0,1]$, it holds that $f(tx + (1-t)y) - tf(x) - (1-t)f(y) \geq \frac{t(1-t)}{2} m(x-y)^2$, where $m = \min \left(\abs{f''(x)}, \abs{f''(y)} \right).$\label{ft:strongconc}} of $f(x)=x\log\frac{1}{x}$.

Similarly with the analysis of Part 1, we consider the four sets of $\{\hiddenvar,\hiddenvar'\}$ pairs:
\begin{enumerate}
\item $\{\hiddenvar,\hiddenvar'\} \in U_{(+,+)}$: both $\hiddenvar$ and $\hiddenvar'$ maps $x$ to $+$.

  In this case, we have
  \begin{align*}
    \sum_{(\hiddenvar,\hiddenvar')\in U_{(+,+)}} \frac{\hplus \hminus}{2} \frac{\left(p_\hiddenvar p_{\hiddenvar'} - q_\hiddenvar q_{\hiddenvar'}\right)^2}{p_\hiddenvar p_{\hiddenvar'} + q_\hiddenvar q_{\hiddenvar'}}
    &\geq \sum_{(\hiddenvar,\hiddenvar')\in U_{(+,+)}} \frac{\hplus \hminus}{2} \left(\sqrt{p_\hiddenvar p_{\hiddenvar'}} - \sqrt{q_\hiddenvar q_{\hiddenvar'}}\right)^2 \\
    &\stackrel{\text{Eq}~\eqrefcus{eq:def_hpq}}{=} \sum_{(\hiddenvar,\hiddenvar')\in U_{(+,+)}} \frac{\hplus \hminus}{2} \left(\sqrt{\frac{h_\hiddenvar \nrbar}{\hplus} \frac{h_{\hiddenvar'} \nrbar}{\hplus}} - \sqrt{\frac{h_\hiddenvar \nr}{\hminus} \frac{h_{\hiddenvar'} \nr}{\hminus}}\right)^2 \\
    &= \sum_{(\hiddenvar,\hiddenvar')\in U_{(+,+)}} \frac{\hplus \hminus}{2} h_\hiddenvar h_{\hiddenvar'} \left(\frac{\nrbar}{\hplus} - \frac{\nr}{\hminus}\right)^2 \\
    &= \sum_{(\hiddenvar,\hiddenvar')\in U_{(+,+)}} \frac{\hplus \hminus}{2} h_\hiddenvar h_{\hiddenvar'} \frac{\beta^2 \left(1 -2 \nr\right)^2}{(\hplus \hminus)^2} \\
    &= \frac{\left(1 -2 \nr\right)^2}{2 \hplus \hminus} \beta^2 \sum_{\targetvar_i \in \targetVarDom} \alpha_i (\alpha-\alpha_i)
  \end{align*}

\item $(\hiddenvar,\hiddenvar') \in U_{(-,-)}$.
  Similarly, we get
  \begin{align*}
    \sum_{(\hiddenvar,\hiddenvar')\in U_{(-,-)}} \frac{\hplus \hminus}{2} \frac{\left(p_\hiddenvar p_{\hiddenvar'} - q_\hiddenvar q_{\hiddenvar'}\right)^2}{p_\hiddenvar p_{\hiddenvar'} + q_\hiddenvar q_{\hiddenvar'}}
    &\geq \frac{\left(1 -2 \nr\right)^2}{2 \hplus \hminus} \alpha^2 \sum_{\targetvar_i \in \targetVarDom} \beta_i (\beta-\beta_i)
  \end{align*}

\item $(\hiddenvar,\hiddenvar') \in U_{(+,-)}$: $\hiddenvar$ maps $x$ to $+$, $\hiddenvar'$ maps $x$ to $-$. We have
  \begin{align*}
    \sum_{(\hiddenvar,\hiddenvar')\in U_{(+,-)}} \frac{\hplus \hminus}{2} \frac{\left(p_\hiddenvar p_{\hiddenvar'} - q_\hiddenvar q_{\hiddenvar'}\right)^2}{p_\hiddenvar p_{\hiddenvar'} + q_\hiddenvar q_{\hiddenvar'}}
    &\geq \sum_{(\hiddenvar,\hiddenvar')\in U_{(+,+)}} \frac{\hplus \hminus}{2} \left(\sqrt{\frac{h_\hiddenvar \nrbar}{\hplus} \frac{h_{\hiddenvar'} \nr}{\hplus}} - \sqrt{\frac{h_\hiddenvar \nr}{\hminus} \frac{h_{\hiddenvar'} \nrbar}{\hminus}}\right)^2 \\
    &= \sum_{(\hiddenvar,\hiddenvar')\in U_{(+,+)}} \frac{\hplus \hminus}{2} h_\hiddenvar h_{\hiddenvar'} \nr \nrbar \left(\frac{1}{\hplus} - \frac{1}{\hminus}\right)^2 \\
    &= \frac{(1-2\nr)^2}{2 \hplus \hminus} \nr\nrbar (\alpha-\beta)^2 \sum_{\targetvar_i \in \targetVarDom} \alpha_i (\beta - \beta_i)
  \end{align*}

\item $(\hiddenvar,\hiddenvar') \in U_{(-,+)}$: $\hiddenvar$ maps $x$ to $-$, $\hiddenvar'$ maps $x$ to $+$. By symmetry we have
  \begin{align*}
    \sum_{(\hiddenvar,\hiddenvar')\in U_{(+,-)}} \frac{\hplus \hminus}{2} \frac{\left(p_\hiddenvar p_{\hiddenvar'} - q_\hiddenvar q_{\hiddenvar'}\right)^2}{p_\hiddenvar p_{\hiddenvar'} + q_\hiddenvar q_{\hiddenvar'}}
    &\geq \frac{(1-2\nr)^2}{2 \hplus \hminus} \nr\nrbar (\alpha-\beta)^2 \sum_{\targetvar_i \in \targetVarDom} \beta_i (\alpha - \alpha_i)
  \end{align*}
\end{enumerate}

Combining the above four equations, we obtain a lower bound on Part 2:

\begin{align}
  \text{Part 2}
  &\geq \sum_{(\hiddenvar,\hiddenvar')\in U_{(+,+)}} \frac{\hplus \hminus}{2} \frac{\left(p_\hiddenvar p_{\hiddenvar'} - q_\hiddenvar q_{\hiddenvar'}\right)^2}{p_\hiddenvar p_{\hiddenvar'} + q_\hiddenvar q_{\hiddenvar'}}
    + \sum_{(\hiddenvar,\hiddenvar')\in U_{(-,-)}} \frac{\hplus \hminus}{2} \frac{\left(p_\hiddenvar p_{\hiddenvar'} - q_\hiddenvar q_{\hiddenvar'}\right)^2}{p_\hiddenvar p_{\hiddenvar'} + q_\hiddenvar q_{\hiddenvar'}} \nonumber \\
  &\qquad \qquad + \sum_{(\hiddenvar,\hiddenvar')\in U_{(+,-)}} \frac{\hplus \hminus}{2}\frac{\left(p_\hiddenvar p_{\hiddenvar'} - q_\hiddenvar q_{\hiddenvar'}\right)^2}{p_\hiddenvar p_{\hiddenvar'} + q_\hiddenvar q_{\hiddenvar'}}
    + \sum_{(\hiddenvar,\hiddenvar')\in U_{(-,+)}} \frac{\hplus \hminus}{2}\frac{\left(p_\hiddenvar p_{\hiddenvar'} - q_\hiddenvar q_{\hiddenvar'}\right)^2}{p_\hiddenvar p_{\hiddenvar'} + q_\hiddenvar q_{\hiddenvar'}} \nonumber \\
  &= \frac{(1-2\nr)^2}{2 \hplus \hminus} \left(\beta^2\sum_{\targetvar_i \in \targetVarDom} \alpha_i (\alpha - \alpha_i) + \alpha^2 \sum_{\targetvar_i \in \targetVarDom} \beta_i (\beta - \beta_i) + 2\nr\nrbar (\alpha-\beta)^2 \sum_{\targetvar_i \in \targetVarDom} \alpha_i (\beta - \beta_i) \right) \label{eq:lowerbound_term1_part2_complete}
\end{align}


\subsubsection{A Lower Bound on Term 2}
Now we move on to analyze Term \circled{2} of Equation~\eqrefcus{eq:term1_expanded}.
By strong concavity of $f(x)=x\log\frac{1}{x} + (1-x)\log\frac{1}{1-x}$, we obtain
\begin{align}
  \text{Term \circled{2}}
  &=c \sum_{\targetvar_i \in \targetVarDom} \left( h_i \log \frac{1}{h_i} + (1-h_i) \log \frac{1}{1-h_i} - \hplus \left( p_i \log \frac{1}{p_i} + (1-p_i) \log \frac{1}{1-p_i} \right)\right. \nonumber \\
  & \hspace{6.1cm} - \left.\hminus  \left( q_i \log \frac{1}{q_i} + (1-q_i) \log \frac{1}{1-q_i} \right) \right) \nonumber \\
  & \stackrel{\text{footnote}~\ref{ft:strongconc}}{\geq}
    \frac{c \cdot \hplus \hminus}{2}\sum_{\targetvar_i \in \targetVarDom} \frac{\left(p_i - q_i \right)^2}{\max\{p_i(1-p_i),  q_i(1-q_i)\}} \nonumber
\end{align}
Plugging in the definition of $p_i$, $q_i$ from {Equation}~\eqrefcus{eq:def_hpq}, we get
\begin{align}
  \text{Term \circled{2}}
  &= \frac{c \cdot \hplus \hminus}{2}\sum_{\targetvar_i \in \targetVarDom} \left(\frac{\alpha_i \nrbar + \beta_i \nr}{\hplus} - \frac{\alpha_i \nr + \beta_i \nrbar}{\hminus}\right)^2 \frac{1}{\max\{p_i(1-p_i),  q_i(1-q_i)\}} \nonumber \\
  &= \frac{c}{2\hplus \hminus} \sum_{\targetvar_i \in \targetVarDom}   \frac{ \left((\alpha \nr + \beta \nrbar)(\alpha_i \nrbar + \beta_i \nr) - (\alpha \nrbar + \beta \nr)(\alpha_i \nr + \beta_i \nrbar)\right)^2 }{\max\{p_i(1-p_i),  q_i(1-q_i)\}}\nonumber  \\
  &= \frac{c}{2\hplus \hminus} \sum_{\targetvar_i \in \targetVarDom}   \frac{ \left(\alpha \beta_i \nr^2 + \beta \alpha_i \nrbar^2 - \alpha \beta_i \nrbar^2 - \beta \alpha_i \nr^2\right)^2 }{\max\{p_i(1-p_i),  q_i(1-q_i)\}} \nonumber \\
  &= \frac{c(1-2\nr)^2}{2\hplus \hminus} \sum_{\targetvar_i \in \targetVarDom}  \frac{ \left( \beta \alpha_i - \alpha \beta_i \right)^2 }{\max\{p_i(1-p_i),  q_i(1-q_i)\}}\label{eq:lowerbound_term3_complete}
\end{align}

\subsubsection{A Combined Lower Bound for $\auxgain$}
Now, combining Equation~\eqrefcus{eq:lowerbound_term1_part1_complete}, \eqrefcus{eq:lowerbound_term1_part2_complete}, 
and \eqrefcus{eq:lowerbound_term3_complete}, we can get a lower bound for $\auxgain$:
\begin{align}
  \auxgain
  &\geq
    \frac{\left(1 -2 \nr\right)^2}{\hplus \hminus} \cdot \left( 2  \sum_{\targetvar_i \in \targetVarDom} (\beta \alpha_i - \alpha \beta_i) \left( \beta \alpha_i \sum_{\hiddenvar\in \setiplus} \frac{h_\hiddenvar}{\alpha_i} \log \frac{1}{h_\hiddenvar} - \alpha \beta_i \sum_{\hiddenvar\in \setiminus} \frac{h_\hiddenvar}{\beta_i} \log \frac{1}{h_\hiddenvar} \right) \right.\nonumber \\
  &\hspace{9cm}  - \left. \sum_{\targetvar_i \in \targetVarDom}\left(\beta \alpha_i - \alpha \beta_i \right)^2 \right) \nonumber \\
  &\qquad + \frac{(1-2\nr)^2}{2 \hplus \hminus} \left(\beta^2\sum_{\targetvar_i \in \targetVarDom} \alpha_i (\alpha - \alpha_i) + \alpha^2 \sum_{\targetvar_i \in \targetVarDom} \beta_i (\beta - \beta_i) + 2\nr\nrbar (\alpha-\beta)^2 \sum_{\targetvar_i \in \targetVarDom} \alpha_i (\beta - \beta_i) \right) \nonumber \\
  &\qquad + \frac{c(1-2\nr)^2}{2\hplus \hminus} \sum_{\targetvar_i \in \targetVarDom}  \frac{ \left( \beta \alpha_i - \alpha \beta_i \right)^2 }{\max\{p_i(1-p_i),  q_i(1-q_i)\}} \label{eq:auxgain_lowebound_complete}
\end{align}

We can rewrite Equation~\eqrefcus{eq:auxgain_lowebound_complete} as
\begin{align}
  &\auxgain \geq \nonumber \\
  & \underbrace{\frac{(1-2\nr)^2 }{4 \hplus \hminus} \left(\sum_{\targetvar_i \in \targetVarDom} \left(\beta \alpha_i - \alpha \beta_i \right)^2 + \beta^2\sum_{\targetvar_i \in \targetVarDom} \alpha_i (\alpha - \alpha_i) + \alpha^2 \sum_{\targetvar_i \in \targetVarDom} \beta_i (\beta - \beta_i) + 2\nr\nrbar (\alpha-\beta)^2 \sum_{\targetvar_i \in \targetVarDom} \alpha_i (\beta - \beta_i) \right)}_{\text{LB1}} \nonumber \\
  & + \frac{\left(1 -2 \nr\right)^2}{4 \hplus \hminus} \left( \beta^2\sum_{\targetvar_i \in \targetVarDom} \alpha_i (\alpha - \alpha_i) + \alpha^2 \sum_{\targetvar_i \in \targetVarDom} \beta_i (\beta - \beta_i) + 2\nr\nrbar (\alpha-\beta)^2 \sum_{\targetvar_i \in \targetVarDom} \alpha_i (\beta - \beta_i) \right. \nonumber \\
  & \qquad \qquad \qquad + \left. 2 c \sum_{\targetvar_i \in \targetVarDom}  \frac{ \left( \beta \alpha_i - \alpha \beta_i \right)^2 }{\max\{p_i(1-p_i),  q_i(1-q_i)\}}  - 5 \sum_{\targetvar_i \in \targetVarDom} \left( \beta \alpha_i - \alpha \beta_i \right)^2\right. \nonumber \\
  & \quad \underbrace{ \qquad \qquad +
    \left. 8 \sum_{\targetvar_i \in \targetVarDom} (\beta \alpha_i - \alpha \beta_i) \left( \beta \alpha_i \sum_{\hiddenvar\in \setiplus} \frac{h_\hiddenvar}{\alpha_i} \log \frac{1}{h_\hiddenvar} - \alpha \beta_i \sum_{\hiddenvar\in \setiminus} \frac{h_\hiddenvar}{\beta_i} \log \frac{1}{h_\hiddenvar} \right) \right) ~~\quad }_{\text{LB2}} \label{eq:auxgain_lowerbound_expanded_lb1lb2}
\end{align}

\subsubsection{Connecting $\auxgain$ with $\ectgain$}\label{sec:proof_auxgain_vs_ec2gain}
Next, we will show that term LB1 is lower-bounded by a factor of $\ectgain$ (i.e., \ectgainggraph), while LB2 cannot be too much less than 0. Concretely, we will show
\begin{itemize}
\item $\text{LB1} \geq \frac{1}{16} \left( 1 - 2\nr \right)^2 \ectgain$, and
\item $\text{LB2} \geq - 2 \numtar \left(1 -2 \nr\right)^2 \eta$, for $\eta \in (0,1)$.
\end{itemize}
At the end of this subsection, we will combine the above results to connect \auxgainggraph with \ectgainggraph (See Equation~\eqrefcus{eq:auxgain_vs_ec2gain_proof_final}).
\paragraph{LB1 VS. $\ectgain$. }

We expand the \ECT gain $\ectgainggraph$ as
\begin{align}
  \ectgain
  &= \sum_{\targetvar_i \in \targetVarDom} (\alpha_i + \beta_i) (1 - \alpha_i - \beta_i) - \alpha \sum_{\targetvar_i \in \targetVarDom} \alpha_i (\alpha - \alpha_i) - \beta \sum_{\targetvar_i \in \targetVarDom} \beta_i (\beta - \beta_i) \nonumber\\
  &= \beta \sum_{\targetvar_i \in \targetVarDom} \alpha_i (\alpha - \alpha_i) + \alpha \sum_{\targetvar_i \in \targetVarDom} \beta_i (\beta - \beta_i) + 2 \sum_{\targetvar_i \in \targetVarDom} \alpha_i (\beta - \beta_i) \label{eq:gain_ec2_noiseless_complete}
\end{align}
Define
\begin{align*}
  \begin{cases}
    \circled{*} &\triangleq \frac{16 \hplus \hminus}{(1-2\nr)^2} \cdot \text{LB1} \\
    &= 4 \left( \sum_{\targetvar_i \in \targetVarDom} \left(\beta \alpha_i - \alpha \beta_i \right)^2 + \beta^2\sum_{\targetvar_i \in \targetVarDom} \alpha_i (\alpha - \alpha_i) + \alpha^2 \sum_{\targetvar_i \in \targetVarDom} \beta_i (\beta - \beta_i) \right. \\
    &\hspace{7cm}\left. +2\nr\nrbar (\alpha-\beta)^2 \sum_{\targetvar_i \in \targetVarDom} \alpha_i (\beta - \beta_i) \right) \\\\
    \circled{\#} &\triangleq \hplus \hminus \ectgain \\
    &= \left( \nr \nrbar (\alpha - \beta)^2 + \alpha \beta \right) \left(\beta \sum_{\targetvar_i \in \targetVarDom} \alpha_i (\alpha - \alpha_i) + \alpha \sum_{\targetvar_i \in \targetVarDom} \beta_i (\beta - \beta_i) + 2 \sum_{\targetvar_i \in \targetVarDom} \alpha_i (\beta - \beta_i)\right)
  \end{cases}
\end{align*}
To bound LB1 against $\frac{1}{16} \left( 1 - 2\nr \right)^2 \ectgain$, it suffices to show $\circled{*} \geq \circled{\#}$. 

To prove the above inequality, we consider the following two cases:
\begin{enumerate}[itemsep=0mm,leftmargin=5mm]
\item $\nr \nrbar (\alpha - \beta)^2 \leq \alpha \beta$.
  In this case, we have $\nr \nrbar (\alpha - \beta)^2 + \alpha \beta \leq 2 \alpha \beta$. Then,
  \begin{align*}
    \frac{\circled{*} - \circled{\#}}{2}
    &\geq \frac{\circled{*}}{2} - \alpha\beta \left(\beta \sum_{\targetvar_i \in \targetVarDom} \alpha_i (\alpha - \alpha_i) + \alpha \sum_{\targetvar_i \in \targetVarDom} \beta_i (\beta - \beta_i) + 2 \sum_{\targetvar_i \in \targetVarDom} \alpha_i (\beta - \beta_i)\right) \\
    &\geq \beta^2(1+\beta) \sum_{\targetvar_i \in \targetVarDom} \alpha_i (\alpha - \alpha_i) + \alpha^2 (1 + \alpha) \sum_{\targetvar_i \in \targetVarDom} \beta_i (\beta - \beta_i) - 2\alpha\beta \sum_{\targetvar_i \in \targetVarDom} \alpha_i (\beta - \beta_i) \\
    &\hspace{9cm}+ \sum_{\targetvar_i \in \targetVarDom} (\beta \alpha_i - \alpha \beta_i)^2 \\
    &\geq \beta^2 \sum_{\targetvar_i \in \targetVarDom} \alpha_i (\alpha - \alpha_i) + \alpha^2 \sum_{\targetvar_i \in \targetVarDom} \beta_i (\beta - \beta_i) - 2\alpha\beta \sum_{\targetvar_i \in \targetVarDom} \alpha_i (\beta - \beta_i) + \sum_{\targetvar_i \in \targetVarDom} (\beta \alpha_i - \alpha \beta_i)^2 \\
    & = 0
  \end{align*}

\item $\nr \nrbar (\alpha - \beta)^2 > \alpha \beta$. W.l.o.g., we assume $\beta \leq \alpha \leq 1$. By $\alpha+\beta=1$ we get $2\alpha \geq 1$.

  Observe the fact that
  \begin{align*}
    \sum_{\targetvar_i \in \targetVarDom} (\beta \alpha_i - \alpha \beta_i)^2 =  -  \beta^2 \sum_{\targetvar_i \in \targetVarDom} \alpha_i (\alpha - \alpha_i) - \alpha^2 \sum_{\targetvar_i \in \targetVarDom} \beta_i (\beta - \beta_i) + 2 \alpha \beta \sum_{\targetvar_i \in \targetVarDom} \alpha_i (\beta - \beta_i) \geq 0
  \end{align*}
  Rearranging the terms in the above inequality, we get
  \begin{align}
    \beta \sum_{\targetvar_i \in \targetVarDom} \alpha_i (\alpha - \alpha_i) \leq 2\alpha \sum_{\targetvar_i \in \targetVarDom} \alpha_i (\beta - \beta_i) \leq 2 (\alpha\beta - \sum_{\targetvar_i \in \targetVarDom} \alpha_i \beta_i) = 2 \sum_{\targetvar_i \in \targetVarDom} \alpha_i (\beta -\beta_i) \label{eq:tmp_buffer_ineq2}
  \end{align}
  Hence,
  \begin{align*}
    \circled{\#}
    &\leq 2 \nr \nrbar (\alpha - \beta)^2 \left(\beta \sum_{\targetvar_i \in \targetVarDom} \alpha_i (\alpha - \alpha_i) + \alpha \sum_{\targetvar_i \in \targetVarDom} \beta_i (\beta - \beta_i) + 2 \sum_{\targetvar_i \in \targetVarDom} \alpha_i (\beta - \beta_i)\right) \\
    &\stackrel{\text{\eqrefcus{eq:tmp_buffer_ineq2}}}{\leq} 2 \nr \nrbar (\alpha - \beta)^2 \left(\alpha \sum_{\targetvar_i \in \targetVarDom} \beta_i (\beta - \beta_i) + 4 \sum_{\targetvar_i \in \targetVarDom} \alpha_i (\beta - \beta_i)\right) \\
    & \stackrel{2\alpha \geq 1}{\leq} 2 \nr \nrbar (\alpha - \beta)^2 \left(2 \alpha^2 \sum_{\targetvar_i \in \targetVarDom} \beta_i (\beta - \beta_i) + 4 \sum_{\targetvar_i \in \targetVarDom} \alpha_i (\beta - \beta_i)\right) \\
    & \stackrel{\nr \nrbar (\alpha - \beta)^2\leq 1}{\leq} 4 \left( 2 \nr \nrbar (\alpha - \beta)^2 \sum_{\targetvar_i \in \targetVarDom} \alpha_i (\beta - \beta_i) + \alpha^2 \sum_{\targetvar_i \in \targetVarDom} \beta_i (\beta - \beta_i)\right) \\
    & \leq \circled{*}
  \end{align*}
  Therefore, we get
  \begin{align}
    \text{LB1} \geq \frac{1}{16} \left( 1 - 2\nr \right)^2 \ectgain\label{eq:auxgain_bound_lb1}
  \end{align}
\end{enumerate}

\paragraph{A lower bound on \text{LB2}.}

In the following, we will analyze LB2.
\begin{align*}
  \text{LB2}
  &\geq \frac{\left(1 -2 \nr\right)^2}{4 \hplus \hminus} \left( \beta^2\sum_{\targetvar_i \in \targetVarDom} \alpha_i (\alpha - \alpha_i) + \alpha^2 \sum_{\targetvar_i \in \targetVarDom} \beta_i (\beta - \beta_i) - 5 \sum_{\targetvar_i \in \targetVarDom} \left( \beta \alpha_i - \alpha \beta_i \right)^2 \right. \\
  & \qquad  + 2 c_2 \sum_{\targetvar_i \in \targetVarDom}  \frac{ \left( \beta \alpha_i - \alpha \beta_i \right)^2 }{\max\{p_i(1-p_i),  q_i(1-q_i)\}} \\
  & \qquad  +
    8 \sum_{\targetvar_i \in \targetVarDom} (\beta \alpha_i - \alpha \beta_i)  \left( \beta \alpha_i \sum_{\hiddenvar\in \setiplus} \frac{h_\hiddenvar}{\alpha_i} \log \frac{\alpha_i}{h_\hiddenvar}  + \beta \alpha_i \log \frac{1}{\alpha_i} - \alpha \beta_i \sum_{\hiddenvar\in \setiminus} \frac{h_\hiddenvar}{\beta_i} \log \frac{\beta_i}{h_\hiddenvar} - \alpha \beta_i \log \frac{1}{\beta_i}\right)
\end{align*}
For brevity, define $\mu_i \triangleq \alpha_i / \alpha$, and $\nu_i \triangleq \beta_i / \beta$. We can simplify
the above equation as
\begin{align}
  \text{LB2}
  &\geq \frac{\alpha^2\beta^2 \left(1 -2 \nr\right)^2}{4 \hplus \hminus} \sum_{\targetvar_i \in \targetVarDom} \left(  \mu_i (1 - \mu_i) + \nu_i (1  - \nu_i) - 5 (\mu_i - \nu_i)^2 +  \frac{ 2 c_2 \left( \mu_i - \nu_i \right)^2 }{\max\{p_i(1-p_i),  q_i(1-q_i)\}} \right. \nonumber \\
  &+ \left.
    8 (\mu_i - \nu_i) \left( \mu_i \sum_{\hiddenvar\in \setiplus} \frac{h_\hiddenvar}{\alpha_i} \log \frac{\alpha_i}{h_\hiddenvar} + \mu_i \log \frac{1}{\mu_i \alpha} - \nu_i \sum_{\hiddenvar\in \setiminus} \frac{h_\hiddenvar}{\beta_i} \log \frac{\beta_i}{h_\hiddenvar} - \nu_i \log \frac{1}{\nu_i \beta}\right)  \right) \label{eq:auxgain_lb2_simplified}
\end{align}

Denote the summand on the RHS of the above equation as $\text{LB2}_i$. If for any $\targetvar_i \in \targetVarDom$ we can lower bound $\text{LB2}_i$, we can then bound the whole sum. Fix $i$. W.l.o.g., we assume $\mu_i \geq \nu_i$. Then
\begin{align*}
  \text{LB2}_i &\triangleq \mu_i (1 - \mu_i) + \nu_i (1  - \nu_i) - 5 (\mu_i - \nu_i)^2 +  \frac{ 2 c \left( \mu_i - \nu_i \right)^2 }{\max\{p_i(1-p_i),  q_i(1-q_i)\}}  \\
               & \qquad  +
                 8 (\mu_i - \nu_i) \left( \cancelto{\geq 0}{\mu_i \sum_{\hiddenvar\in \setiplus} \frac{h_\hiddenvar}{\alpha_i} \log \frac{\alpha_i}{h_\hiddenvar} + \mu_i \log \frac{1}{\mu_i \alpha}} - \nu_i \sum_{\hiddenvar\in \setiminus} \frac{h_\hiddenvar}{\beta_i} \log \frac{\beta_i}{h_\hiddenvar} - \nu_i \log \frac{1}{\nu_i \beta}\right) \\
               &\geq \mu_i (1 - \mu_i) + \nu_i (1  - \nu_i) - 5 (\mu_i - \nu_i)^2 +  \frac{ 2 c \left( \mu_i - \nu_i \right)^2 }{\max\{p_i(1-p_i),  q_i(1-q_i)\}}  \\
               & \qquad  - 8 (\mu_i - \nu_i) \left(\nu_i \cancelto{\leq \log n}{\sum_{\hiddenvar\in \setiminus} \frac{h_\hiddenvar}{\beta_i} \log \frac{\beta_i}{h_\hiddenvar}} + \nu_i \log \frac{1}{\nu_i} + \nu_i \log \frac{1}{\beta}\right) \\
               &\geq \mu_i (1 - \mu_i) + \nu_i (1  - \nu_i)  - 5 (\mu_i - \nu_i)^2  - 8 (\mu_i - \nu_i) \left(\nu_i \log \frac{n}{\beta} +\nu_i \log \frac{1}{\nu_i}\right) \\
               & \qquad  +  \frac{ 2 c \left( \mu_i - \nu_i \right)^2 }{\max\{p_i(1-p_i),  q_i(1-q_i)\}} \\
\end{align*}

In order to put a lower bound on the above terms, we first need to lower bound the term involving $\frac{ \left( \mu_i - \nu_i \right)^2 }{\max\{p_i(1-p_i),  q_i(1-q_i)\}}$. Notice that $p_i = \frac{\alpha_i + \beta_i \nr/\nrbar}{\alpha + \beta \nr/\nrbar}$, and $p_i = \frac{\alpha_i \nr/\nrbar + \beta_i}{\alpha \nr/\nrbar+ \beta}$. Therefore, $\min \left\{\mu_i,\nu_i\right\} \leq p_i, q_i \leq \max \left\{\mu_i,\nu_i \right\}$.

We check three different cases:
\begin{itemize}[itemsep=0mm,leftmargin=.5cm]
\item $\mu_i \geq \nu_i \geq 1/2$, or $\nu_i \leq \mu_i \leq 1/2$.

  In this case, $ \max\{p_i(1-p_i),  q_i(1-q_i)\} \leq \max\{\mu_i(1-\mu_i), \nu_i(1-\nu_i)\}$. Therefore,
  \begin{align*}
    \text{LB2}_i &\geq  - 5 (\mu_i - \nu_i)^2  - 8 (\mu_i - \nu_i) \left(\nu_i \log \frac{n}{\beta} + \nu_i \log \frac{1}{\nu_i}\right) \\
                 & \qquad  +  \frac{ 2 c \left( \mu_i - \nu_i \right)^2 }{\max\{\mu_i(1-\mu_i), \nu_i(1-\nu_i)\}} + \mu_i (1 - \mu_i) + \nu_i (1  - \nu_i) \\
                 &\geq - 5 (\mu_i - \nu_i)^2  - 8 (\mu_i - \nu_i) \left(\nu_i \log \frac{n}{\beta} + \nu_i \log \frac{1}{\nu_i}\right) \\
                 & \qquad  +  \frac{ 2 c \left( \mu_i - \nu_i \right)^2 }{\max\{\mu_i(1-\mu_i), \nu_i(1-\nu_i)\}} + \max\{\mu_i(1-\mu_i), \nu_i(1-\nu_i)\} \\
                 &\geq - 5 (\mu_i - \nu_i)^2  - 8 (\mu_i - \nu_i) \left(\nu_i \log \frac{n}{\beta} + \nu_i \log \frac{1}{\nu_i}\right) + 2\sqrt{2c} (\mu_i - \nu_i) \\
                 &\stackrel{\mu_i - \nu_i \leq 1/2}{\geq} (\mu_i - \nu_i) \left( 2\sqrt{2c} - 5/2  - 8 \left(\nu_i \log \frac{n}{\beta} + \nu_i \log \frac{1}{\nu_i}\right)  \right) \\
                 &\stackrel{\text{(a)}}{\geq} (\mu_i - \nu_i) \left( 2\sqrt{2c} - 5/2 - 8 \log \frac{n}{\beta}  \right)
  \end{align*}
  Here, step (a) is due to the fact that $f(x) = x\log \frac{n}{\beta x}$ is monotone increasing for $n \geq 3$. When $n < 3$, we have $\mu_i = 1$ and $\nu_i = 0$ (otherwise, there is no uncertainty left in $\Targetvar$) and hence the problem becomes trivial.
\item $1/n \leq \nu_i \leq 1/2 \leq \mu_i$.

  In this case, we cannot replace $p_i, q_i$ with $\mu_i$ or $\nu_i$. However, notice that $\max\{\mu_i(1-\mu_i), \nu_i(1-\nu_i)\} \leq 1/4$, we have
  \begin{align}
    \text{LB2}_i &\geq \mu_i (1 - \mu_i) + \nu_i (1  - \nu_i) - 5 (\mu_i - \nu_i)^2  - 8 (\mu_i - \nu_i) \left(\nu_i \log \frac{n}{\beta} + \nu_i \log \frac{1}{\nu_i}\right) + 8 c \left( \mu_i - \nu_i \right)^2 \nonumber\\
                 &= \mu_i (1 - \mu_i) + \nu_i (1  - \nu_i) + (\mu_i - \nu_i)^2 + (8 c - 6) (\mu_i - \nu_i)^2 \nonumber \\ &\qquad - 8 (\mu_i - \nu_i) \left(\nu_i \log \frac{n}{\beta} + \nu_i \log \frac{1}{\nu_i}\right) \nonumber\\
                 &= \mu_i (1 - \nu_i) + \nu_i (1  - \mu_i) + (8 c - 6) (\mu_i - \nu_i)^2 - 8 (\mu_i - \nu_i) \left(\nu_i \log \frac{n}{\beta} + \nu_i \log \frac{1}{\nu_i}\right) \nonumber\\
                 &\geq \mu_i (1 - \nu_i) + (8 c - 6) (\mu_i - \nu_i)^2 - 8 (\mu_i - \nu_i) \left(\nu_i \log \frac{n}{\beta} + \nu_i \log \frac{1}{\nu_i}\right) \label{eq:lb2_case2}\\
                 &\stackrel{\nu_i \geq 1/n}{\geq} \mu_i (1 - \nu_i) + (8 c - 6) (\mu_i - \nu_i)^2 - 8 (\mu_i - \nu_i) \nu_i \log \frac{n^2}{\beta} \nonumber
  \end{align}
  To further simplify notation, we denote $\gamma_1 \triangleq 8 c - 6$, and $\gamma_2 \triangleq 8 \log \frac{n^2}{\beta}$. Then the above equation can be rewritten as
  \begin{align*}
    \text{LB2}_i &\geq \mu_i (1 - \nu_i) + \gamma_1 (\mu_i - \nu_i)^2 - \gamma_2 (\mu_i - \nu_i) \nu_i
  \end{align*}
  If $\mu_i - \nu_i \leq \frac{1}{2\gamma_2}$, then
  \begin{align*}
    \text{LB2}_i &\geq \mu_i (1 - \nu_i) + \gamma_1 (\mu_i - \nu_i)^2 - \frac{1}{2 \gamma_2} \gamma_2 \nu_i = \mu_i (1 - \nu_i) - \frac{\nu_i}{2} \geq 0
  \end{align*}
  Otherwise, if $\mu_i - \nu_i > \frac{1}{2\gamma_2}$, we have
  \begin{align*}
    \text{LB2}_i &\geq \mu_i (1 - \nu_i) + (\mu_i - \nu_i) \left( \gamma_1 (\mu_i - \nu_i) - \gamma_2 \nu_i \right)\\
                 &> \mu_i (1 - \nu_i) + (\mu_i - \nu_i) \left( \gamma_1 \frac{1}{2\gamma_2} - \gamma_2 \nu_i \right) \\
                 &> \frac{\mu_i - \nu_i}{2} \left( \frac{\gamma_1}{\gamma_2} - \gamma_2 \right)
  \end{align*}
\item $ \nu_i \leq 1/n < 1/2 \leq \mu_i$. In this case, we have
  \begin{align*}
    \text{LB2}_i &\stackrel{\text{Eq}~\eqrefcus{eq:lb2_case2}}{\geq} \mu_i (1 - \nu_i) + \gamma_1 (\mu_i - \nu_i)^2 - 8 (\mu_i - \nu_i) \left(\nu_i \log \frac{n}{\beta} + \nu_i \log \frac{1}{\nu_i}\right) \\
                 &\geq \mu_i (1 - \nu_i) + \gamma_1 (\mu_i - \nu_i)^2 - 8 (\mu_i - \nu_i) \left( \frac{1}{n} \log \frac{n}{\beta} + \frac {\log n}{n}\right) \\
                 &= \mu_i (1 - \nu_i) + \gamma_1 (\mu_i - \nu_i)^2 - \frac{\gamma_2}{n} (\mu_i - \nu_i) \\
                 &> \mu_i (1 - \nu_i) + (\mu_i - \nu_i) \left(\gamma_1 \frac{n-2}{2n} - \frac{\gamma_2}{n} \right) \\
                 &\stackrel{\text{(a)}}{\geq} \frac{\mu_i - \nu_i}{3} \left(\frac{\gamma_1}{2} - \gamma_2 \right) \\
                 &\stackrel{}{\geq} \frac{\mu_i - \nu_i}{3} \left(\frac{\gamma_1}{\gamma_2} - \gamma_2 \right)
  \end{align*}
  Step (a) is due to the fact that $1/n < 1/2$ and therefore $n \geq 3$.
\end{itemize}
Putting the above cases together, we obtain the following equations:
\[
  \text{LB2}_i \geq
  \begin{cases}
    (\mu_i - \nu_i) \left( 2\sqrt{2c} - 5/2 - 8 \log \frac{n}{\beta}  \right) &\text{if $\mu_i \geq \nu_i \geq 1/2$, or $\nu_i \leq \mu_i \leq 1/2$} \\
    0 &\text{if $1/n \leq \nu_i \leq 1/2 \leq \mu_i$, and $\mu_i - \nu_i \leq \frac{1}{2\gamma_2}$}\\
    \frac{\mu_i - \nu_i}{2} \left( \frac{\gamma_1}{\gamma_2} - \gamma_2 \right) &\text{if $1/n \leq \nu_i \leq 1/2 \leq \mu_i$, and $\mu_i - \nu_i > \frac{1}{2\gamma_2}$} \\
    \frac{\mu_i - \nu_i}{3} \left( \frac{\gamma_1}{\gamma_2} - \gamma_2 \right) &\text{if $\nu_i \leq 1/n <1/2 \leq \mu_i$}

  \end{cases}
\]
Fix $\eta\geq 0$. Let $c = 8 \left(\log \frac{2n^2}{\eta}\right)^2$, we have $\gamma_1 > \left( 8 \log \frac{n^2}{\eta}\right)^2$, and $\gamma_2 = 8 \log \frac{n^2}{\beta}$, so $$\frac{\gamma_1}{\gamma_2} - \gamma_2 = \frac{(\sqrt{\gamma_1}-\gamma_2)(\sqrt{\gamma_1}+\gamma_2)}{\gamma_2} > 8 \frac{\sqrt{\gamma_1}+\gamma_2}{\gamma_2} \log \frac{\beta}{\eta} $$
and thus we get
\[
  \text{LB2}_i \geq
  \begin{cases}
    8(\mu_i - \nu_i) \log \frac{\beta}{\eta} &\text{if $\mu_i \geq \nu_i \geq 1/2$, or $\nu_i \leq \mu_i \leq 1/2$} \\
    0 &\text{if $1/n \leq \nu_i \leq 1/2 \leq \mu_i$, and $\mu_i - \nu_i \leq \frac{1}{2\gamma_2}$}\\
    \frac{4 (\mu_i - \nu_i) (\sqrt{\gamma_1}+\gamma_2)}{\gamma_2} \log \frac{\beta}{\eta}  &\text{if $\nu_i \leq 1/2 \leq \mu_i$, and $\mu_i - \nu_i > \frac{1}{2\gamma_2}$}
  \end{cases}
\]
That is, if $\beta \geq \eta$, we have $ \text{LB2}_i \geq 0$ for all $i \in \{1,\dots,\numtar\}$.

On the other hand, if $\beta < \eta$, we get $\frac{4(\sqrt{\gamma_1}+\gamma_2)}{\gamma_2} = \frac{4(\log \frac{n^2}{\eta} +\log \frac{n^2}{\beta})}{\log \frac{n^2}{\beta}} \leq 8$, and therefore $ \text{LB2}_i \geq 8(\mu_i - \nu_i) \log \frac{\beta}{\eta}$.

Summing over all $i \in \{1,\dots,\numtar\}$, we get that for $\beta < \eta$, it holds $\text{LB2} \geq \sum_{\targetvar_i \in \targetVarDom} \abs{\mu_i - \nu_i} \cdot \frac{2 \alpha^2\beta^2 \left(1 -2 \nr\right)^2}{\hplus \hminus} \log \frac{\beta}{\eta}$. We hence get
\begin{align*}
  \text{LB2} \geq
  \begin{cases}
    - 2 \numtar \left(1 -2 \nr\right)^2 \alpha \beta\log \frac{\eta}{\alpha\beta}   &\text{if $\alpha \beta < \eta$} \\
    0 &\text{if $\alpha \beta \geq \eta$}
  \end{cases}
\end{align*}
Further relaxing the above condition by 
$\alpha \beta\log \frac{\eta}{\alpha\beta} \leq \eta - \alpha \beta \leq \eta$, we obtain:
\begin{align}
  \text{LB2} \geq - 2 \numtar \left(1 -2 \nr\right)^2 \eta \label{eq:auxgain_bound_lb2}
\end{align}
Combining Equation~\eqrefcus{eq:auxgain_lowerbound_expanded_lb1lb2}, \eqrefcus{eq:auxgain_bound_lb1}, and \eqrefcus{eq:auxgain_bound_lb2}, we get
\begin{align}
  \auxgain \geq \frac{1}{16} \left( 1 - 2\nr \right)^2 \ectgain - 2 \numtar \left(1 -2 \nr\right)^2 \eta. \label{eq:auxgain_vs_ec2gain_proof_final}
\end{align}
Hence, we have related \auxgainggraph to \ectgainggraph, as stated in \lemref{lm:hi_eced_hi_faux}.

\subsubsection{Bounding $\auxgain$ against  $\ecedgain$}\label{sec:proof_auxgain_vs_ecedgain}

To finish the proof for \lemref{lm:hi_eced_hi_faux}, it remains to bound $\auxgain$ against $\ecedgain$. In this subsection, we complete the proof of \lemref{lm:hi_eced_hi_faux}, by showing that $\auxgain(\Test_e \given \psi) + 2 \numtar \left(1 -2 \nr\right)^2 \eta \geq \ecedgainat{\psi}{\Test_e}/64 $.


Recall that $\nr$ is the noise rate of test $e$. Let $\rho = \frac{\nr}{1-\nr}$ be the discount factor for inconsistent root-causes. By the definition of $\ecedgain$ in Equation~\eqrefcus{eq:ecedgain}, we first expand the expected offset value of performing test $e$: 
\begin{align*}
  \expctover{\test_e}{\ecedgainPeroutcomeoffset(x_e)} = \sum_{\targetvar_i \in \targetVarDom} (\alpha_i + \beta_i) (1 - \alpha_i - \beta_i) \nr \left(1 - \rho^2\right).
\end{align*}

Denote $\gamma = \nr\left(1 - \rho^2\right)$. 
Then, we can expand $\ecedgain$ as
\begin{align}
  & \ecedgain\nonumber \\
  &= \sum_{\targetvar_i \in \targetVarDom} \left( \overbrace{(\alpha_i + \beta_i) (1 - \alpha_i - \beta_i) \left(1 - \gamma\right)}^{(\text{initial total edge weight}) - (\text{offset value})}\right. \nonumber \\
  & \qquad \qquad \left. - \overbrace{\left( \hplus (\alpha_i + \rho \beta_i) (\alpha + \rho \beta - \alpha_i - \rho \beta_i) +  \hminus (\beta_i + \rho \alpha_i) (\beta  + \rho \alpha - \beta_i - \rho \alpha_i)   \right)}^{\text{expected remaining weight after discounting}} \right) \nonumber \\
  &= \hplus \sum_{\targetvar_i \in \targetVarDom} \left( - \gamma  \alpha_i (\alpha - \alpha_i) + \alpha_i (\beta - \beta_i) (1 - \gamma - \rho) + \beta_i (\alpha - \alpha_i) (1 - \gamma - \rho) + \beta_i (\beta - \beta_i) (1-\gamma - \rho^2) \right) \nonumber  \\
  &  +  \hminus \sum_{\targetvar_i \in \targetVarDom} \left(-\gamma \beta_i (\beta - \beta_i) + \beta_i (\alpha - \alpha_i) (1 - \gamma - \rho) + \alpha_i (\beta - \beta_i) (1 - \gamma - \rho) +  \alpha_i (\alpha - \alpha_i) (1-\gamma - \rho^2) \right)  \nonumber  \\
  &=  \sum_{\targetvar_i \in \targetVarDom} \left(2 (1-\gamma-\rho) \alpha_i (\beta - \beta_i) + \left(\hplus (1-\gamma - \rho^2)  - \hminus \gamma\right)\beta_i (\beta - \beta_i) \right. \nonumber \\
  &\hspace{8cm}+  \left.\left( \hminus (1-\gamma - \rho^2)  - \hplus \gamma \right) \alpha_i (\alpha - \alpha_i)\right) \nonumber 
\end{align}
Since $\gamma = \frac{\nr (1-2\nr)}{(1-\nr)^2}$, $1-\gamma - \rho^2 = 
\frac{1-2\nr}{1-\nr}$, and $1-\gamma - \rho = \left(\frac{1-2\nr}{1-\nr}\right)^2$, we have,
\begin{align*}
  \hplus (1-\gamma - \rho^2)  - \hminus \gamma &= (\alpha (1-\nr) + \beta \nr) \frac{1-2\nr}{1-\nr} - (\alpha \nr + \beta (1 - \nr)) \frac{\nr (1-2\nr)}{(1-\nr)^2}  = \left( \frac{1-2\nr}{1-\nr}\right)^2 \alpha
\end{align*}
Therefore
\begin{align}
  \ecedgain
  &=  \left( \frac{1-2\nr}{1-\nr}\right)^2 \left(\alpha \sum_{\targetvar_i \in \targetVarDom} \beta_i (\beta - \beta_i) + \beta \sum_{\targetvar_i \in \targetVarDom} \alpha_i (\alpha - \alpha_i) + 2 \sum_{\targetvar_i \in \targetVarDom} \alpha_i (\beta - \beta_i)\right) \nonumber \\
    &=  \left( \frac{1-2\nr}{1-\nr}\right)^2 \ectgain \label{eq:gain_ec2_noisy}
\end{align}
Combining Equation \eqrefcus{eq:gain_ec2_noisy} with Equation~\eqrefcus{eq:auxgain_vs_ec2gain_proof_final} we obtain
\begin{align}
  \auxgain + 2 \numtar \left(1 -2 \nr\right)^2 \eta 
                                                      &\geq \frac{(1-\nr)^2}{16} \ecedgain\nonumber \\
                                                      &= \frac{1}{16} \left( 1-2\nr \right)^2 \ectgain \nonumber 
\end{align}

With the results from Appendix~\secref{sec:proof_auxgain_vs_ec2gain} and \secref{sec:proof_auxgain_vs_ecedgain},  we therefore complete the proof of \lemref{lm:hi_eced_hi_faux}.

\subsection{Bounding the error probability: Noiseless vs. Noisy setting}\label{sec:proof_nyvsnl}

Now that we have seen how \ECED interacts with our auxiliary function in terms of the one-step gain, it remains to understand how one can relate the one-step gain to the gain of an optimal policy \optkgaingraph, over $k$ steps. In this subsection, we make an important step towards this goal.

Specifically, we provide

\begin{lemma}\label{lm:nyvsnl_supp}
  Consider a policy $\pi$ of length $k$, and assume that we are using a stochastic estimator ({\sc SE}). Let $\eptop$ be the error probability of {\sc SE} before running policy $\pi$, $\epnoisybot$ be the average error probability of {\sc SE} after running $\pi$ in the noisy setting, and $\epnlessbot$ be the average error probability of {\sc SE} after running $\pi$ in the noiseless setting. Then
  \begin{align*}
    \epnlessbot \leq \epnoisybot
  \end{align*}
\end{lemma}
\begin{proof}[Proof of Lemma~\ref{lm:nyvsnl_supp}]Recall that a \emph{stochastic estimator} predicts the value of a random variable, by randomly drawing from its distribution. Let $\pi$ be a policy. We denote by $\errorst({\pi_\phi})$ the expected error probability of an stochastic estimator after observing $\pi_\phi$ :
  \begin{align*}
    \epnoisybot &= \expctover{\phi}{\errorst({\pi_\phi})} = \sum_{\phi} p(\pi_\phi) \sum_{\targetvar\in \targetVarDom} p(\targetvar \given \pi_\phi) (1 - p(\targetvar \given \pi_\phi))
  \end{align*}
  where $\phi \in \Testset \times \obsDom$ denotes a set of test-outcome pairs, and $\pi_\phi$ denotes a path taken by $\pi$, given that it observes $\phi$.

  Now, let us see what happens in the noiseless setting: we run $\pi$ exactly as it is, but in the end compute the error probability of the noiseless setting (i.e., as if we know which test outcomes are corrupted by noise). Denote the noise put on the tests by $\Xi$, and the realized noise by $\xi$. We can imagine the noiseless setting through the following equivalent way: we ran the same policy $\pi$ exactly as in the noisy setting. But upon completion of $\pi$ we reveal what $\Xi$ was. We thus have
  \begin{align*}
    p(\targetvar \given \pi_\phi) = \sum_{\Xi=\xi} p(\targetvar \given \pi_\phi, \xi) p(\xi \given \pi)
  \end{align*}
  The error probability upon observing $\pi_\phi$ and $\Xi = \xi$ is
  \begin{align*}
    \errorst(\pi_\phi, \xi) = \sum_{\targetvar\in \targetVarDom} p(\targetvar \given \pi_\phi, \xi) (1 - p(\targetvar \given \pi_\phi, \xi)).
  \end{align*}
  The expected error probability in the noiseless setting after running $\pi$ is
  \begin{align}\label{eq:epbotnoiseless}
    \epnlessbot &= \expctover{\phi,n}{\errorst(\pi_\phi, \xi)} = \sum_{\phi, n} p(\pi_\phi, \xi) \sum_{\targetvar\in \targetVarDom} p(\targetvar \given \pi_\phi, \xi) (1 - p(\targetvar \given \pi_\phi, \xi))
  \end{align}

  Now, we can relate $\epnoisybot$ to $\epnlessbot$.
  \begin{align*}
    \epnoisybot &= \sum_\phi p(\pi_\phi) \sum_{\targetvar\in \targetVarDom} p(\targetvar \given \pi_\phi)(1 - p(\targetvar \given \pi_\phi))\\
                &= \sum_\phi p(\pi_\phi) \sum_{\targetvar\in \targetVarDom} \sum_\xi p(\xi \given \pi_\phi) p(\targetvar \given \pi_\phi, \xi)(1 -  \sum_n p(\xi \given \pi_\phi) p(\targetvar \given \pi_\phi, \xi))\\
                &\stackrel{\text{(a)}}\geq \sum_\phi p(\pi_\phi) \sum_{\targetvar\in \targetVarDom} \sum_\xi p(\xi \given \pi_\phi) p(\targetvar \given \pi_\phi, \xi)(1 - p(\targetvar \given \pi_\phi, \xi))\\
                &= \sum_{\phi, \xi} p(\pi_\phi, \xi) \sum_{\targetvar\in \targetVarDom} p(\targetvar \given \pi_\phi, \xi) (1 - p(\targetvar \given \pi_\phi, \xi))
  \end{align*}
  where (a) is by Jensen's inequality and the fact that $f(x) = x(1-x)$ is concave. Combining with Equation~\eqrefcus{eq:epbotnoiseless} we complete the proof.
\end{proof}

Essentially, \lemref{lm:nyvsnl_supp} implies that, in terms of the reduction in the expected prediction error of {\sc SE}, running a policy in the noise-free setting has higher gain than running the exact same policy in the noisy setting. This result is important to us, since analyzing a policy in the noise-free setting is often easier. We are going to use \lemref{lm:nyvsnl_supp} in the next section, to relate the gain of an optimal policy \optgainectggraph in the \ECT objective (which assumes tests to be noise-free), with the gain \optkgaingraph in the auxiliary function (which considers noisy test outcomes).
\subsection{The Key Lemma: One-step Gain of \ECED VS. $k$-step Gain of $\OPT$}\label{sec:proofkeylemma}
Now we are ready to state our key lemma, which connects \auxgainggraph to \optkgaingraph.
\begin{lemma}[Key Lemma]\label{lm:keylemma_onestepgain_supp}
  Fix $\eta,\tau\in(0,1)$. Let $\numrc=|\hiddenVarDom|$ be the number of root-causes, $\numtar=|\targetVarDom|$ be the number of target values, $\OPT(\delta_\opt)$ be the optimal policy that achieves $\errorprob(\OPT(\delta_\opt)) \leq \delta_\opt$, and $\psi_\ell$ be the partial realization observed by running \ECED with cost $\ell$. We denote by $\auxavg(\ell) := \expctover{\psi_\ell}{\auxobj(\psi_\ell)}$ the expected value of $\auxobj(\psi_\ell)$ over all the paths $\psi_\ell$ at cost $\ell$. Assume that $\auxavg(\ell) \leq \deltag$. We then have
  \begin{align*}
    \auxavg(\ell) - \auxavg(\ell+1) \geq \frac{\auxavg(\ell) - \delta_\opt }{k}\cdot \frac{c_\epsilon}{c_\delta} + c_{\eta,\nr}.
  \end{align*}
  where $k = \cost(\OPT(\delta_\opt)))$,
  $c_{\eta,\nr} \triangleq 2\numtar (1-2\nr)^2 \eta $, 
  $c_\delta \triangleq \left(6 c+8 \right)\log(n/\deltag)$,
  $c \triangleq 8 \left(\log (2n^2/\eta)\right)^2$,
  and $c_\nr \triangleq (1-2\nr)^2/16$. 
\end{lemma}

\begin{proof}[Proof of Lemma~\ref{lm:keylemma_onestepgain_supp}.]  Let $\psi_\ell$ be a path ending up at level $\ell$ of the greedy algorithm. 
  Recall that $\ectgain (\Test_e \given \psi_\ell)$ denotes the gain in $\ecobj$ if we perform test $e$ and assuming it to be \emph{noiseless} (i.e., we perform edge cutting as if the outcome of test $e$ is noiseless), conditioning on partial observation $\psi_\ell$. Further, recall that $\auxgain(\Test_e \given \psi_\ell)$ denotes the gain in $\auxobj$ if we perform \emph{noisy} test $e$ after observing $\psi_\ell$ and perform Bayesian update on the root-causes.

  Let $e = \argmax_{e'} \ecedgain(\Test_{e'} \given \psi_\ell)$ be the test chosen by \ECED, and $\hat{e} = \argmax_{e'} \ectgain(\Test_{e'} \given \psi_\ell)$ be the test that maximizes $\ectgain$, then by \lemref{lm:hi_eced_hi_faux} we know
  \begin{align}
    \auxgain({\Test_e}\given \psi_\ell) + c_{\eta,\nr}
    &\geq \frac{(1-\nr)^2}{16}\left(\ecedgainat{\psi_\ell}{\Test_e}\right) \nonumber \\
    &\geq \frac{(1-\nr)^2}{16}\left(\ecedgainat{\psi_\ell}{\Test_{\hat{e}}} \right) \nonumber \\
    &= \frac{1}{16} \left( 1-2\nr \right)^2 \ectgainat{\psi}{\Test_{\hat{e}}} \label{eq:gainfaux_vs_gainec2}
  \end{align}

  Note that $\ectgainat{\psi_\ell}{\Test_e}$ is the EC2 gain of test $e$ over the \emph{normalized} edge weights at step $\ell+1$ in the noiseless setting. That is, upon observing $\psi_\ell$, we create a new EC2 problem instance (by considering the posterior probability over root-causes at $\psi_\ell$), and run (noiseless) greedy algorithm w.r.t. the EC2 objective on such problem instance. 
  Recall that $c_\nr \triangleq (1-2\epsilon)/16$. By adaptive submodularity \adasubmgraph of $\ecobj$ (in the noiseless setting, see \citet{golovin10near}), we obtain
  \begin{align*}
    \max_e \ectgainat{\psi}{\Test_{{e}}} \stackrel{\substack{\text{adaptive} \\ \text{submodularity}}}{\geq} \frac{\ecobjatl^{\top} - \expct{\ecobjatl^{\bot}}}{k}
  \end{align*}
  where by $\ecobjatl^{\top}$ we mean the initial EC2 objective value given partial realization $\psi_\ell$, and by $\expct{\ecobjatl^{\bot}}$ we mean the expected gain in $\ecobj$ when we run $\OPT\left(\delta_\opt\right)$. 
  Note that $\OPT\left(\delta_\opt\right)$ has worst-case length $k$.

  Now, imagine that we run the policy $\OPT\left(\delta_\opt\right)$, and upon completion of the policy we can observe the noise. We consider the gain of such policy in $\ecobj$:
  \begin{align*}
    \ecobj^{\top} - \expct{\ecobj^{\bot}} \stackrel{\text{(a)}}{=} \eptop - \expct{\ecobj^{\bot}} \stackrel{\text{(b)}}{\geq} \eptop - \epnlessbot.
  \end{align*}
  The reason for step (a) is that the error probability of the stochastic estimator upon observing $\psi_\ell$, i.e., $\eptop$, is equivalent to the total amount of edge weight at $\psi_\ell$, i.e., $\ecobjatl^{\top}$. The reason for step (b) is that under the noiseless setting (i.e., assuming we have access to the noise), the EC2 objective is always a lower-bound on the error probability of the \emph{stochastic estimator} (due to normalization). Thus, $ \expct{\ecobj^{\bot}} \leq \epnlessbot$.

  Hence we get
  \begin{align*}
    \auxgain({\Test_e}\given \psi) + c_{\eta,\nr} \geq c_\epsilon \frac{\eptopatl - \epnlessbotatl}{k}.
  \end{align*}
  Here $\eptopatl$ denotes the error probability under $\Pr{\Targetvar \given \psi_\ell}$, and $\epnoisybotatl$ denotes the expected error probability of running $\OPT\left(\delta_\opt\right)$ after $\psi_\ell$ in the \emph{noise-free} setting.
  By Lemma~\ref{lm:nyvsnl_supp} we get
  \begin{align*}
    \auxgain({\Test_e}\given \psi) + c_{\eta,\nr}  \geq c_\epsilon \frac{\eptopatl - \epnoisybotatl}{k},
  \end{align*}
  where $\epnoisybotatl$ denotes the expected error probability of running $\OPT\left(\delta_\opt\right)$ after $\psi_\ell$ in the \emph{noisy} setting.
  By (the lower bound in) Lemma~\ref{lm:stoc_vs_map_supp}, we know that $\eptopatl = \errorst(\psi_\ell) \geq \errormap(\psi_\ell)$, and hence
  \begin{align*}
    \auxgain({\Test_e}\given \psi) + c_{\eta,\nr}  \geq c_\epsilon \frac{\errormap(\psi_\ell) - \delta_\opt}{k},
  \end{align*}
  Taking expectation with respect to $\psi_\ell$, we get
  \begin{align}
    \expctover{\psi_\ell}{ \auxgain({\Test_e}\given \psi) + c_{\eta,\nr} } \geq c_\epsilon \frac{\expctover{\psi_\ell}{\errormap(\psi_\ell)} - \delta_\opt}{k}.\label{eq:1gain_in_H_vs_error}
  \end{align}
  Using (the upper bound in) Lemma~\ref{lm:aux_vs_perr}, we obtain
  \begin{align}
    \auxavg(\ell) &= \expctover{\psi_\ell}{\auxobj(\psi_\ell)} \nonumber \\
                  &\leq (3c +4) \left(\expctover{\psi_\ell}{\bientropy{ \errormap(\psi_\ell)} } + \expctover{\psi_\ell}{\errormap(\psi_\ell)} \log n\right) \nonumber \\
                  &\stackrel{\text{(a)}}{\leq} (3c +4) \left( \bientropy{\expctover{\psi_\ell}{\errormap(\psi_\ell)}} + \expctover{\psi_\ell}{\errormap(\psi_\ell)} \log n \right) \label{eq:faux_upperbound_a}
  \end{align}
  where (a) is by Jensen's inequality.  

  Suppose we run \ECED, and achieve expected error probability $\deltag$, then clearly before \ECED terminates we have $\expctover{\psi_\ell}{\errormap(\psi_\ell)} \geq \deltag$. Assuming $\expctover{\psi_\ell}{\errormap(\psi_\ell)} \leq 1/2$, we have
  \begin{align}
    \auxavg(\ell)
    &\leq (3c+4) \expctover{\psi_\ell}{\errormap(\psi_\ell)} \left( 2\log \frac{1}{\expctover{\psi_\ell}{\errormap(\psi_\ell)}} + \log n \right) \nonumber \\
    &\leq (3c+4) \expctover{\psi_\ell}{\errormap(\psi_\ell)} \left( 2\log \frac{1}{\deltag} + \log n \right) \nonumber \\
    &\leq  \expctover{\psi_\ell}{\errormap(\psi_\ell)} \cdot (6c +8) \log \frac{n}{\deltag} \label{eq:faux_upperbound_b}
  \end{align}
  which gives us
  \begin{align}
    \expctover{\psi_\ell}{\errormap(\psi_\ell)}
    \geq \frac {\auxavg(\ell)} {(6c +8) \log \frac{n}{\deltag}}
    \stackrel{c_\delta \triangleq (6c +8) \log \frac{n}{\deltag}}{=} \frac {\auxavg(\ell)} {c_\delta}.\label{eq:eplowerbound}
  \end{align}
  Combining Equation~\eqrefcus{eq:eplowerbound} with Equation \eqrefcus{eq:1gain_in_H_vs_error}, we get
  \begin{align}
    \auxavg(\ell) - \auxavg(\ell+1)
    &= \expctover{\psi_\ell}{ \auxgain({e}\given \psi) } \nonumber \\
    &\geq c_\epsilon \frac{\frac {\auxavg(\ell)} {c_\delta} - \delta_\opt}{k} - c_{\eta,\nr} \nonumber \\
    &= \frac{\auxavg(\ell) - \delta_\opt c_\delta}{k} \cdot \frac{c_\epsilon}{c_\delta} - c_{\eta,\nr}\nonumber
  \end{align}
  which completes the proof.
\end{proof}

\subsection{Proof of \thmref{thm:mainresults}: Near-optimality of \ECED}\label{sec:proofmainresults}
We are going to put together the pieces from previous subsection, to give a proof of our main theoretical result (\thmref{thm:mainresults}).
\begin{proof}[Proof of \thmref{thm:mainresults}]
  In the following, we use both $\OPT_{[k]}$ and $\OPT(\delta_\opt)$ to represent the optimal policy that achieves prediction error $\delta_\opt$, with worst-cast cost (i.e., length) $k$. 
  Define $S(\policy,\phi)$ to be the (partial) realization seen by policy $\pi$ under realization $\phi$. With slight abuse of notation, we use $\auxavg\left(\OPT_{[k]}\right) := \expctover{\phi}{\auxobj(S(\OPT_{[k]},\phi))}$ to denote the expected value achieved by running $\OPT_{[k]}$.

  After running $\OPT_{[k]}$, we know by \lemref{lm:aux_vs_perr} that the expected value of $\auxobj$ is lower bounded by $2c \cdot \delta_\opt$. That is, $\delta_\opt \cdot c_\delta \leq \auxavg\left(\OPT_{[k]}\right) \cdot \frac{c_\delta}{2c} \leq \auxavg\left(\OPT_{[k]}\right) \cdot 4 \log (n/\deltag)$, where the last inequality is due to $c_\delta \triangleq (6c +8) \log \frac{n}{\deltag} < 8c \log \frac{n}{\deltag}$. 
  We then have
  \begin{align}
    \auxavg(\ell) - \auxavg(\ell+1)
    &\stackrel{\text{\lemref{lm:keylemma_onestepgain_supp}}}{\geq} \left( \auxavg(\ell) - \delta_\opt \cdot c_\delta\right) \cdot \frac{c_\varepsilon}{kc_\delta} - c_{\eta,\nr} \nonumber \\
    &{\geq} \left( \auxavg(\ell) - \auxavg\left(\OPT_{[k]}\right) \cdot 4 \log \frac{n}{\deltag} \right) \cdot \frac{c_\varepsilon}{kc_\delta} - c_{\eta,\nr} \label{eq:1stepgainfinal}
  \end{align}
  Let $\Delta_\ell \triangleq \auxavg(\ell) - \auxavg\left(\OPT_{[k]}\right) \cdot 4 \log \frac{n}{\deltag}$, so that Inequality (\ref{eq:1stepgainfinal}) implies $\Delta_\ell - \Delta_{\ell+1} \geq \Delta_\ell \cdot \frac{c_\nr}{kc_\delta} - c_{\eta,\nr} $.  From here we get $\Delta_{\ell+1} \leq \left( 1-\frac{c_\nr}{kc_\delta} \right) \Delta_\ell + c_{\eta,\nr}$, and hence
  \begin{align*}
    \Delta_{\gsz} &\leq \left( 1-\frac{c_\nr}{kc_\delta} \right)^{\gsz} \Delta_0 + \sum_{i=0}^{k'} \left(1-\frac{c_\nr}{kc_\delta} \right)^{i} \cdot c_{\eta,\nr} \\
                  &\stackrel{(a)}{\leq} \exp\left(-k'\frac{c_\nr}{kc_\delta} \right) \Delta_0 + \frac{1-\left(1-\frac{c_\nr}{kc_\delta} \right)^{k'}}{\frac{c_\nr}{kc_\delta}} \cdot c_{\eta,\nr} \\
                  &\stackrel{(b)}{\leq} \exp\left(-k'\frac{c_\nr}{kc_\delta} \right) \Delta_0 + \frac{kc_\delta}{c_\nr} \cdot c_{\eta,\nr}
  \end{align*}
  where step (a) is due to the fact that $(1-x)^{k'} \leq \exp(-k'x)$ for any $x<1$, and step (b) is due to $\left(1-\frac{c_\nr}{kc_\delta} \right)^{k'} > 0$. It follows that
  \begin{align*}
    \auxavg(k') - \auxavg\left(\OPT_{[k]}\right) \cdot 4 \log \frac{n}{\deltag}
    &\leq \exp\left(-k'\frac{c_\nr}{kc_\delta} \right) \Delta_0 + \frac{kc_\delta}{c_\nr} \cdot c_{\eta,\nr} \\
    & \leq \exp\left(-k'\frac{c_\nr}{kc_\delta}\right) \left(\auxavg(\emptyset) - \auxavg\left(\OPT_{[k]}\right) \cdot 4 \log \frac{n}{\deltag} \right) + \frac{kc_\delta}{c_\nr} \cdot c_{\eta,\nr}
  \end{align*}
  This gives us
  \begin{align}
    \auxavg(k') & \leq \underbrace{\auxavg(\emptyset) \cdot \exp\left(-k'\frac{c_\nr}{kc_\delta}\right)}_{\text{UB1}} + \underbrace{\auxavg\left(\OPT_{[k]}\right) \cdot 4 \log \frac{n}{\deltag}\left( 1-\exp\left(-k'\frac{c_\nr}{kc_\delta}\right)  \right)}_{\text{UB2}} + \underbrace{\frac{kc_\delta}{c_\nr} \cdot c_{\eta,\nr}}_{\text{UB3}} \label{eq:greedy_aux_upperbound}
  \end{align}


  Denote the three terms on the RHS. of Equation \eqrefcus{eq:greedy_aux_upperbound} as UB1, UB2 and UB3, respectively. We get

  \begin{equation*}
    \begin{cases}
      \text{UB1} &\stackrel{\text{Eq}~\eqrefcus{eq:faux_upperbound_a}}{\leq} (3c +4) \left(1 + \log n \right) \cdot \exp\left(-k'\frac{c_\nr}{kc_\delta}\right) \\
      \text{UB2} &\stackrel{\text{Eq}~\eqrefcus{eq:faux_upperbound_b}}{<} (6c+8) \cdot \delta_\opt \log \frac{n}{\delta_\opt} \cdot 4 \log \frac{n}{\deltag}\\
      \text{UB3} & = k \cdot (6c +8) \log \frac{n}{\deltag} \cdot\frac{2\numtar (1-2\nr)^2 \eta}{\frac{1}{16}(1-2\nr)^2} = (6c +8) \cdot 32 \cdot k \cdot \log \frac{n}{\deltag} \cdot \numtar \eta
    \end{cases}
  \end{equation*}

  Now we set
  \begin{equation}
    \begin{cases}
      k' &\triangleq \frac{k c_\delta}{c_\varepsilon} \cdot \ln \frac{8 \log n}{\deltag} \\
      \delta_\opt & \triangleq \frac{\deltag}{64\cdot 36 \cdot \log n \cdot \log \frac{1}{\deltag} \cdot \log \frac{n}{\deltag}}
    \end{cases}\label{eq:set_eced_cost}
  \end{equation}
  and obtain $\exp\left(-k'\frac{c_\nr}{kc_\delta}\right) = \frac{\deltag}{8\log n}$. It is easy to verify that $\text{UB1} \leq 2c \cdot \frac{\delta_g}{4}$, and $\text{UB2} \leq 2c \cdot\frac{\delta_g}{2}$.

  We further set
  \begin{align}
    \eta \triangleq \textstyle \frac{\deltag}{16\cdot 32 \cdot k \numtar \cdot \log \frac{n}{\deltag}},\label{eq:set_eta}
  \end{align}
  and obtain $\text{UB3} = 2c \cdot\frac{\delta_g}{4}$. 

  Combining the upper bound derived above for UB1, UB2, UB3, and by Equation~\eqrefcus{eq:greedy_aux_upperbound}, we get $\auxavg(k') \leq 2c \cdot \deltag$. By \lemref{lm:aux_vs_perr} we know that the error probability is upper bounded by $\errorprob = \expctover{\psi_{k'}}{\errormap(\psi_{k'})} \leq \frac{\auxavg(k')}{2c} \leq \deltag$. That is, with the cost $k'$ specified in Equation~\eqrefcus{eq:set_eced_cost}, \ECED is guaranteed to achieve $\errorprob \leq \deltag$.

  It remains to compute the (exact) value of $k'$. Combining the definition of $c \triangleq 8 \left(\log (2n^2/\eta)\right)^2$ and $c_\delta \triangleq \left(6 c+8 \right)\log(n/\deltag)$ with Equation~\eqrefcus{eq:set_eta} it is easy to verify that
  \begin{align*}
    c_\delta \leq c_1 \cdot \left( \log \frac{nk}{\deltag} \right)^2 \cdot \log \frac{n}{\deltag},
  \end{align*}
  holds for some constant $c_1$. Therefore by Equation~\eqrefcus{eq:set_eced_cost},
  \begin{align*}
    k' \leq k \cdot c_1 \left( \log \frac{nk}{\deltag} \right)^2 \log \frac{n}{\deltag} \cdot \frac{1}{c_\varepsilon} \ln{ \frac{8 \log n}{\deltag}} =  \bigO{\frac{k}{c_\varepsilon} \left(\log \frac{nk}{\deltag}\right)^2 \left(\log \frac{n}{\deltag}\right)^2 }. 
  \end{align*}
  To put it in words, it suffices to run \ECED for 
  $\bigO{\frac{k}{c_\varepsilon} \left(\log \frac{nk}{\deltag}\right)^2 \left(\log \frac{n}{\deltag}\right)^2}$ steps to have expected error below $\deltag$, where $k$ denotes the worst-case cost the optimal policy that achieves expected error probability $\delta_\opt \triangleq \bigO{\frac{\deltag}{\left(\log n \cdot \log (1/\deltag)\right)^2}}$; hence the completion of the proof.
\end{proof}


\section{Examples When \GBS and the Most Informative Policy Fail}\label{sec:countereg}

In this section, we provide problem instances where \GBS and/or the Most Informative Policy may fail, while \ECED performs well. Since in the noise-free setting \ECED is equivalent to \ECT, it suffices to demonstrate the limitations of \GBS and the most informative policy, even if we provide just examples that apply to the noise-free setting.

\subsection{A Bad Example for \GBS: Imbalanced Equivalence Classes}\label{sec:badexample_gbs}

We use the same example as provided in \citet{golovin10near}. Consider an instance with a uniform prior over $\numrc$ root-causes, $\hiddenvar_1, \dots, \hiddenvar_\numrc$, and two target values $\targetvar_1 = r(\hiddenvar_1) = \dots r(\hiddenvar_{\numrc-1})$, and $\targetvar_2 = r(\hiddenvar_\numrc) $. There are tests $\Testset = \{1, \dots, \numrc\}$ such that $\Pr{\Test_e = 1 \given \hiddenvar_i} = \unit{i = e}$ (all of unit cost). Here, $\unit{\cdot}$ is the indicator function. See \figref{fig:example_gbs} for illustration.
\begin{figure*}[h]
  \centering
  \includegraphics[width=.5\textwidth]{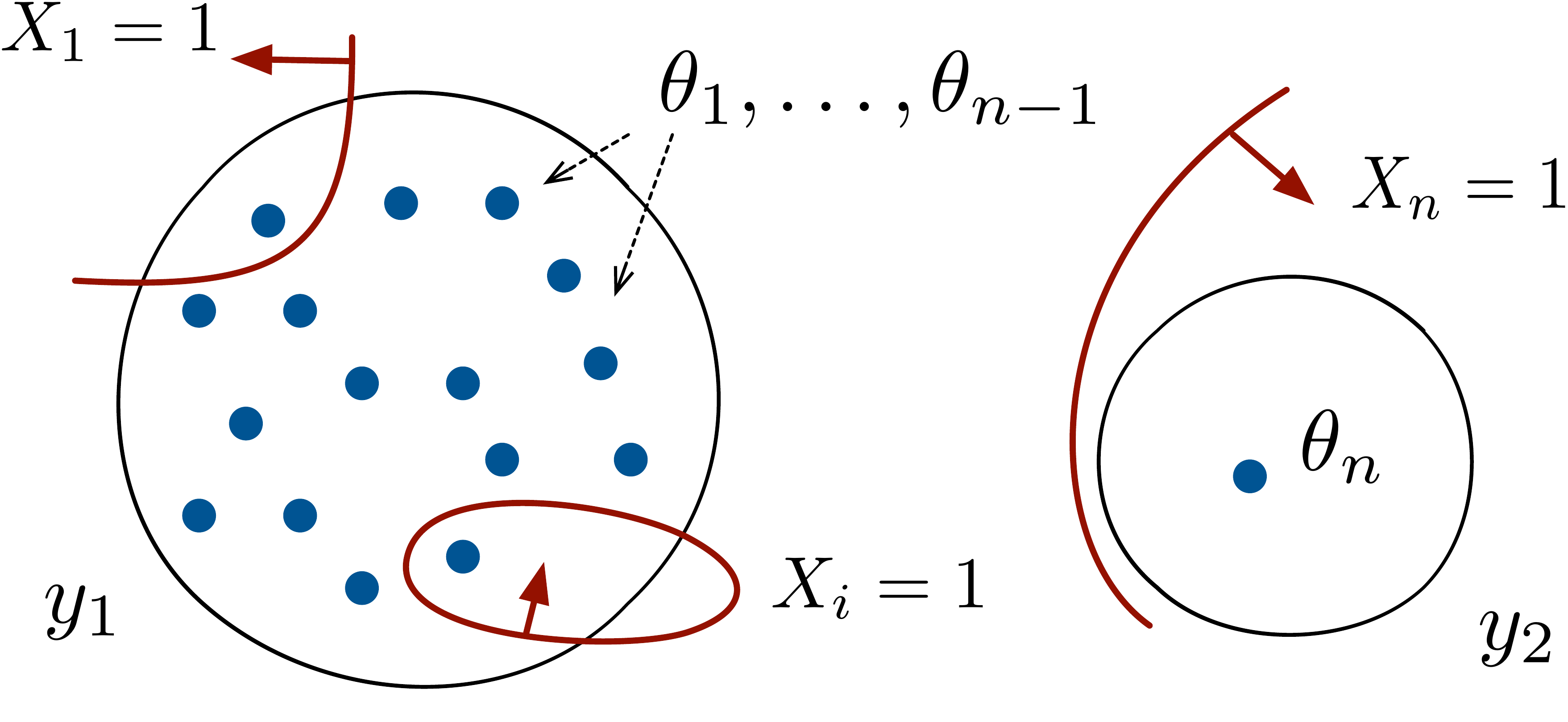}
  \caption{A problem instance where GBS performs significantly worse than \ECED (\ECT).}\label{fig:example_gbs}
\end{figure*}

Now, suppose we want to solve Problem~\eqrefcus{eq:drd} for $\delta = 1/n$. Note that in the noise-free setting, the problem is equivalent to find a minimal cost policy $\pi$ that achieves 0 prediction error, because once the error probability drops below $1/n$ we will know precisely which target value is realized. 

In this case, the optimal policy only needs to select test $\numrc$, however \GBS may select tests $\{1, \dots, \numrc\}$ in order until running test $e$, where $\Hiddenvar = \hiddenvar_e$ is the true root-cause. Given our uniform prior, it takes $\numrc/2$ tests in expectation until this happens, so that \GBS pays, in expectation, $\numrc/2$ times the optimal expected cost in this instance. Note that in this example, \ECED (equivalently, \ECT) also selects test $n$, which is optimal.

\subsection{A Bad Example for the Most Informative Policy: Treasure Hunt}
In this section, we provide a \emph{treasure-hunt} example, in which the most informative policy pays $\bigOmega{\numrc/ \log(\numrc)}$ times the optimal cost. This example is adapted from  \citet{golovin10near}, where they show that the most informative policy (referred to as the \emph{Informative Gain} policy), as well as the myopic policy that greedily maximizes the reduction in the expected prediction error (referred as the \emph{Value of Information} policy), both perform badly, compared with \ECT.

\begin{figure*}[h]
  \begin{center}
    \subfigure[Root-causes and their associated target values]{%
      \includegraphics[width=.75\textwidth]{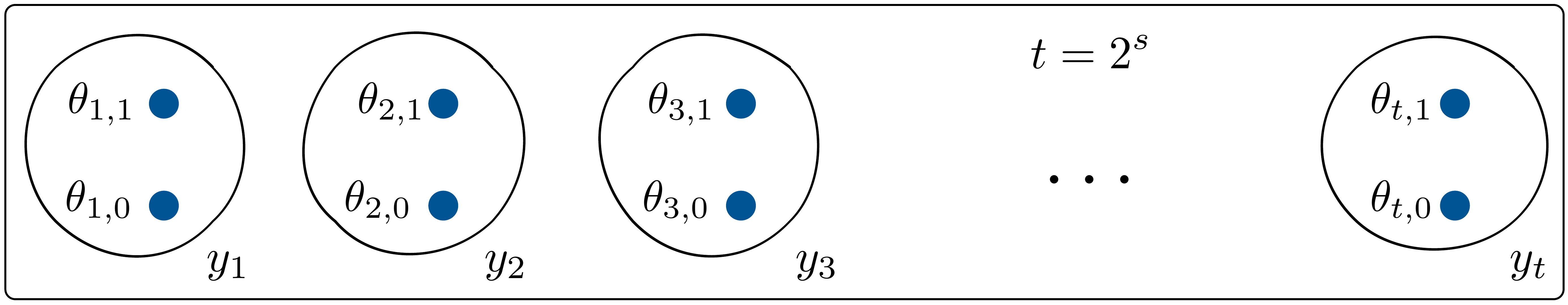}
      \label{fig:illu_example_mis}
    }
    \subfigure[Test set 1]{%
      \includegraphics[width=.7\textwidth]{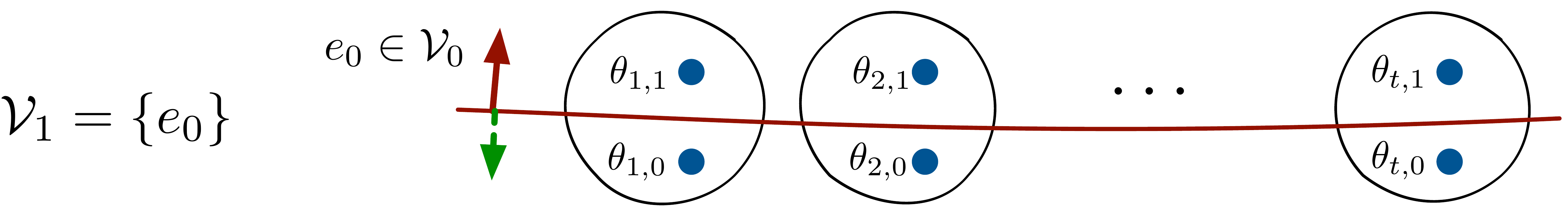}
      \label{fig:illu_example_mis_test1}
    }
    \subfigure[Test set 2]{%
      \includegraphics[width=.7\textwidth]{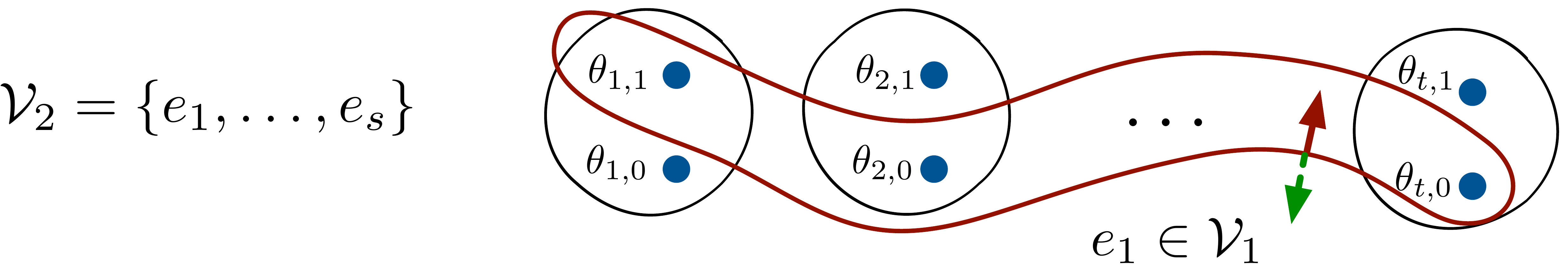}
      \label{fig:illu_example_mis_test2}
    }
    \subfigure[Test set 3]{%
      \includegraphics[width=.7\textwidth]{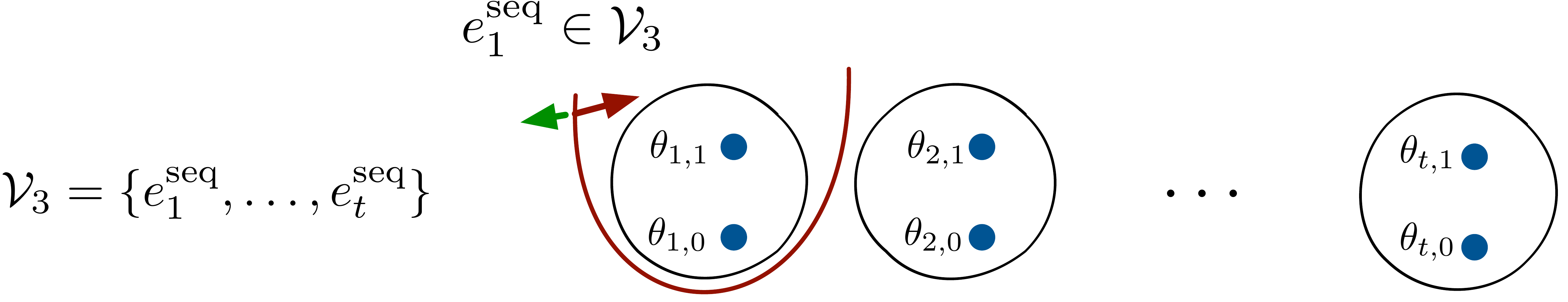}
      \label{fig:illu_example_mis_test3}
    }
  \end{center}
  \caption{A problem instance where the maximal informative policy, and the the myopic policy that greedily maximizes the reduction in the expected prediction error, perform significantly worse than \ECED (\ECT).}
\end{figure*}

Consider the problem instance in \figref{fig:illu_example_mis}. Fix $s > 0$ to be some integer, and let $\numtar = |\targetVarDom| = 2^s$. For each target value $\targetvar_i \in \targetVarDom$, there exists two root-causes, i.e., $\hiddenvar_{i,1}$, $\hiddenvar_{i,0}$, such that $r(\hiddenvar_{i,1}) = r(\hiddenvar_{i,0}) = \targetvar_i$. Denote a root-causes as $\hiddenvar_{i,o}$, if it belongs to target $i$ and is indexed by $o$. We assume a uniform prior over the root-causes: $\{\hiddenvar_{i,o}\}_{i\in\{1,\dots,t\}, o\in{0,1}}$.

Suppose we want to solve Problem~\eqrefcus{eq:drd} for $\delta = 1/3$. Similarly with \secref{sec:badexample_gbs}, the problem is equivalent to find a minimal cost policy $\pi$ that achieves 0 prediction error, because once the error probability drops below $1/3$, we will know precisely which target value is realized.

There are three set of tests, and all of them have binary outcomes and unit cost. The first set $\Testset_1 := \{e_0\}$  contains one test $e_0$, which tells us the value of $o$ of the underlying root-cause $\theta_{i,o}$. Hence for all $i$, $\Hiddenvar = \hiddenvar_{i,o} \Rightarrow X_{e_0} = o$ (see \figref{fig:illu_example_mis_test1}). The second set of tests are designed to help us quickly discover the index of the target value via binary search if we have already run $e_0$, but to offer no information whatsoever (in terms of expected reduction in the prediction error, or expected reduction in entropy of $\Targetvar$) if $e_0$ has not yet been run. There are a total number of $s$ tests in the second set $\Testset_2 := \{e_1, e_2, \dots, e_s\}$. For $z \in \{1,\dots, \numtar\}$, let $b_k(z)$ be the $k^{\text{th}}$ least-significant bit of the binary encoding of $z$, so that $z = \sum_{k=1}^{s} 2^{k-1} b_k(z)$. Then, if $\Hiddenvar = \hiddenvar_{i,o}$, then the outcome of test $e_k \in \Testset_2$ is $\Test_{e_k} = \unit{\phi_k(i) = o}$ (see \figref{fig:illu_example_mis_test2}). The third set of tests are designed to allow us to do a (comparatively slow) sequential search on the index of the the target values. Specifically, we have $\Testset_3 := \{e^{\text{seq}}_1, \dots, e^{\text{seq}}_\numtar\}$, such that $\Hiddenvar = \hiddenvar_{i,o} \Rightarrow X_{e^{\text{seq}}_k} = \unit{i = k}$ (\figref{fig:illu_example_mis_test3}).

Now consider running the maximal informative policy $\policy$ (the same analysis also applies to the value of information policy, which we omits from the paper). Note that in the beginning, no single test from $\Testset_1 \cup \Testset_2$ results in any change in the distribution over $\Targetvar$, as it remains uniform no matter with test is performed. Hence, the maximal informative policy only picks tests from $\Testset_3$, which have non-zero (positive) expected reduction in the posterior entropy of $\Targetvar$. In the likely event that the test chosen is not the index of $\Targetvar$, we are left with a residual problem in which tests in $\Testset_1 \cup \Testset_2$ still have no effect on the posterior. The only difference is that there is one less class, but the prior remains uniform. Hence our previous argument still applies, and $\policy$ will repeatedly select tests in $\Testset_3$, until a test has an outcome of 1. In expectation, the cost of $\policy$ is least $\cost(\policy) \geq \frac{1}{\numtar}\sum_{z=1}^\numtar z = \frac{\numtar + 1}{2}$.

On the other hand, a smarter policy $\policy^*$ will select test $e_0 \in \Testset_1$ first, and then performs a binary search by running test $e_1, \dots, e_s \in \Testset_2$ to determine $b_{k}(i)$ for all $1\leq k \leq s$ (and hence to determine the index $i$ of $\Targetvar$). Since the tests have unit cost, the cost of $\policy^*$ is
$\cost(\policy^*) = s + 1$.

Since $\numtar=2^s$, and $\numrc = 2 \numtar = 2^{s+1}$, we conclude that
\begin{align*}
  \cost(\policy) = \frac{\numtar + 1}{2} > \frac{\numtar}{2} = \frac{\numrc}{4} \frac{s+1}{\log \numrc} = \frac{\numrc}{4 \log(\numrc)} \cost(\policy^*).
\end{align*}


\section{Case Study: Pool-based Active Learning for Classification}
\paragraph{Experimental setup.} To demonstrate the empirical performance of \ECED, we further conduct experiments on two pool-based binary active classification tasks. In the active learning application, we can sequentially query from a pool of data points, and the goal is to learn a binary classifier, which achieves some small prediction error on the unseen data points from the pool, with the smallest number of queries as possible.

\paragraph{Active Learning: Targets and Root-causes}
To discretize the hypotheses space, we use a noisy version of hit-and-run sampler as suggested in \citet{chen2013near}. Each hypothesis can be represented by a binary vector indicating the outcomes of all data points in the training set. Then, we construct an epsilon-net on the set of hypotheses (based on the Hamming distance between hypotheses). We obtain the equivalence classes for \ECED, by assigning each hypothesis to its closest center of epsilon-ball, measured by their Hamming distances. 
Note that the Hamming distance between two hypotheses reflects the difference of prediction error. Consider epsilon-net of fixed radius $\varepsilon$. By construction, hypotheses that lie in the some equivalence classes are at most $2\varepsilon$ away from each other; therefore the hypotheses which are within the epsilon-ball of the optimal hypotheses are considered to be near-optimal. Using the terminology in this paper, hypotheses correspond to root-causes, and the groups of hypothesis correspond to the target variable of interest. Running \ECED, ideally, will help us locate a near-optimal epsilon-ball as quickly as possible.

\paragraph{Baselines.} 
We compare \ECED with the popular uncertainty sampling heuristic (UNC-SVM), which sequentially queries the data points which are the closest to the decision boundary of a SVM classifier. We also compare with the \GBS algorithm, which sequentially queries the data points that maximally reduces the volume of the version space.

\begin{figure}[h]
  \centering
  \subfigure[{\ECT VS. \GBS}]{%
    \includegraphics[width=.33\textwidth]{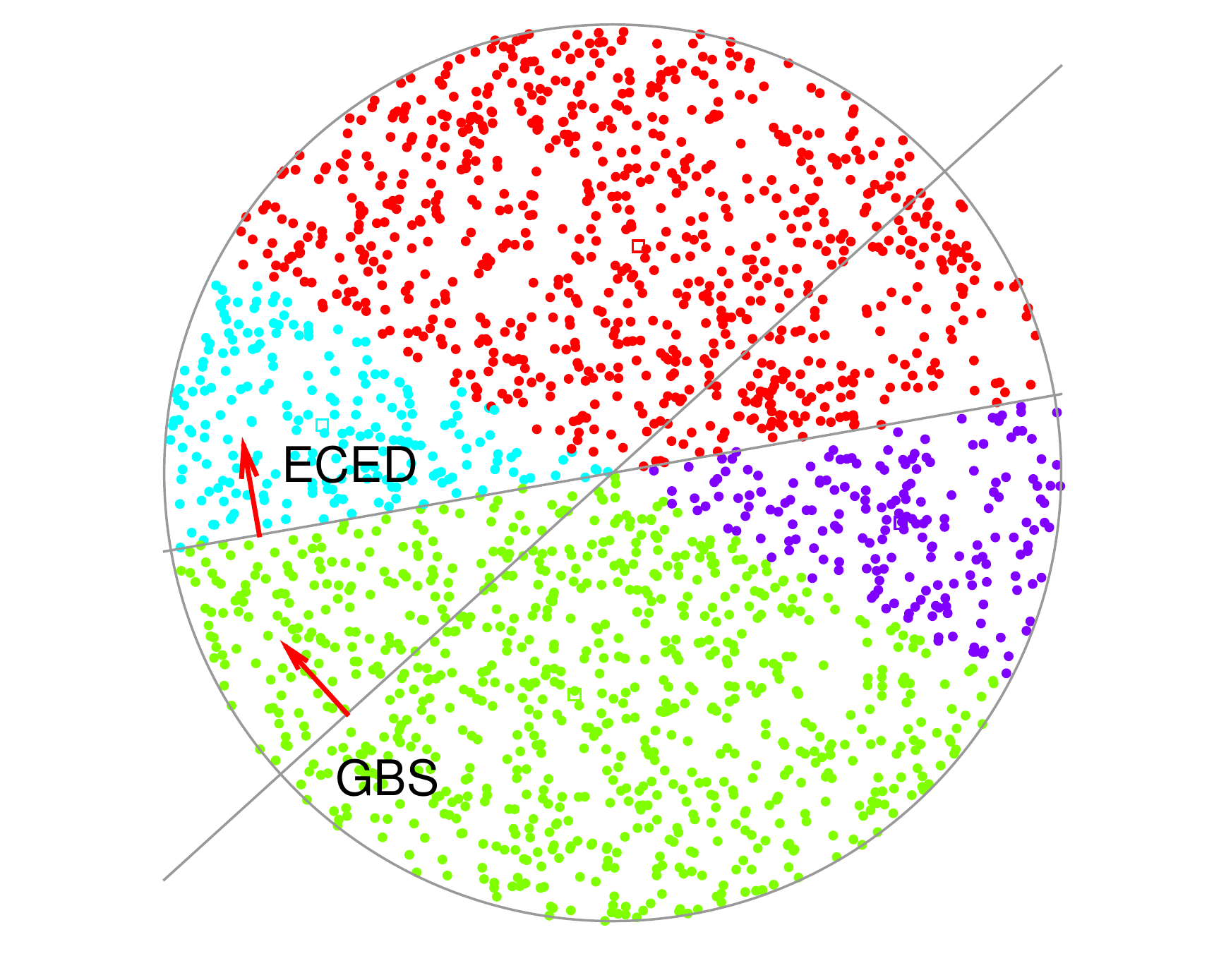}
    \label{fig:al_ec2vsgbs}} 
  \subfigure[\textsl{WDBC}]{%
    \includegraphics[width=.31\textwidth]{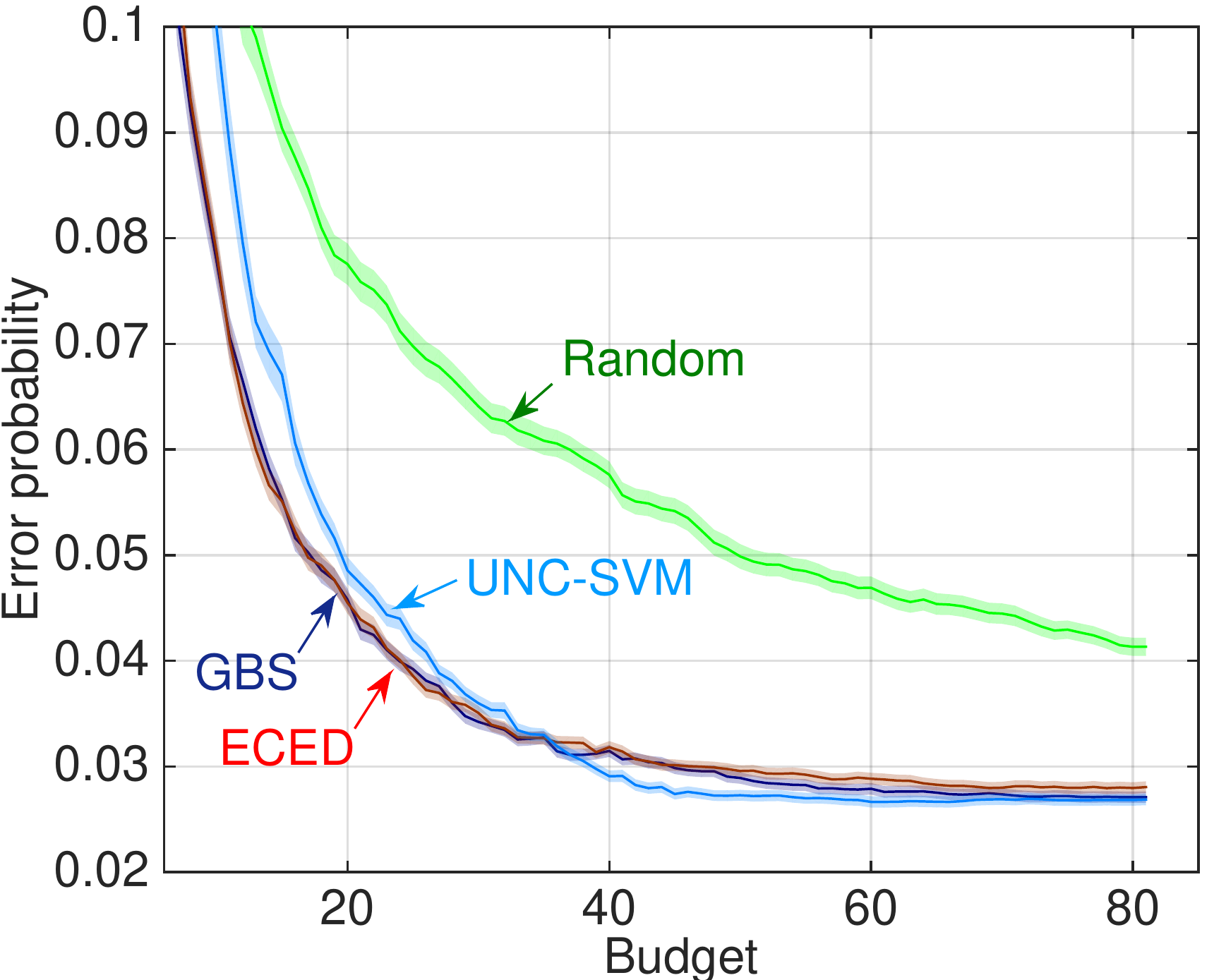}
    \label{fig:al_wdbc_1000hyp_errp02}} 
  \subfigure[\textsl{Fourclass}]{%
    \includegraphics[width=.31\textwidth]{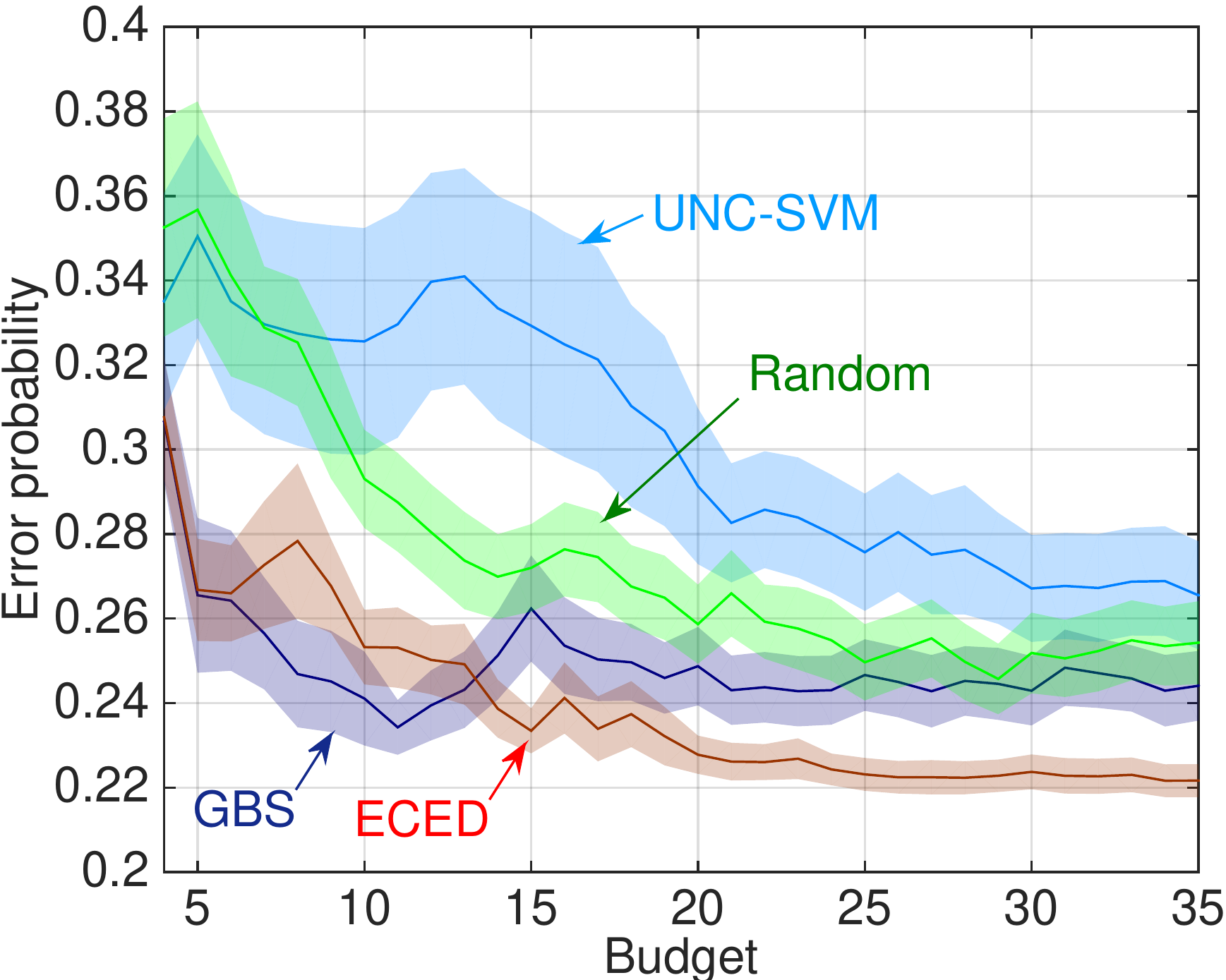}
    \label{fig:al_fourclass_1000hyp_errp1}}
  \caption{Pool-based Active Learning for Classification}
  \label{fig:al_result}
\end{figure}

In \figref{fig:al_ec2vsgbs}, we demonstrate the different behaviors between \GBS and \ECT on a 2-d plane. In this simple example, there are 4 color-coded equivalence classes: we first sample hypotheses uniformly within the unit circle, and then generate equivalence classes, by constructing an epsilon-net over the sampled hypotheses as previously described. 
\figref{fig:al_ec2vsgbs} illustrates two tests (i.e., the gray lines intersecting the circles) selected by \ECED and \GBS, respectively. \ECED primarily selects tests that best disambiguate the clusters, while \GBS focuses on disambiguate individual hypotheses.

\paragraph{Results.}

We evaluate \ECED and the baseline algorithms on the {\sl UCI WDBC} dataset (569 instances, 32-d) and {\sl Fourclass} dataset (862 instances, 2-d). 
For \ECED and \GBS, we sample a fixed number of 1000 hypotheses in each random trial. For both instances we assume a constant error rate $\epsilon = 0.02$ for all tests. \figref{fig:al_wdbc_1000hyp_errp02} and \figref{fig:al_fourclass_1000hyp_errp1} demonstrate that \ECED is competitive with the baselines. Such results suggests that grouping of hypotheses could be beneficial when learning under noisy data.



\end{document}